\documentclass[11pt,reqno]{article}

\usepackage{yub}

\usepackage{amssymb,amsmath,amsfonts}
\usepackage{amsthm}
\usepackage{graphicx}
\usepackage{epstopdf}
\usepackage{bm}
\usepackage{color}
\usepackage{fullpage}
\usepackage{algorithm}
\usepackage[noend]{algorithmic}

\usepackage{amsmath,amsfonts}
\usepackage{graphicx}
\usepackage{mathtools}
\usepackage{amsthm}
\usepackage{subcaption}

\usepackage[utf8]{inputenc} 
\usepackage[T1]{fontenc}    
\usepackage{url}            
\usepackage{booktabs}       
\usepackage{amsfonts}       
\usepackage{nicefrac}       
\usepackage{microtype}      
\usepackage{xcolor}

\newcommand{\actlip}{L_{\sigma}}
\newcommand{\rholip}{L_{\psi}}

\newcommand{\heslb}{\underline{\kappa_0}}
\newcommand{\hesub}{\overline{\kappa_0}}
\newcommand{\gradlb}{\underline{L_0}}
\newcommand{\gradub}{\overline{L_0}}
\newcommand{\gradlip}{\underline{T_0}}

\newcommand{\bddsubg}{C_{sg}}
\newcommand{\bddsube}{C_{se}}
\newcommand{\derivdecay}{C_{\sigma}}
\newcommand{\derivdecayrho}{C_{\psi}}
\newcommand{\covlb}{{\underline{\gamma}}}

\def\rb{{r_{b}}}

\def\<{\langle}
\def\>{\rangle}
\def\sT{{\sf T}}
\def\BT{{\sf T}}
\def\bT{{\boldsymbol T}}
\def\rad{r}

\def\Trace{{\rm Tr}}

\def\hVar{\widehat{\rm Var}}

\def\ud{{\rm d}}
\def\de{{\rm d}}
\def\hR{\widehat{R}}
\def\prob{{\mathbb P}}
\def\E{{\mathbb E}}
\def\reals{{\mathbb R}}
\def\naturals{{\mathbb N}}
\def\bpsi{{\boldsymbol \psi}}

\def\bx{{\boldsymbol x}}
\def\bz{{\boldsymbol z}}

\def\mbX{\mathbb X}
\def\bX{{\boldsymbol X}}
\def\bY{{\boldsymbol Y}}
\def\bu{{\boldsymbol u}}
\def\bv{{\boldsymbol v}}
\def\bw{{\boldsymbol w}}
\def\blambda{{\boldsymbol \lambda}}
\def\btheta{{\boldsymbol \theta}}
\def\hbtheta{\hat{\boldsymbol \theta}}

\def\id{{\rm I}}
\def\Ball{{\sf B}}
\def\Psucc{{\widehat{\rm P}_{\rm succ}}}
\def\normal{{\sf N}}

\def\bzero{{\boldsymbol 0}}
\def\bfone{{\boldsymbol 1}}
\def\bM{{\boldsymbol M}}

\def\sp{{\sf p}}

\def\Rd{R^{(d)}}
\def\Rt{R^{(3)}}

\def\hR{\widehat{R}}
\def\prob{{\mathbb P}}
\def\E{{\mathbb E}}
\def\reals{{\mathbb R}}
\def\bx{{\boldsymbol x}}
\def\bz{{\boldsymbol z}}
\def\ba{{\boldsymbol a}}
\def\bv{{\boldsymbol v}}
\def\bn{{\boldsymbol n}}
\def\bM{{\boldsymbol M}}
\def\bX{{\boldsymbol X}}
\def\bY{{\boldsymbol Y}}
\def\bw{{\boldsymbol w}}
\def\bW{{\boldsymbol W}}
\def\bZ{{\boldsymbol Z}}

\def\cC{{\mathcal C}}

\def\bu{{\boldsymbol u}}
\def\bU{{\boldsymbol U}}
\def\bv{{\boldsymbol v}}
\def\be{{\boldsymbol e}}
\def\blambda{{\boldsymbol \lambda}}

\def\id{{\rm I}}
\def\Ball{{\sf B}}
\def\cZ{{\mathcal Z}}

\def\cA{{\mathcal A}}

\def\Psucc{{\widehat{\rm P}_{\rm succ}}}

\def\bzero{{\boldsymbol 0}}

\def\bt{{\boldsymbol t}}
\def\bx{{\boldsymbol x}}
\def\bxi{{\boldsymbol \xi}}
\def\hbxi{\hat{\boldsymbol \xi}}
\def\bz{{\boldsymbol z}}
\def\bzero{{\boldsymbol 0}}
\def\reals{{\mathbb R}}
\def\inde{{\rm ind}}
\def\deg{{\rm deg}}
\def\eps{\varepsilon}
\def\by{{\boldsymbol y}}

\def\Cn{C_n}
\def\Cl{C_\lambda}
\def\Cs{C_s}
\def\Csp{C_{sp}}
\def\Ch{c_h}

\def\HUB{H}

\title{The Landscape of Empirical Risk for Non-convex Losses}

\author{Song~Mei\footnote{Institute for Computational and Mathematical Engineering, Stanford University}
\and
Yu~Bai\footnote{Department of Statistics, Stanford University}
\and
Andrea~Montanari\footnote{Department of Electrical
Engineering and Department of Statistics, Stanford University}
}

\begin{document}

\maketitle

\begin{abstract}
Most high-dimensional estimation and prediction methods propose to minimize a cost function (empirical risk) that is written as a sum of losses associated to each data point (each example). In this paper we focus on the case of non-convex losses, which is practically important but still poorly understood. Classical empirical process theory implies uniform convergence of the empirical (or sample) risk to the population risk. While --under additional assumptions-- uniform convergence implies consistency of the resulting M-estimator, it does not ensure that the latter can be computed efficiently.


In order to capture the complexity of computing M-estimators, we propose to study the landscape of the empirical risk, namely its stationary points and their properties. We establish uniform convergence of the gradient and Hessian of the empirical risk to their population counterparts, as soon as the number of samples becomes larger than the number of unknown parameters (modulo logarithmic factors). Consequently, good properties of the population risk can be carried to the empirical risk, and we are able to establish one-to-one correspondence of their stationary points. We demonstrate that in several problems such as non-convex binary classification, robust regression, and Gaussian mixture model, this result implies a complete characterization of the landscape of the empirical risk, and of the convergence properties of descent algorithms.

We extend our analysis to the very high-dimensional setting in which the number of parameters exceeds the number of samples, and provide a characterization of the empirical risk landscape under a nearly information-theoretically minimal condition.  Namely, if the number of samples exceeds the sparsity of the unknown parameters vector (modulo logarithmic factors), then a suitable uniform convergence result takes place. We apply this result to non-convex binary classification and robust regression in very high-dimension.
\end{abstract}

\section{Introduction }

M-estimation is arguably the most popular approach to high-dimensional estimation. Given data-points $\{\bz_1,\bz_2, \dots,\bz_n\}$,
$\bz_i\in \reals^d$, we estimate a parameter vector $\btheta\in\reals^p$ via%
\begin{align}
\hbtheta_n & = \arg\min_{\btheta\in \Theta_{n,p}}\hR_n(\btheta) \, ,\label{eq:GeneralERM}\\
 \hR_n(\btheta) &\equiv \frac{1}{n}\sum_{i=1}^n\ell(\btheta;\bz_i)\, . \label{eq:Risk}
\end{align}
Here $\ell:\reals^p\times\reals^d\to\reals$ is a loss function, and $\Theta_{n,p}$ is a constraint set. Prominent examples of this general framework include
maximum likelihood (ML) estimation \cite{fisher1922mathematical} and empirical risk minimization \cite{vapnik1998statistical}.

Once the objective (\ref{eq:GeneralERM}) is formed, it remains to define a computationally efficient scheme 
to approximate it. Gradient descent is the most frequently applied idea. Assuming --for the moment-- $\Theta_{n,p} = \reals^p$, this takes the form
\begin{align}
\hbtheta_n(k+1) = \hbtheta_n(k) - h_k\,\nabla \hR_n(\hbtheta_n(k))\, . \label{eq:GradientDescent}
\end{align}
While a large number of variants and refinements have been developed over the years (projected gradient,
accelerated gradient \cite{nesterov2013introductory}, stochastic
gradient \cite{robbins1951stochastic}, distributed gradient 
\cite{tsitsiklis1984distributed}, and so on), these share many of the strengths and weaknesses of the elementary iteration (\ref{eq:GradientDescent}).

If gradient descent is adopted, the only  freedom is in the choice of the loss function $\ell(\,\cdot\,;\, \cdot\,)$. Convexity 
has been a major guiding principle in this respect. If  the function $\ell(\,\cdot\,;\bz):\reals^p\to\reals$ is convex,
then the empirical risk $\hR_n(\,\cdot\,)$ is convex as well and hence gradient descent is globally convergent to an M-estimator
(the latter is unique under strict convexity). 
Also, strong convexity of $\hR_n(\,\cdot\,)$ can be used to prove optimal statistical guarantees for the M-estimator $\hbtheta_n$.
This line of thought can be traced back as far as Fisher's argument for the asymptotic efficiency of maximum likelihood estimators 
\cite{fisher1922mathematical,fisher1925theory},
and originated many beautiful contributions. In recent years,  a flourishing line of research addresses the very high-dimensional regime $p\gg n$, 
by leveraging on suitable restricted strong convexity assumptions \cite{candes2005decoding,candes2007dantzig,bickel2009simultaneous,negahban2012unified}.

Despite these successes, many problems of practical interest call for non-convex loss functions. Let us briefly mention a few 
examples of non-convex M-estimators that are often preferred by practitioners to their convex counterparts. 
We will revisit these examples in Section \ref{sec:Applications}.

In binary linear classification we are given $n$ pairs $\bz_1=(y_1,\bx_1),\dots, \bz_n=(y_n,\bx_n)$ with $y_i\in \{0,1\}$, $\bx_i\in\reals^d$,
and would like to learn a model of the form $\prob(Y_i=1|\bX_i=\bx) = \sigma(\<\btheta_0,\bx\>)$ with 
$\btheta_0\in\reals^d$ a parameter vector and $\sigma:\reals\to [0,1]$ a threshold function. 
The non-linear square loss $\ell(\btheta;y,\bx) = (y-\sigma(\<\btheta,\bx\>))^2$ is commonly used in practice
\begin{equation}\label{eqn:erm}
  \hR_n(\btheta)\equiv \frac{1}{n}\sum_{i=1}^{n} \Big(y_i - \sigma(\<\btheta,\bx_i\>)\Big)^2 \, .
\end{equation}
Several empirical studies \cite{chapelle2009tighter,wu2012robust,nguyen2013algorithms} demonstrate superior robustness and
classification accuracy of non-convex losses in contrast to convex losses (e.g. hinge or logistic loss). The same loss function is commonly used used in neural-network
models \cite{lecun2015deep}.

A similar scenario arises in robust regression. In this case,
we are given $n$ pairs $\bz_1=(y_1,\bx_1)$,\dots ,$\bz_n=(y_n,\bx_n)$ with $y_i\in \reals$, $\bx_i\in\reals^d$,
and we assume the linear model $y_i =  \<\btheta_0,\bx_i\> + \eps_i$, where the noise terms $\eps_i$ are i.i.d. 
with mean zero. Since Huber's seminal work \cite{huber1973robust}, M-estimators are the method of choice for this problem:
\begin{equation}\label{eqn:erm-robust}
  \hR_n(\btheta)\equiv \frac{1}{n}\sum_{i=1}^{n} \rho\big(y_i - \<\btheta,\bx_i\>\big) \, .
\end{equation}
Robustness naturally suggests to investigate the use of a non-convex function $\rho:\reals\to\reals$, either bounded or increasing
slowly at infinity.

Finally, missing data problems famously lead to non-convex optimization formulations. Consider for instance a mixture-of-Gaussians
problems in which we are given data points $\bz_1,\dots,\bz_n\in\reals^d$, $\bz_i \sim_{iid} \sum_{a=1}^kp_a\normal(\btheta_a,\id_{d\times d})$
(for the sake of simplicity we assume identity covariance and known proportions). The maximum-likelihood problem 
requires to minimize\footnote{Here and below $\phi_d(\bx) \equiv \exp\{-\|\bx\|_2^2/2\}/(2\pi)^{d/2}$ denotes the $d$-dimensional standard Gaussian density.}
\begin{equation}\label{eqn:erm-mixture-0}
  \hR_n(\btheta)\equiv -\frac{1}{n}\sum_{i=1}^{n} \log\left( \sum_{a=1}^k p_a\; \phi_d\big(\bz_i-\btheta_a\big)\right) \, ,
\end{equation}
with respect to the cluster centers $\btheta=(\btheta_1,\dots,\btheta_k)\in\reals^{d\times k}$.
Other examples include low-rank matrix completion \cite{keshavan2009matrix}, phase retrieval \cite{sun2016geometric}, tensor estimation problems 
\cite{richard2014statistical}, and so on.

M-estimation with non-convex loss functions $\ell(\,\cdot\,;\bz):\reals^p\to\reals$ is far less understood than in the convex case.
Empirical process theory guarantees uniform convergence of the sample risk $\hR_n(\,\cdot\,)$ to the population risk
$R(\btheta) \equiv \E[\hR_n(\btheta)]$ \cite{boucheron2013concentration}. However, this does not provide a computationally practical scheme, since 
gradient descent can get stuck in stationary points that are not global minimizers.

In this paper, we present several general 
results on non-convex M-estimation and apply them to develop new analysis in each of the three problems mentioned above.
We next overview our main results and the paper's organization, referring to Section \ref{sec:Related} for a discussion of related work.
\begin{description}
\item[Uniform convergence of gradient and Hessian.] We prove that, under technical conditions on the loss function $\ell(\,\cdot\,;\,\cdot\,)$,
$\sup_{\btheta}\|\nabla\hR_n(\btheta)-\nabla R(\btheta)\|_2\lesssim \sqrt{p(\log n)/n}$ and $\sup_{\btheta}\|\nabla^2\hR_n(\btheta)-\nabla^2 R(\btheta)\|_{\op}
\lesssim \sqrt{p(\log n)/n}$ (we use $\lesssim$ to hide constant factors). We refer to Section \ref{sec:HighDim} for formal statements.

These results complement  the classical analysis that implies uniform convergence of the risk itself, but allow us to control the behavior
of stationary points.
Note that they guarantee uniform convergence of the gradient and  Hessian provided $n,p\to\infty$
with $p(\log p)/n\to 0$. Apart from logarithmic factors, this is the optimal condition.

(In this paper we will refer to the asymptotics $n,p\to\infty$ with $n$ roughly of the same order as $p$ as \emph{high-dimensional 
regime}\footnote{The specific asymptotics $n,p\to\infty$ with $n/p$ converging to a constant is also known as `Kolmogorov asymptotics'
\cite{serdobolskii2013multivariate}.},
to contrast it with the low-dimensional analysis for $n\gg p$. We will refer to the asymptotics $n\ll p$ under sparsity assumptions as 
\emph{very high-dimensional regime}.)
\item[Topology of the empirical risk.] As an immediate consequence of the previous result, the structure of the empirical risk
function $\btheta\mapsto \hR_n(\btheta)$ is --in many cases-- surprisingly simple. Recall that a Morse function is a twice differentiable function 
whose stationary points are non-degenerate (i.e. have an invertible Hessian). In particular, stationary points are isolated, and have a well-defined index.
Assume that the population risk $R(\btheta)$ is \emph{strongly Morse} (i.e., at any stationary point $\btheta$, all the eigenvalues of the Hessian
are bounded away from zero $|\lambda_i(\nabla^2R(\btheta))|\ge \delta$). Then, for $n\gtrsim p\log p$, the stationary points of the empirical risk $\hR_n(\btheta)$
are in one-to-one correspondence with those of the population risk and have the same index (minima correspond to minima, saddles to saddles, and so on).
Weaker conditions ensure this correspondence for local minima alone.
\item[Very high-dimensional regime.] We then extend the above picture to the case in which the number of parameters $p$ exceeds the number of samples $n$,
under the assumption that the true parameter vector $\btheta_0$ is $s_0$-sparse. This setting is relevant to a large number of applications, 
ranging from genomics \cite{peng2010regularized} to signal processing \cite{donoho2006compressed}.
In order to promote  sparse estimates, we study the following $\ell_1$-regularized non-convex problem, cf. Section \ref{sec:VeryHighDim}:
\begin{equation}
\begin{aligned}
\mbox{\rm minimize}&\phantom{AA}\hR_n(\btheta)+\lambda_n \Vert \btheta\Vert_1\, ,\\
\mbox{\rm subject to}&\phantom{AA} \|\btheta\|_2\le \rad\, .
\end{aligned}
\end{equation}
We introduce a \emph{generalized gradient linearity} condition on the loss function $\ell(\,\cdot\,,\,\cdot\,)$ and prove that -- under this condition-- the
above problem has a unique local minimum for $n\gtrsim s_0\log p$. Again this is a nearly optimal scaling since no consistent estimation is possible when $n\lesssim s_0$.
\item[Applications.] Given a particular M-estimation problem with a suitable statistical model, we combine the above results with an 
analysis of the population risk $R(\btheta)$ to derive precise characterizations of the empirical risk. In Section \ref{sec:Applications} we demonstrate that this program can be carried out by studying the three problems outlined below:
\begin{enumerate}
\item \emph{Binary linear classification.} We prove that, for\footnote{Recall that, in this case, the number of parameters $p$ is equal to the ambient dimension $d$.} 
$n\gtrsim d\log d$, the empirical risk has a unique local minimum, that is also the global minimum.
Further, gradient descent converges exponentially to this minimizer: $\|\hbtheta_n (k)-\hbtheta_n\|_2\le C\|\hbtheta(0)-\hbtheta_n\|_2\, (1- h/C)^k$, and enjoys nearly optimal estimation error guarantees: $\|\hbtheta_n-\btheta_0\|_2\le C\sqrt{(d\log n)/n}$. If the true parameter $\btheta_0$ is $s_0$-sparse, 
for  $n \gtrsim s_0 \log d$, the $\ell_1$-regularized empirical risk has a unique local minimum, that is also the global minimum. The minimizer enjoys nearly optimal estimation error guarantees: $\|\hbtheta_n-\btheta_0\|_2\le C\sqrt{(s_0\log n)/n}$. 
\item \emph{Robust regression.} We establish similar results for the robust regression model, under technical assumptions on the loss function $\rho:\reals\to \reals$ and on the distribution of the noise $\eps_i$. 
Namely, we prove that the empirical risk has a unique local minimum, that can be found efficiently via gradient descent,
provided $n\gtrsim d\log d$. If the true parameter $\btheta_0$ is $s_0$-sparse, for $n \gtrsim s_0 \log d$, the $\ell_1$-regularized empirical risk has a unique local minimum. 
\item \emph{Mixture of Gaussians.} We consider the special case of two Gaussians with equal proportions, i.e. $k=2$ with $p_1=p_2=1/2$.
We prove that, for $n\gtrsim d\log d$, the empirical risk has two global minima that are related by exchange of the two Gaussian components 
$(\hbtheta_{1},\hbtheta_2)$ and $(\hbtheta_{2},\hbtheta_1)$, connected via saddle points. The trust region algorithm converges to one of these two minima when 
initialized at random. Also the two minima are within nearly optimal statistical errors from the true centers.
\end{enumerate}
\end{description}

\subsection{Notations}

We use normal font for scalars (e.g.  $a,b,c\dots$) and boldface for 
vectors ($\bx,\bw,\dots$). We will typically reserve capital letters for random variables (and capital bold for random vectors). Given $\bu,\bv\in\reals^m$, 
their  standard scalar product is denoted by $\<\bu,\bv\> \equiv \sum_{i=1}^m u_iv_i$. The $\ell_p$ norm of a vector is --as usual-- 
indicated by $\|\bx\|_p$. 
The $m\times m$ identity matrix is denoted by $\id_{m\times m}$. 

Given a matrix $\bM\in\reals^{m\times m}$,  we denote by $\lambda_i(\bM)$, $i\in \{1,\dots,m\}$ its eigenvalues in decreasing order,
and by $\|\bM\|_{\op} = \max\{\lambda_1(\bM),-\lambda_m(\bM)\}$ its
operator norm. Finally, we shall occasionally consider third order tensors $\bT\in\reals^{m\times m\times m}$. In this case the operator (or injective)
norm is defined as $\|\bT\|_{\op} = \max\{ \vert \<\bT,\bx^{\otimes 3}\> \vert\; :\;\; \|\bx\|_2=1\}$, where $\<\bT,\bx^{\otimes 3}\>=\sum_{i,j,k} T_{ijk}x_ix_jx_k$.

We let $\Ball_q^d(\ba,\rho)\equiv\{\bx\in\reals^d:\; \|\bx-\ba\|_q\le \rho\}$ be the $\ell_q$ ball in $\reals^d$
with center $\ba$ and radius $\rho$. We will often omit the dimension superscript $d$ when clear from the context,
the subscript $q$ when $q=2$, and the center $\ba$ when $\ba = \bzero$. In particular $\Ball(\rho)$ is the euclidean ball of radius $\rho$. For any set $D \subset \R^d$, we let $\partial D$ be the boundary of the set.  

We will generally use upper case letters for random variables and lower case for deterministic values (unless the latter are matrices).

\section{Related literature}
\label{sec:Related}

While developing a theory on non-convex M-estimators is an outstanding challenge, several important facts are by now
well understood thanks to a stream of beautiful works. We will provide a necessarily incomplete summary in the next paragraphs.

\vspace{0.25cm}

\noindent\emph{Uniform convergence of the empirical risk.} Let $R(\btheta) = \E\hR_n(\btheta)$ denote the population risk. 
Under mild conditions on the loss function $\ell$ and on the sample size, it is known that with high probability
\begin{align}
\sup_{\btheta\in\Theta_{n,p}}\big|\hR_n(\btheta)-R(\btheta)\big|\le \eps_n\, ,\label{eq:StandardUniform}
\end{align}
for some small $\eps_n\to 0$ \cite{van2000applications,boucheron2013concentration}. 
This immediately implies guarantees for the M-estimator $\hbtheta$ in $\ell$-loss (or prediction error). 
Under additional conditions on the population risk $R(\btheta)$, bounds in estimation error can be derived as well.

For general non-convex losses, uniform convergence results of the form (\ref{eq:StandardUniform}) do not preclude the existence of multiple local minima 
of the sample risk $\hR_n(\btheta)$. Hence, this theory does not provide --by itself-- computationally practical methods to compute $\hbtheta$.

\vspace{0.25cm}

\noindent\emph{Algorithmic convergence to the `statistical neighborhood'.} In general, gradient descent and other local optimization procedures
are expected to converge to local minima of the empirical risk $\hR_n(\btheta)$. In several cases, it is proved that every local minimizer $\hbtheta^{\rm loc}$
is `statistically good'. More precisely, the estimation error  (e.g. the $\ell_2$ error  $\|\hbtheta^{\rm loc}-\btheta_0\|_2$) is within a constant from the minimax rate
for the problem at hand.
Also, gradient descent converges to such a neighborhood of the true $\btheta_0$ within a small number of iterations. Results of this type
have been proved, among others, for linear regression with noisy covariates \cite{loh2012high}, generalized linear models with non-convex regularizers 
\cite{loh2013regularized}, robust regression \cite{lozano2013minimum}, and sparse regression \cite{yang2015sparse}.

While these results are very important, they are not completely satisfactory. For instance, one natural question is whether the statistical error 
might be improved by finding a better local minimum. If, for instance, the estimation error could be improved by a factor $2$
by finding a better local minimum, it would be worth in many applications to restart gradient descent at multiple initializations.
 Also, since convergence to a fixed point is not guaranteed, these approaches come without a clear stopping criterion.
Finally these proofs make use of  the restricted strong convexity (RSC) 
assumption introduced \cite{negahban2012unified,loh2013regularized}, but do not provide any general tool to establish this condition.
In contrast, we prove uniform convergence results that can be used to ensure a condition similar to RSC. 

To the best of our knowledge, the only proof of unique local minimum of the regularized empirical risk
is obtained in a recent paper by Po-Ling Loh \cite{loh2015statistical}. This works assumes the linear regression model $y_i = \<\btheta,\bx_i\>+\eps_i$,
and establishes uniqueness for penalized regression with a certain class of bounded regularizers. This result is comparable to our Theorem \ref{thm:SparseRobust}, see
Section \ref{sec:Robust_VHD}, which uses $\ell_1$ regularization instead.  Note that, in  \cite{loh2015statistical}, the sample size is required to scale quadratically
in the sparsity: $n\gtrsim s_0^2$. Our proof technique is substantially different from the one of  \cite{loh2015statistical}, and we only require $n\gtrsim s_0\log d$.

\vspace{0.25cm}

\noindent\emph{Hybrid optimization methods.} It is often difficult to ensure global convergence to a minimizer of the sample risk
$\hR_n(\,\cdot\,)$ or even to a statistical neighborhood of the true parameters. Several papers develop  two-stage procedures to overcome this
problem. The first stage constructs a smart initialization $\hbtheta(0)$ that is within a certain large neighborhood of the true parameters.
Spectral methods are often used to implement this step. In the second stage, the estimate is refined by gradient descent (or another local procedure) 
initialized at $\hbtheta(0)$.  This general approach was studied in a number of problems including matrix completion 
\cite{keshavan2009matrix}, phase retrieval \cite{chen2015solving}, tensor decomposition \cite{anandkumar2015learning}.

In some cases, the local optimization stage is only proved to converge to a statistical neighborhood of $\btheta_0$, and hence this style of
analysis shares the shortcomings emphasized in the previous paragraph. In others, it is proven to converge to a single point. 
Further, in practice, the smart initialization is often not needed, and descent algorithms converge from random initialization as well.
Finally, as mentioned above, these analyses are typically carried on in a case-by-case manner.

\section{Uniform convergence results}

In this section we develop our key tools, that are uniform convergence results  on the gradient and Hessian of
the empirical risk. We also establish some of the direct implications of our results. 
Throughout, the data consists of the i.i.d. random variables $\{\bZ_1,\dots,\bZ_n\}$. We will use
$\{\bz_1,\dots,\bz_n\}$ if we want to refer to the corresponding realization.
The empirical risk is defines by Eq.~(\ref{eq:Risk}) and the corresponding population risk is  $R(\btheta) = \E\hR_n(\btheta) = \E\ell(\btheta;\bZ)$. 
The true parameter vector $\btheta_0$ satisfies the condition $\nabla R(\btheta_0) = \E[ \nabla \ell(\btheta_0; \bZ)] = \bzero$.

We consider two regimes, a \emph{high dimensional regime} in which the number of parameters $p$ is allowed to diverge 
roughly in proportion with the number of samples $n$, and a \emph{very high-dimensional regime} in which the true parameters' vector $\btheta_0$
is sparse and the number  of parameters $p$ can be much larger than $n$.
We treat these two cases separately because the theory is simpler and more general in the first regime.

\subsection{High-dimensional regime}
\label{sec:HighDim}

In order to avoid technical complications, we will limit optimization to a bounded set, i.e.
we will let  $\Theta_{n,p} =\Ball^p(\rad)\equiv \{\btheta \in\reals^p,\; \|\btheta\|_2\le \rad\}$ to be the Euclidean ball in $p$ dimensions.

We begin by stating our assumptions. Assumptions \ref{ass:GradientSub} and \ref{ass:Hessianub} below quantify the amount of statistical noise 
in the gradient and Hessian of the loss function.
\begin{assumption}[Gradient statistical noise]\label{ass:GradientSub}
The gradient of the loss is  $\tau^2$-sub-Gaussian. Namely, for any $\blambda\in\reals^p$, and $\btheta\in\Ball^p(\rad)$
\begin{align}
\E\Big\{\exp\Big(\<\blambda,\nabla\ell(\btheta;\bZ)-\E [\nabla\ell(\btheta;\bZ)]\>\Big)\Big\}\le \exp\left(\frac{\tau^2\|\blambda\|_2^2}{2}\right)\, .
\end{align}
\end{assumption}
\begin{assumption}[Hessian statistical noise]\label{ass:Hessianub}
The Hessian of the loss, evaluated on a unit vector, is  $\tau^2$-sub-exponential. Namely, for any $\blambda\in\Ball^p(1)$, and $\btheta\in\Ball^p(\rad)$
\begin{align}
&\cZ_{\blambda,\btheta}\equiv \<\blambda,\nabla^2\ell(\btheta;\bZ) \blambda\>\, ,\\
&\E\left\{\exp\Big(\frac{1}{\tau^2}\big\vert\cZ_{\blambda,\btheta}-\E \cZ_{\blambda,\btheta}\big\vert\Big)\right\}\le 2\, .
\end{align}
\end{assumption}

Our third assumption requires the Hessian of the loss to be a Lipschitz function of the vector of parameters $\btheta$.
\begin{assumption}[Hessian regularity]\label{ass:BoundedHessian}
The Hessian of the population risk is bounded at one point. Namely, there exists $\theta_*\in\Ball^p(\rad)$ and $\HUB$ such that
$\|\nabla^2 R(\btheta_*)\|_{\op}\le \HUB$. 

Further, the Hessian of the loss function is Lipschitz continuous with integrable Lipschitz constant. Namely, there 
exists $J_*$ (potentially diverging polynomially in $p$) such that
\begin{align}
J(\bz) \equiv& \sup_{\btheta_1\neq\btheta_2\in\Ball^p(\rad)} \frac{\big\|\nabla^2\ell(\btheta_1;\bz) -\nabla^2\ell(\btheta_2;\bz) \big\|_{\op}}{\|\btheta_1-\btheta_2\|_2} \, ,\\
&\E\big\{J(\bZ)\big\} \le J_*\, .
\end{align}
Further, there exists a constant  $\Ch$ such that $\HUB \leq \tau^2 p^{\Ch}$, $J_*\leq \tau^3 p^{\Ch}$. 
\end{assumption}
\begin{remark}
Note that $\nabla \ell$ has the same units\footnote{By this we mean that the two quantities behave in the same way under a rescaling of the parameters $\btheta$.} as $1/\rad$, and $\nabla^2 \ell$ has the same units as $1/\rad^2$. Thus, $\tau$ has the same units as $1/\rad$, $\HUB$ has the same units as $\tau^2$, and $J_*$ has the same units as $\tau^3$. This is the reason why we bound $\HUB$ and $J_*$ in the form as in Assumption \ref{ass:BoundedHessian}. In this way, $(\rad \cdot \tau)$ and $\Ch$ are dimensionless. 
\end{remark}

Discrete loss functions (e.g. the $0-1$ loss) are common within the statistical learning literature, but do not satisfy the above assumption because the gradient and Hessian are not defined everywhere. Note however that these can be well approximated by differentiable losses, with little --if any-- practical difference.

We are now in position to state our uniform convergence result. 
\begin{theorem} \label{thm:uniformconvergence1}
Under Assumptions \ref{ass:GradientSub}, \ref{ass:Hessianub}, and \ref{ass:BoundedHessian} stated above, 
there exists a universal constant $C_0$, such that letting $C = C_0\cdot (\Ch\vee\log(\rad\tau/\delta)\vee 1)$,
the following hold:
\begin{enumerate}
\item[$(a)$] The sample gradient converges uniformly to the population gradient in Euclidean norm. Namely, if $n\ge  C p\log p$, we have 
\begin{align}
\P \left(\sup_{\btheta\in\Ball^p(\rad)}\big\|\nabla \hR_n(\btheta)-\nabla R(\btheta)\big\|_2\le \tau \sqrt{\frac{C p\log n}{n}} \right) \geq 1 - \delta\, .
\end{align}
\item[$(b)$]  The sample Hessian converges uniformly to the population Hessian in operator norm. Namely, if $n\ge  C p\log p$, we have 
\begin{align}
\P \left( \sup_{\btheta\in\Ball^p(\rad)}\big\|\nabla^2 \hR_n(\btheta)-\nabla^2R(\btheta)\big\|_{\op}\le \tau^2 \sqrt{\frac{C p\log n}{n}} \right) \geq 1-\delta\, .
\end{align}
\end{enumerate}
\end{theorem}

The above theorem immediately implies that the structure of stationary points of the sample risk $\hR_n(\,\cdot\,)$ must reflect that of the population risk. In order to formalize this intuition, we introduce the notion of \emph{strongly Morse function}.
Given a differentiable function $F:\Ball^d(\rad)\to\reals$, we say that $\bx$ in the interior of the ball $\Ball^d(\rad)$ is
critical (or stationary) if $\nabla F(\bx) = 0$. 

Recall that a twice differentiable function $F:\reals^d\to\reals$ is Morse if  all its critical points are non-degenerate, i.e. have an invertible Hessian.
In other words $\nabla F (\bx)=0$ implies $\lambda_i(\nabla^2 F(\bx))\neq 0$ for all $i\in\{1,\dots, d\}$. Our next definition
provides a quantitative version of this notion.
\begin{definition}
We say that a twice differentiable function $F:\Ball^d(\rad)\to\reals$ is \emph{$(\eps,\eta)$-strongly Morse} if 
$\|\nabla F(\bx)\|_2>\eps$ for $\|\bx\|_2=\rad$ and, for any $\bx\in\reals^d$, $\|\bx\|_2<\rad$, the following holds
\begin{align}
\big\|\nabla F(\bx)\big\|_2\le \eps \;\;\; \Rightarrow\;\;\; \min_{i\in [d]}\big|\lambda_i\big(\nabla^2 F(\bx)\big)\big|\ge \eta\, .
\end{align}
\end{definition}
Note that, analogously to a Morse function on a compact domain, a strongly Morse function can have only a finite number of 
critical points which are in the interior of $\Ball^d(\rad)$. Also recall that the index of a non-degenerate critical point $\bx$ is the number of negative eigenvalues 
of the Hessian at $\bx$ (assuming $F$ to be twice differentiable).

\begin{theorem}\label{thm:morse}
Under  Assumptions \ref{ass:GradientSub}, \ref{ass:Hessianub}, and \ref{ass:BoundedHessian}, let $n\ge 4 C p\log n\cdot ((\tau^2/\eps^2) \vee (\tau^4 /\eta^2))$,
where $C=C(\tau^2,\delta,\rad,\Ch)$ is as in the statement of Theorem \ref{thm:uniformconvergence1}. Then the following happens with probability at least $1-\delta$.

If the population risk $R:\btheta\to R(\btheta)$ is $(\eps, \eta)$-strongly Morse in $\Ball^p(\rad)$, then the sample risk $\hR_n:\btheta \mapsto \hR_n(\btheta)$
is $(\eps/2, \eta/2)$-strongly Morse in $\Ball^p(\rad)$. Further  there is a one-to-one correspondence 
between the set of critical points of $R(\,\cdot\,)$, $\cC = \{\btheta^{(1)},\dots,\btheta^{(k)}\}$ and 
the set of critical points of $\hR_n(\,\cdot\,)$, $\cC_n= \{\hbtheta_n^{(1)},\dots,\hbtheta_n^{(k)}\}$  such that (letting 
$\hbtheta_n^{(j)}$ be the point in correspondence with $\btheta^{(j)}$, for any $j\in [k]$)
\begin{enumerate}
\item[$(a)$] The index of $\hbtheta_n^{(j)}$ coincides with the index of $\btheta^{(j)}$. (In particular, local minima correspond to local minima, and saddles to saddles.)
\item[$(b)$] If we further let $L = \sup_{\btheta \in \Ball^p(\rad)} \Vert \nabla^3 R(\btheta) \Vert_{\op}$,  and assume $n \geq 4 C p \log n /\eta_*^2$
where $\eta_*^2=(\eps^2/\tau^2) \wedge (\eta^2/\tau^4)\wedge (\eta^4/(L^2\tau^2))$,
we have, for each $j\in\{1,\dots,k\}$,
\begin{align}
\|\hbtheta_n^{(j)} - \btheta^{(j)}\|_2\le \frac{2 \tau}{\eta}\sqrt{\frac{C p\log n}{n}}\, .
\end{align}
\end{enumerate}
\end{theorem}

The strong Morse assumption imposes conditions on all the eigenvalues of the Hessian $\nabla^2 R(\btheta)$ at near-critical points, and implies
a detailed characterization of the empirical risk. If only weaker properties can be established for the population risk, Theorem \ref{thm:uniformconvergence1} 
can nevertheless be very useful. For instance, in Section \ref{sec:GaussianMixture} we consider an example  in which near critical points have a Hessian whose smallest
eigenvalue is either positive or negative, but in both cases bounded away from $0$. This weaker condition is sufficient to obtain a characterization of the local minima of the empirical risk.

\subsection{Very high-dimensional regime}
\label{sec:VeryHighDim}

In the  very-high dimensional regime $n\ll p$, we will solve the $\ell_1$-penalized risk minimization problem
\begin{equation}
\begin{aligned}\label{eq:RiskConstrained}
\mbox{\rm minimize}\;\;\; & \hR_n(\btheta) + \lambda_n \Vert \btheta \Vert_1 \, ,\\
\mbox{\rm subject to}\;\;\; & \|\btheta\|_2\le \rad\, . 
\end{aligned}
\end{equation}
We need some additional assumptions. It is fairly straightforward to check them in specific cases, see e.g. Section \ref{sec:Class}.
The first assumption is mainly technical, and not overly restrictive: it requires the loss function to have almost surely bounded gradient, in a suitable sense.
\begin{assumption}[Gradient bounds]\label{ass:BoundedGradient}
There exists a constant $T_*$ such that $\bZ$-almost surely, for all $\btheta \in \Ball^p_2(\rad)$,
\begin{align}
\big\| \nabla \ell(\btheta; \bZ) \big\|_\infty &\le  T_*\,.
\end{align}
\end{assumption}

Our key structural assumption is stated next. It requires the gradient of the loss function to depend on the parameters only through a linear function of $\btheta$,
possibly dependent on the feature vector $\bz$. Note that $\btheta_0$ is regarded here as fixed, and hence omitted from the arguments. 
\begin{assumption}[Generalized gradient linearity]\label{ass:GGL}
There exist functions $g:\reals\times\reals^d\to\reals$, $(t,\bz)\mapsto g(t; \bz)$ and $\bpsi:\reals^d\to\reals^p$, $\bz\mapsto \bpsi_{2}(\bz)$,  such that 
\begin{align}
\<\nabla \ell(\btheta; \bz), \btheta - \btheta_0 \> = g(\langle \btheta - \btheta_0, \bpsi(\bz) \rangle; \bz)\, . 
\end{align}
In addition, $g(t; \bz)$ is $L_*$-Lipschitz to its first argument, $g(0; \bz) = 0$, and $\bpsi(\bZ)$ is mean-zero and $\tau^2$-sub-Gaussian. 
\end{assumption}
As an example, in the case of  binary linear classification and robust regression, the data is given as a pair $\bz = (y, \bx)$, and there exists a function $f(t;\bz)$ such that $\nabla \ell (\btheta; \bz) = f(\langle \btheta - \btheta_0, \bx \rangle; \bz)\bx$. Assumption \ref{ass:GGL} is satisfied with $g(t;\bz) = t\, f(t;\bz)$ provided the latter is Lipschitz as a function of $t\in\reals$.

\begin{theorem} \label{thm:uniformconvergence2}
Under Assumptions \ref{ass:Hessianub}, \ref{ass:BoundedHessian}, \ref{ass:BoundedGradient} and \ref{ass:GGL} stated above, 
there exists a constant $C_1$ that depends on $(\rad, \tau^2, \Ch, \delta)$, and a universal constant $C_0$ such that letting $C_2 = C_0 \cdot (\Ch\vee\log(\rad\tau/\delta)\vee 1)$, the following hold:
\begin{enumerate}
\item[$(a)$]  The sample directional gradient converges uniformly to the population directional gradient,
along the direction $(\btheta-\btheta_0)$. Namely, we have 
\begin{align}
\P \left( \sup_{\btheta\in\Ball_2^p(\rad) \setminus \{ \bzero\}} \frac{\big \vert \big\langle \nabla \hR_n(\btheta)-\nabla R(\btheta), \btheta - \btheta_0 \big\rangle \big \vert}{\Vert \btheta - \btheta_0 \Vert_1 } \le (T_* + L_* \tau)  \sqrt{ \frac{  C_1 \log (np)}{n}} \right) \geq 1-\delta\, .
\end{align} 
\item[$(b)$] The sample restricted Hessian converges uniformly to the population restricted Hessian in the set $\Ball_2^p(\rad) \cap \Ball_0^p(s_0)$ for any $s_0 \leq p$. Namely, as $n \geq C_2 s_0 \log (np)$ we have 
\begin{align}
\P \left( \sup_{\btheta\in\Ball_2^p(\rad) \cap \Ball_0^p(s_0), \bv \in \Ball_2^p(1) \cap \Ball_0^p(s_0)} \left\vert \left\langle \bv, \left(\nabla^2\hR_n(\btheta)-\nabla^2 R(\btheta)\right) \bv\right \rangle \right\vert \le  \tau^2 \sqrt{\frac{  C_2 s_0 \log (np)}{n}} \right) \geq 1 -\delta\, .
\end{align}

\end{enumerate}
\end{theorem}

\section{Applications}
\label{sec:Applications}

\subsection{Binary linear classification: High dimensional regime}
\label{sec:Class}
\begin{figure}
\centering
\begin{subfigure}{0.45\linewidth}
\centering
\includegraphics[width=0.95\linewidth]{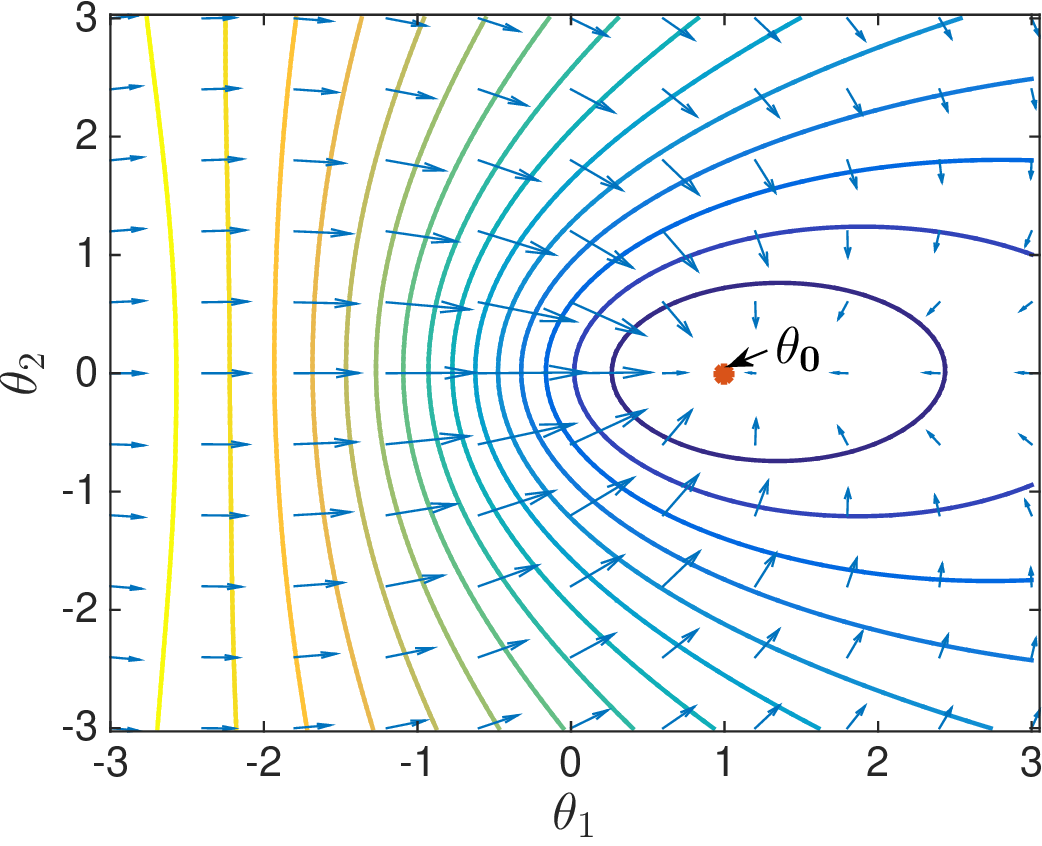}
\caption{}\label{fig:poprisk_bc}
\end{subfigure}\hspace{0.1cm}
\begin{subfigure}{0.45\linewidth}
\centering
\includegraphics[width=0.95\linewidth]{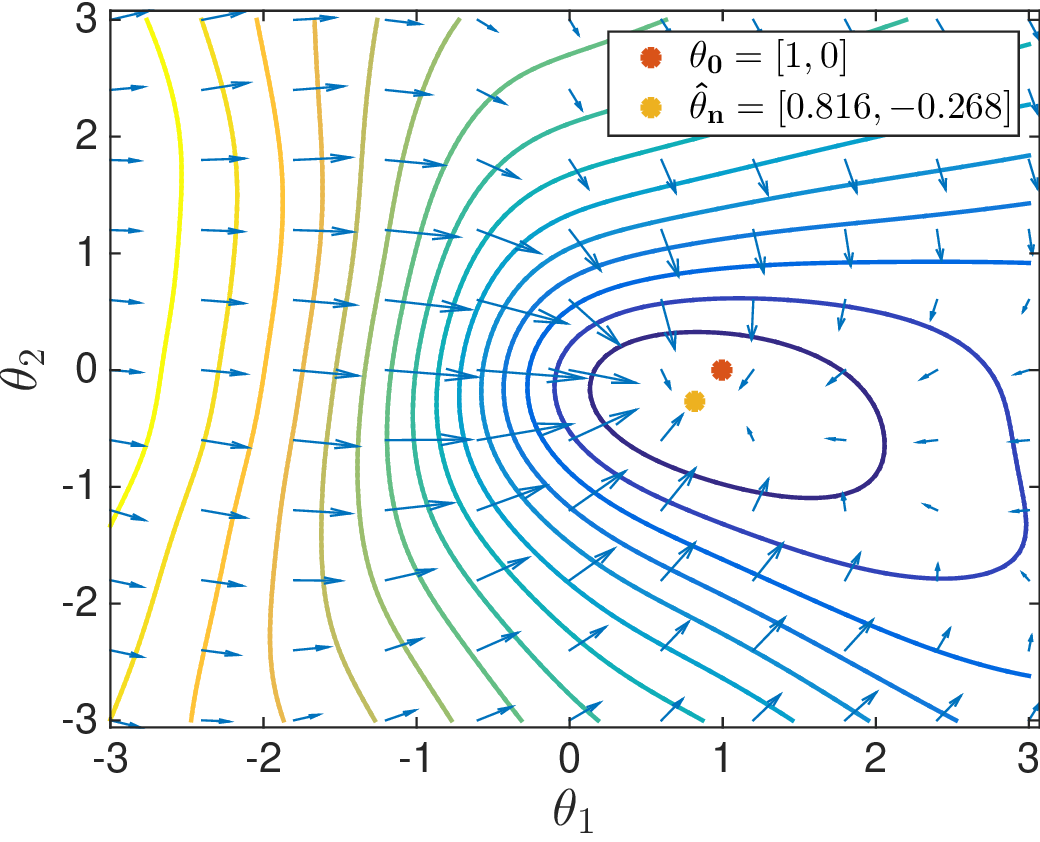}
\caption{}\label{fig:emprisk}
\end{subfigure}\label{fig:emprisk_bc}
\caption{Binary linear classification: $(a)$ Population risk for $d = 2$. $(b)$ A realization of the empirical risk for $d = 2$, and $n/d = 20$. }\label{fig:risk_bc}
\end{figure}

As mentioned in the introduction, in this case we are given $n$ pairs $\bz_1=(y_1,\bx_1), \dots, \bz_n=(y_n,\bx_n)$ with $y_i\in \{0,1\}$, $\bx_i\in\reals^d$,
whereby $\prob(Y_i=1|\bX_i=\bx) = \sigma(\<\btheta_0,\bx\>)$  (hence $p=d$ in this case).  We estimate $\btheta_0$ by minimizing the non-linear square loss
(\ref{eqn:erm}), which we copy here for the reader's convenience:
%
\begin{equation}
\begin{aligned}\label{eq:nonconvex1}
\mbox{\rm minimize} \;\;\;\ &\hR_n(\btheta)\equiv \frac{1}{n}\sum_{i=1}^{n} \Big(y_i - \sigma(\<\btheta,\bx_i\>)\Big)^2\,, \\
\mbox{\rm subject to} \;\;\; &\|\btheta\|_2\le \rad\, .
\end{aligned}
\end{equation}
This can be regarded as a smooth version of the $0-1$ loss.

We collect below the technical assumptions on this model.
\begin{assumption}[Binary linear classification]\label{ass:classification}
\begin{enumerate}
\item[$(a)$]  The activation $z\mapsto \sigma(z)$ is three times  differentiable with $\sigma'(z)>0$ for all $z$, and has bounded first, second and third derivatives. Namely,  for some constant $\actlip>0$:
\begin{equation}
\max\Big\{ \|\sigma'\|_\infty, \|\sigma''\|_\infty , \|\sigma'''\|_\infty \Big\} \le \actlip\, . \label{ass:activation}
\end{equation}
\item[$(b)$]
The feature vector $\bX$ has zero mean and is $\tau^2$-sub-Gaussian, that is
$\E[e^{\langle \blambda,\bX\rangle}] \le e^{\frac{\tau^2\|\blambda\|_2^2}{2}}$
for all $\blambda\in\R^d$.
\item[$(c)$] The feature vector  $\bX$ spans all directions in $\R^d$,
  that is, $\E[\bX\bX^{\sT}]\succeq \covlb \tau^2\id_{d\times d}$ for some $0 < \covlb < 1$.
\end{enumerate}
\end{assumption}
Assumption \ref{ass:classification}.$(a)$ is satisfied by many classical activation functions, a prominent example being the logistic
(or sigmoid) function $\sigma_L(z) = (1+e^{-z})^{-1}$.

Our main results on binary linear classification are summarized in the theorem below.
\begin{theorem}\label{thm:MainClassification}
Under Assumption \ref{ass:classification}, further assume $\|\btheta_0\|_2\le \rad/3$. There exist positive constants $C_1$, $C_2$ and $h_{\max}$ depending on parameters $(L_{\sigma},\rad,\tau^2,\covlb, \delta)$ and the activation function $\sigma(\cdot)$, but independent of $n$ and $d$, 
such that, if $n\ge C_1 d \log d$, the following hold with probability at least $1-\delta$:
\begin{itemize}
\item[$(a)$] The empirical risk function $\btheta\mapsto \hR_n(\btheta)$ has a unique local  minimizer 
in $\Ball^d(\bzero,\rad)$, that is the global minimizer $\hbtheta_n$.
\item[$(b)$] Gradient descent with fixed  step size $h_k=h \leq h_{\max}$ converges exponentially fast to the global minimizer, for any initialization $\btheta_s \in \Ball^d(\btheta_0, 2\rad/3)$:
$\|\hbtheta_n (k)-\hbtheta_n\|_2\le C_1\|\btheta_s-\hbtheta_n\|_2\, (1- h/C_1)^k$.  
\item[$(c)$] We have $\|\hbtheta_n-\btheta_0\|_2\le C_2 \sqrt{(d\log n)/n}$.
\end{itemize} 
\end{theorem}
The proof of this theorem can be found in Appendix \ref{app:Class_High}, and  is based on the following two-step strategy. First,
we study the population risk $R(\btheta)$, and establish its qualitative properties using  analysis.
Second, we use our uniform convergence result (Theorem \ref{thm:uniformconvergence1}) to prove that the same properties carry over to the sample risk $\hR_n(\btheta)$.
Figure \ref{fig:risk_bc}  presents a small numerical example that illustrates how the qualitative features of the population risk apply to the empirical risk as well.

A few remarks are in order.
First of all, the convergence rate of gradient descent (at point $(b)$) is 
\emph{independent of the dimension $d$ and number of samples $n$}.
In other words, $O(\log(1/\eps))$ iterations are sufficient to converge within distance $\eps$ from the global minimizer.
Classical theory of empirical risk minimization only concerns the statistical properties of the optimum, but does not 
provide efficient algorithms.

Next, note that our condition on the sample size $n$ is nearly optimal. Indeed, it is information-theoretically 
impossible to estimate $\btheta_0$  from less than $n<d$ binary samples.
Finally, the convergence rate at point $(c)$ also nearly matches the optimal (parametric) rate $\sqrt{d/n}$.

\subsection{Binary linear classification: Very high-dimensional regime}

As in the previous section, we are given $n$ pairs $\bz_1=(y_1,\bx_1), \dots, \bz_n=(y_n,\bx_n)$ with $y_i\in \{0,1\}$, $\bx_i\in\reals^d$,
and  $\prob(Y_i=1|\bX_i=\bx) = \sigma(\<\btheta_0,\bx\>)$. However $\btheta_0$ is assumed to be sparse, and the number of samples $n$
is allowed to be much smaller than the ambient dimension $d=p$. 
We adopt again the non-linear square loss (\ref{eqn:erm}), but now use a $\ell_2$-constrained $\ell_1$-regularized risk minimization,
as per Eq.~(\ref{eq:RiskConstrained}), which we rewrite here explicitly for the reader's ease
\begin{equation}
\begin{aligned}\label{eq:SparseClass_1}
\mbox{\rm minimize}\;\;\; & \frac{1}{n}\sum_{i=1}^{n} \Big(y_i - \sigma(\<\btheta,\bx_i\>)\Big)^2 + \lambda_n\|\btheta\|_1 \, ,\\
\mbox{\rm subject to}\;\;\; & \|\btheta\|_2\le \rad\, . 
\end{aligned}
\end{equation}

The very high-dimensional regime $d\gg n$ is of interest in many contexts.
In machine learning, the number of parameters $p$ can increase when a large number of additional features are added to the model
(for instance, nonlinear functions of an original set of features). In signal processing, $\btheta_0$ represents an unknown signal, of which we measure
noisy random linear projections $\<\bx_i,\btheta_0\>$, $i\in[n]$, quantized to \emph{one single bit}. This 
scenario is relevant to group testing \cite{atia2012boolean} 
and analog-to-digital conversion \cite{laska2011trust,laska2012regime},  and has been studied under the name of
`one-bit compressed sensing'; see \cite{plan2013one} and references therein.

In the very high-dimensional regime we need additional assumptions on the distribution of $\bX$ as well as the activation function $\sigma$.
\begin{assumption}[Fast-decaying activation]\label{ass:derivdecay}
	The activation function $\sigma$ satisfy $\sup_{t\in\R}\{\vert\sigma'(t)t\vert, \vert \sigma''(t)t\vert\} \le \derivdecay$ for some absolute constant $\derivdecay$.
\end{assumption}
\begin{assumption}[Continuous and bounded features]\label{ass:continuous}
	The feature vector $\bX$ has a density $p(\,\cdot\,)$ in $\R^d$, that is, $\P(\bX\in A) =\int_{A} p(\bx)\,\de\bx$ for all Borel sets $A\subseteq\R^d$. In addition, the feature vector is bounded: $\Vert \bX \Vert_\infty \leq M \tau$, and $\vert \langle \bX, \btheta_0 / \Vert \btheta_0 \Vert_2 \rangle \vert \leq M \tau$ almost surely. Here $M$ is a dimensionless constant greater than $1$. 
\end{assumption}

\begin{remark}
Assumption \ref{ass:derivdecay} holds popular examples of activation functions, such as the logistic $\sigma_L(z) = (1+e^{-z})^{-1}$ or probit $\sigma_{P}(z) = \Phi(z)$. 
For unbounded sub-Gaussian feature vectors, the next theorem can be supplemented by a truncation argument at level 
$M = C \sqrt{\log(nd)}$. Hence, the conclusions of this theorem hold,  with an additional $\log(nd)$ factor. 
\end{remark}

In the statement of the following theorem, for convenience, we will also assume $n \leq d^{100}$. This is a technical assumption so that we can bound $\log(nd) \leq 101 \log(d)$. And since we are considering the very high dimensional regime, it is not meaningful to discuss $n > d^{100}$.

\begin{theorem}\label{thm:SparseClass}
Under Assumptions \ref{ass:classification}, \ref{ass:derivdecay} and \ref{ass:continuous}, further assume $\|\btheta_0\|_0\le s_0$, $\|\btheta_0\|_2\le \rad/2$, and $n\leq d^{100}$. Then there exist constants $\Cn$, $\Cl$, $\Cs$, and $\eps_0$ depending on $(\actlip,\derivdecay,\rad,\tau^2, \covlb, \delta)$ and the activation function $\sigma(\cdot)$, but independent of $n$, $d$, $s_0$, and $M$, such that as $n\ge \Cn \, s_0\, \log d$ and $\lambda_n \geq \Cl M \sqrt{(\log d)/n}$, the following hold with probability at least $1-\delta$:
\begin{enumerate}
\item[$(a)$] Any stationary point of problem (\ref{eq:SparseClass_1}) is in $\Ball_2^d(\btheta_0, \Cs \sqrt{(M^2 s_0\log d)/n + s_0 \lambda_n^2})$.
\item[$(b)$]  As long as $n$ is large enough such that $n\ge \Cn \, s_0\, \log^2 d$ and $\Cs \sqrt{(M^2 s_0\log d)/n + s_0 \lambda_n^2} \leq \eps_0$, the problem has a unique local minimizer $\hbtheta_n$ which is also the global minimizer.
\end{enumerate}
\end{theorem}
As in the previous section, our proof makes a crucial use of the sparse uniform convergence result, Theorem \ref{thm:uniformconvergence2},
together with an analysis of the population risk.

\begin{remark}
Let us emphasize that Theorem \ref{thm:SparseClass} leaves open the existence of a fast algorithm to find the global optimizer
$\hbtheta_n$. However \cite[Theorem 3]{nesterov2013gradient} implies that, by running $k$ steps of 
projected gradient descent, we can find an estimate $\hbtheta_n(k)$ which has a subgradient of order $O(1/k)$. 
While we expect this sequence to converge to $\hbtheta_n$, we defer this question to future work.
\end{remark}

Theorem  \ref{thm:SparseClass} establishes a nearly optimal upper bound on the $\ell_2$ estimation
error $\|\hbtheta_n-\btheta_0\|_2$.  Indeed this error is within a logarithmic factor from the error achieved by an oracle estimator that
is given the exact support of $\btheta_0$.
For comparison, \cite{plan2013one,plan2013robust}  proves $\|\hbtheta^{\mbox{\tiny \rm LP}}_n-\btheta_0\|_2
\lesssim (s_0/n)^{1/4}(\log p/s_0)^{1/4}$ for a linear programming formulation, under the
more restrictive assumption of Gaussian feature vectors $\bx_i\sim\normal(\bzero,\id_{d\times d})$. This analysis was generalized in \cite{ai2014one}
to feature vectors with i.i.d. entries, although  with the same estimation error bound. The optimal rate 
$\|\hbtheta^{\mbox{\tiny \rm cvx}}_n-\btheta_0\|_2\lesssim (s_0/n)\log(p/s_0)$
was obtained only recently in  \cite{plan2014high}, again for standard Gaussian feature vectors. 

Let us finally emphasize that the estimator defined here uses a bounded loss function and is potentially more robust to outliers than  other approaches that use a convex loss (e.g. logistic loss).

\subsection{Robust regression: High-dimensional regime}

In robust regression we are given data $\bz_1=(y_1,\bx_1), \dots, \bz_n=(y_n,\bx_n)$ with $y_i\in \reals$, $\bx_i\in\reals^d$,
and we assume the linear model $y_i =  \<\btheta_0,\bx_i\> + \eps_i$, where the noise terms $\eps_i$ are i.i.d. 
with mean zero. Also in this case we have $p=d$. We use the loss (\ref{eqn:erm-robust}), which we copy here for the reader's convenience:
\begin{equation}
\begin{aligned}\label{eq:RobustDefinition}
\mbox{\rm minimize} \;\;\; & \frac{1}{n}\sum_{i=1}^{n} \rho\big(y_i - \<\btheta,\bx_i\>\big) \, , \\
\mbox{\rm subject to} \;\;\; & \|\btheta\|_2\le \rad\, .
\end{aligned}
\end{equation}
Classical choices for loss function $t\mapsto \rho(t)$ are the Huber loss \cite{huber1973robust} which is convex 
with $\rho_{\mbox{\rm\tiny Huber}}(t) = |t|-{\rm const.}$ for $t$ large enough, and Tukey's bisquare loss, which is bounded and defined as
\begin{align} 
\rho_{\mbox{\rm\tiny Tukey}}(t) = \begin{cases}
1-\big(1-(t/t_0)^2)^3\, \;\;& \mbox{ for $|t|\le t_0$,}\\
1  \;\;& \mbox{ for $|t|\ge t_0$.}
\end{cases}\label{eq:Tukey}
\end{align}
It is common to define the associated score function as $\psi(t)  =\rho'(t)$.

We next formulate our assumptions. 
\begin{assumption}[Robust regression]\label{ass:Robust}
\begin{enumerate}
\item[$(a)$]  The score function $z\mapsto \psi(z)$ is twice differentiable and odd in $z$
with $\psi(z)\ge 0$ for all $z\ge 0$, and has bounded zero, first, and second derivatives. Namely,  for some constant $\rholip>0$:
\begin{equation}
\max\Big\{ \|\psi\|_\infty, \|\psi'\|_\infty, \|\psi''\|_\infty \Big\} \le \rholip\, . \label{ass:score}
\end{equation}
\item[$(b)$] The noise $\eps$ has a symmetric distribution, i.e. is such that $\eps$ is distributed as $-\eps$.
Further, defining $g(z) \equiv \E_{\eps}\{\psi(z+\eps)\}$ we have $g(z)>0$ for all $z>0$, as well as $g'(0) > 0$. 
\item[$(c)$]
The feature vector $\bX$ has zero mean and is $\tau^2$-sub-Gaussian, that is
$\E[e^{\langle \blambda,\bX\rangle}] \le e^{\frac{\tau^2\|\blambda\|_2^2}{2}}$
for all $\blambda\in\R^d$.
\item[$(d)$] The feature vector  $\bX$ spans all directions in $\R^d$,
  that is, $\E[\bX\bX^{\sT}]\succeq \covlb \tau^2\id_{d\times d}$ 
for some $0 < \covlb < 1$.
\end{enumerate}
\end{assumption}

Note that the condition $g(z) \equiv \E_{\eps}\{\psi(z+\eps)\}>0$ for all $z>0$ and $g'(0) > 0$ are quite mild, and holds --for instance-- 
if the noise has a density that is strictly positive for all $\eps$, and decreasing for $\eps>0$. 
\begin{theorem}\label{thm:MainRobustRegression}
Under Assumption \ref{ass:Robust}, further assume $\|\btheta_0\|_2\le \rad/3$. Then 
there exist positive constants $C_1$, $C_2$ and $h_{\max}$ depending on parameters $(\rholip, \rad, \tau^2, \covlb, \delta)$, the loss function $\rho(\cdot)$, and the law of noise $\P_\eps$ but independent of $n$ and $d$, such that as $n\ge C_1 d\log d$, the robust regression estimator satisfies the following with probability at least $1-\delta$:
\begin{enumerate}
\item[$(a)$] The empirical risk function $\bw\mapsto \hR_n(\btheta)$ has a unique local minimizer 
in $\Ball^d(\rad)$, that is the global minimizer $\hbtheta_n$.
\item[$(b)$] Gradient descent with fixed  step size $h_k=h \leq h_{\max}$ converges exponentially fast to the global minimizer, for any initialization $\btheta_s \in \Ball^d(\btheta_0, 2\rad/3)$:
$\|\hbtheta_n (k)-\hbtheta_n\|_2\le C_1\|\btheta_s-\hbtheta_n\|_2\, (1- h/C_1)^k$.  
\item[$(c)$] We have $\|\hbtheta_n-\btheta_0\|_2\le C_2 \sqrt{(d\log n)/n}$.
\end{enumerate}
\end{theorem}

\subsection{Robust regression: Very high-dimensional regime}
\label{sec:Robust_VHD}

As in the previous section, we are given $n$ pairs $\bz_1=(y_1,\bx_1),\dots, \bz_n=(y_n,\bx_n)$ with $y_i\in \reals$, $\bx_i\in\reals^d$,
and we assume the linear model $y_i =  \<\btheta_0,\bx_i\> + \eps_i$, where the noise terms $\eps_i$ are i.i.d. 
with mean zero. However $\btheta_0$ is assumed to be sparse, while the number of samples $n$
is much smaller than the ambient dimension $d=p$. 
We adopt again the loss (\ref{eqn:erm-robust}), but now use a $\ell_2$-constrained $\ell_1$-regularized risk minimization,
as per Eq.~(\ref{eq:RiskConstrained}), which we rewrite here explicitly for the reader's ease
\begin{equation}
\begin{aligned}\label{eq:SparseClass_rho}
\mbox{\rm minimize}\;\;\; & \frac{1}{n}\sum_{i=1}^{n} \rho\big(y_i - \<\btheta,\bx_i\>\big) + \lambda_n\|\btheta\|_1 \, ,\\
\mbox{\rm subject to}\;\;\; & \|\btheta\|_2\le \rad\, . 
\end{aligned}
\end{equation}

Like the case of very high dimensional binary classification, we also need continuous and bounded feature assumptions, i.e. Assumption \ref{ass:continuous}, and need a fast decaying assumption on $\psi = \rho'$.
\begin{assumption}[Fast-decaying score function]\label{ass:derivdecayrho}
The score function $\psi$ satisfies $\sup_{t\in\R}\{\vert \psi(t)t \vert\} \le \derivdecayrho$ for some absolute constant $\derivdecayrho$.
\end{assumption}

\begin{theorem}\label{thm:SparseRobust}
Under Assumptions \ref{ass:classification},  \ref{ass:continuous} and \ref{ass:derivdecayrho}, further assume $\|\btheta_0\|_0\le s_0$, $\|\btheta_0\|_2\le \rad/2$, and $n\leq d^{100}$. Then there exist constants $\Cn$, $\Cl$, $\Cs$, and $\eps_0$ depending on $(\rholip,\derivdecayrho,\rad,\tau^2, \covlb, \delta)$, the loss function $\rho$, and the law of noise $\P_\eps$, but independent of $n$, $d$, $s_0$ and $M$, such that as $n\ge \Cn \, s_0\, \log d$ and $\lambda_n \geq \Cl M \sqrt{(\log d)/n}$, the following hold with probability at least $1-\delta$:
\begin{enumerate}
\item[$(a)$] Any stationary point of problem (\ref{eq:SparseClass_rho}) is in $\Ball_2^d(\btheta_0, \Cs \sqrt{(M^2 s_0\log d)/n + s_0 \lambda_n^2})$.
\item[$(b)$]  As long as $n$ is large enough such that $n\ge \Cn \, s_0\, \log^2 d$ and $\Cs \sqrt{(M^2 s_0\log d)/n + s_0 \lambda_n^2} \leq \eps_0$, the problem has a unique local minimizer $\hbtheta_n$ which is also the global minimizer.
\end{enumerate}
\end{theorem}
The proof of this theorem is almost the same as the proof of Theorem \ref{thm:SparseClass}. We will omit the proof to avoid redundancies.

\subsection{Gaussian mixture model}
\label{sec:GaussianMixture}

In the applications considered so far, the population risk has a unique 
stationary point which is also the global minimum. We used our uniform convergence theorems to prove that the empirical
risk has the same property and hence can be optimized efficiently. 

In order to illustrate our approach  on an example with multiple local minima, we consider clustering within a 
simple Gaussian mixture model. We are given data points $\bz_1,\dots,\bz_n\in\reals^d$, with $\bz_i$ drawn from a mixture
of two Gaussians, in equal proportions,
$\bz_i \sim (1/2)\normal(\btheta_{0,1},\id_{d\times d})+(1/2)\normal(\btheta_{0,2},\id_{d\times d})$. Define the separation parameter $D = \Vert \btheta_{0,2} - \btheta_{0,1}\Vert_2/2$. We want to estimate the centers $\btheta_{0,1}$, $\btheta_{0,2}$ by solving the maximum likelihood problem (here $\btheta = (\btheta_1, \btheta_2) \in \R^{2d}$)
\begin{equation}\label{eqn:erm-mixture}
\begin{aligned}
\mbox{\rm minimize}&\;\;\; \hR_n(\btheta)\equiv -\frac{1}{n}\sum_{i=1}^{n} \log\left( \sum_{a=1}^2\; \phi_d\big(\bz_i-\btheta_a\big)\right).\\
\end{aligned}
\end{equation}

In this case, the population risk has at least two global minima related by the symmetry under exchange of the two components:
$\btheta_+ = (\btheta_{0,1},\btheta_{0,2})$ and  $\btheta_- = (\btheta_{0,2},\btheta_{0,1})$, as well as a saddle point $\btheta_s =
((\btheta_{0,1} + \btheta_{0,2})/2,(\btheta_{0,1} + \btheta_{0,2})/2)$. This is a common phenomenon: symmetries lead to
multiple minima of the risk function. In a recent paper, Xu, Hsu and Maleki \cite{xu2016global} prove that 
these are the only critical points. See Figure \ref{fig:risk_gmm} for an illustration. 

\begin{figure}
\centering
\begin{subfigure}{0.45\linewidth}
\centering
\includegraphics[width=0.95\linewidth]{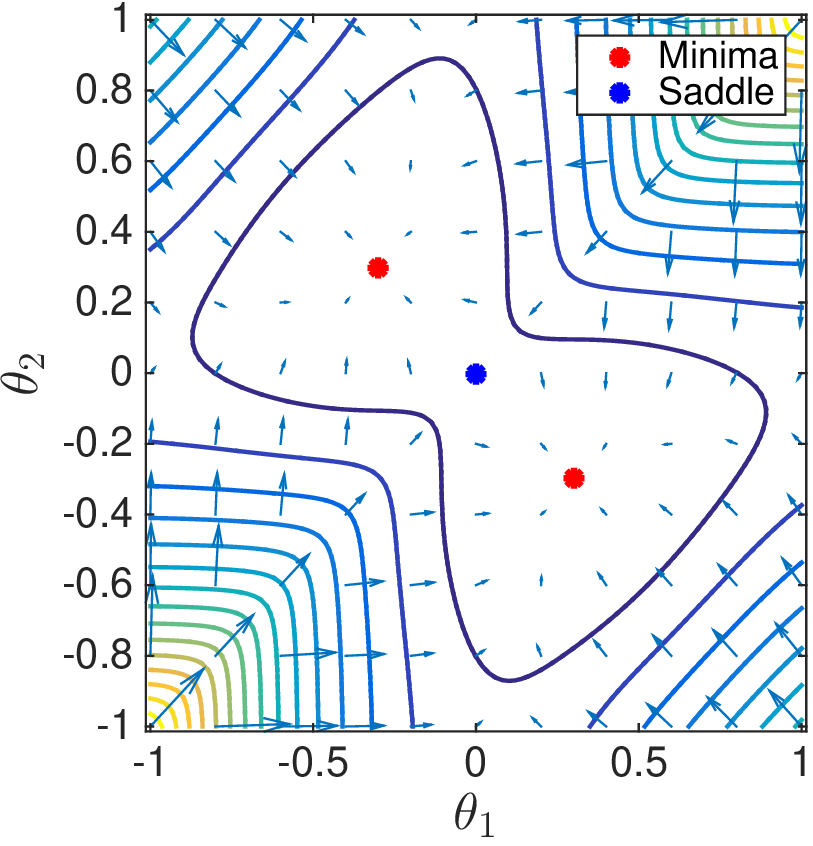}
\caption{}\label{fig:poprisk_gmm}
\end{subfigure}\hspace{0.1cm}
\begin{subfigure}{0.45\linewidth}
\centering
\includegraphics[width=0.95\linewidth]{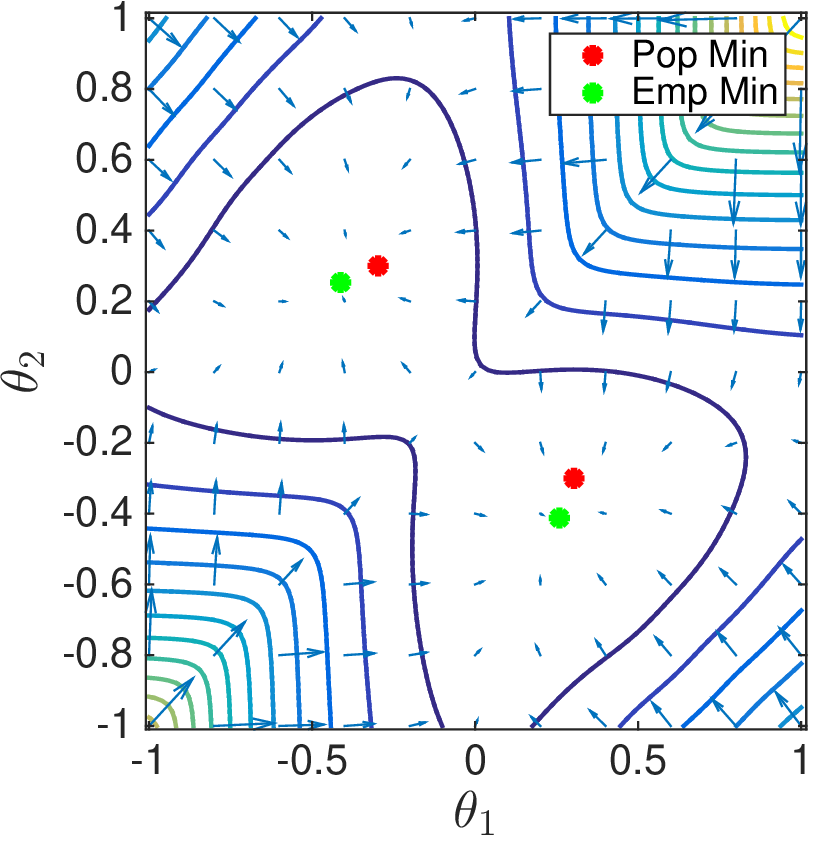}
\caption{}\label{fig:emprisk_gmm}
\end{subfigure}
\caption{Gaussian mixture model: $(a)$ Population risk for $d = 1$. $(b)$ A realization of the empirical risk for $d = 1$, and $n = 30$. }\label{fig:risk_gmm}
\end{figure}

\begin{theorem}\label{thm:Mixture}
Let $\hR_n(\btheta)$ be the empirical risk for an equal-proportion mixture of two Gaussians. Then there exist constants 
$C_1$, $C_2$, and $C_3$ depending on $(D,\delta)$ but independent of $n$ and $d$, such that
as $n\ge C_1\, d\log d$,
the following holds with probability at least $1-\delta$:
\begin{enumerate}
\item[$(a)$] In side $\Ball^{2d}(\btheta_s, C_2)$, the empirical risk has exactly two local minima $\hbtheta_+$, $\hbtheta_-$ related by an exchange of the two classes.
\item[$(b)$] For any initialization $\hbtheta_0 \in \Ball^{2d}(\btheta_s, C_2)$, the trust region algorithm will converge to one of the local minima. 
\item[$(c)$] The local minima satisfy 
\begin{align}
\Vert \hbtheta_+ - \btheta_+ \Vert_2 \leq C_3 \sqrt{\frac{d \log n}{n}},\quad \Vert \hbtheta_- - \btheta_- \Vert_2 \leq C_3 \sqrt{\frac{d \log n}{n}}.
\end{align}
\end{enumerate}
\end{theorem}

\section{Numerical experiments}

We carried out extensive numerical experiments in order to verify how accurate is our theory. 
Sections  \ref{sec:NumericalClass} to \ref{sec:MixtureSimulations} present simulations for the models studied
in Section \ref{sec:Applications}. Sections \ref{sec:AustralianData} and \ref{sec:ColonCancerData} present illustrations with real data.

\subsection{Binary linear classification: high-dimensional regime }
\label{sec:NumericalClass}

Figures \ref{fig:success_rate}, \ref{fig:error_bc}, \ref{fig:conv_bc}, \ref{fig:gditer} report our results 
for the non-convex binary classification model of Section \ref{sec:Class}. 
\begin{figure}
\centering
\begin{subfigure}{0.45\textwidth}
\centering
\includegraphics[width=0.95\linewidth]{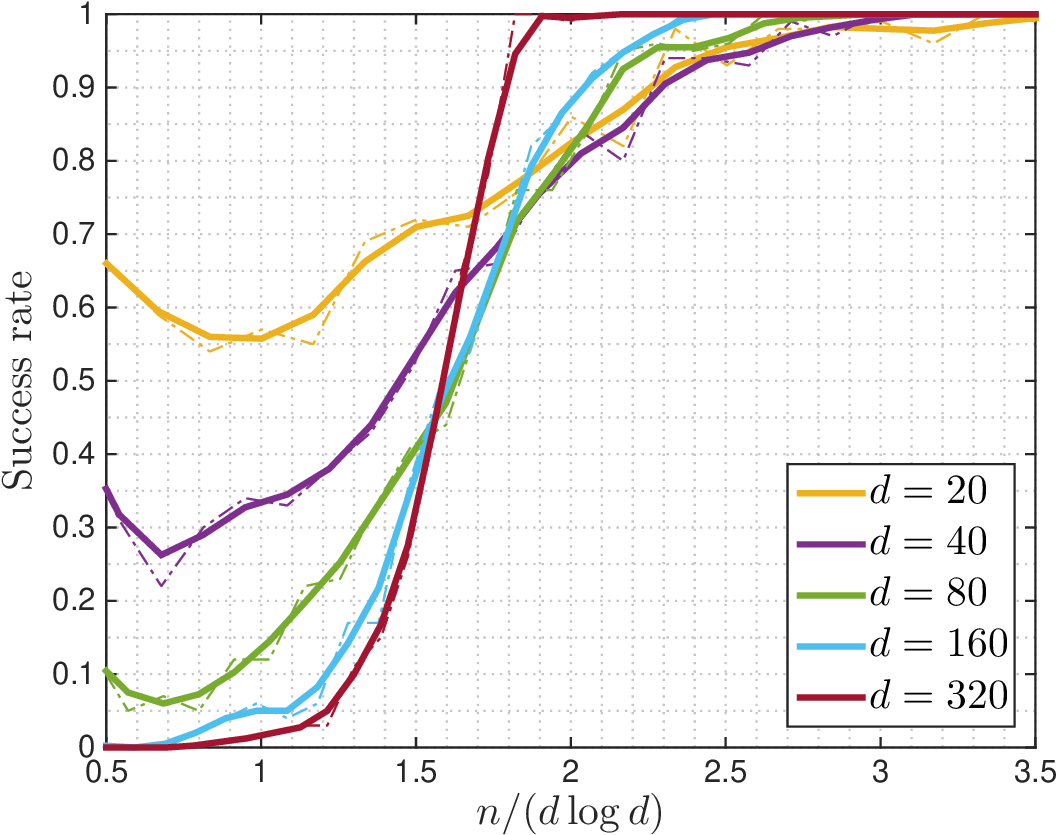}
\caption{}\label{fig:success_rate}
\end{subfigure}
\begin{subfigure}{0.45\textwidth}
\centering
\includegraphics[width=0.95\linewidth]{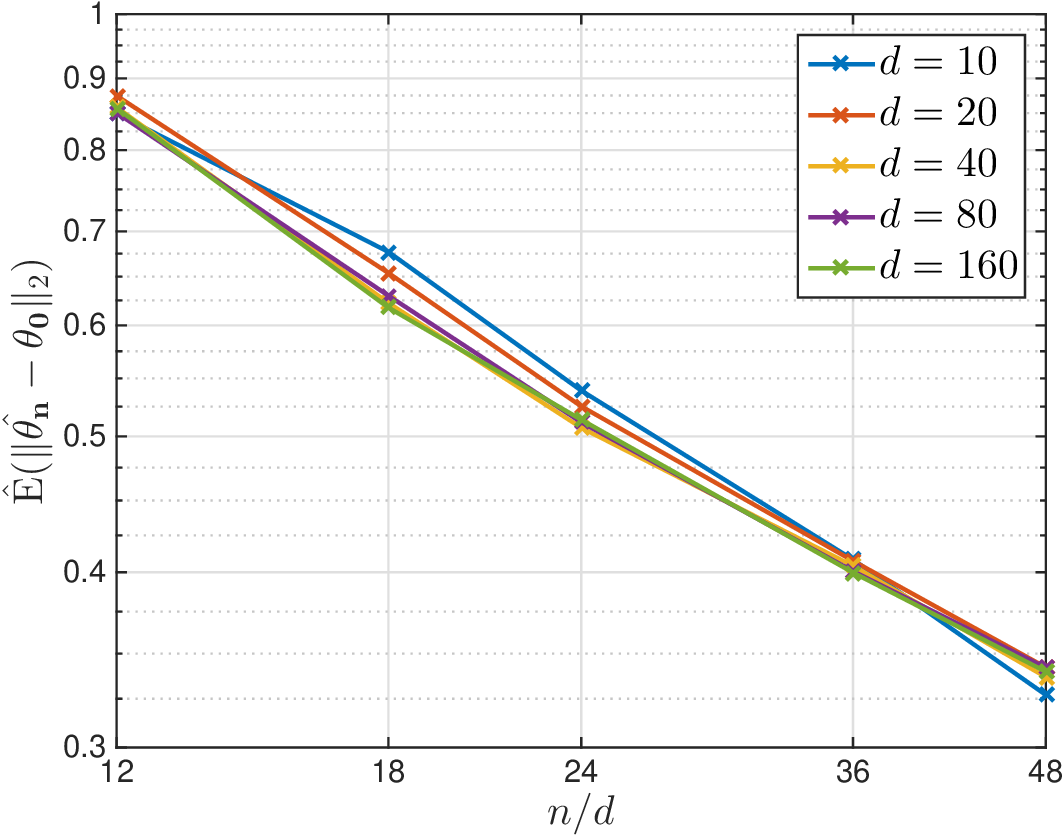}
\caption{}\label{fig:error_bc}
\end{subfigure}
\caption{Binary linear classification, high dimensional: $(a)$ Success rate versus $n/(d\log d)$ for several ambient dimensions $d$,  with
 $\Vert \btheta_0 \Vert_2=3$ (dashed lines are empirical averages, continuous lines are a smoothed version); $(b)$ Estimation error $\widehat{\rm E} [\Vert \hbtheta_n - \btheta_0 \Vert_2]$ versus $n/d$,
for  $\Vert \btheta_0 \Vert_2 = 1$. }
\end{figure}

We consider i.i.d. predictors $\bX_i \sim \normal(\bzero, \id_{d\times d})$,  and generate labels $Y_i\in\{0,1\}$ with
$\prob(Y_i=1|\bX_i=\bx)= \sigma(\<\btheta_0,\bx\>)$ where  $\sigma(u) = \sigma_{L}(u) = (1+e^{-u})^{-1}$ is the logistic activation. 
We perform gradient descent, cf. Eq~(\ref{eq:GradientDescent}) to
minimize the empirical risk (\ref{eqn:erm}), with a minor revision in practice: we will project the 
points back into $\Ball^d(\rad)$ if the iteration points fall out of the ball, with $\rad = 3 \Vert \btheta_0 \Vert_2$. The step size is fixed to be $h = 1$. 

In order to test the hypothesis that the landscape is simple (i.e. it has a unique local minimum), 
we run projected gradient descent starting from multiple random initializations $\btheta_s \sim \normal(\bzero, \id_{d\times d}/d)$. If the landscape is simple, 
we expect the iterates $\hbtheta_n(k)$ to converge to the same global minimizer with no dependence on the initialization.
 If the landscape is rough, projected gradient descent will converge to different points depending on the initialization. 
Given a maximum number of iterations $k_{\max}$,  we define the following quantity, depending 
on the data $(\bY,\bX)\equiv \{(Y_i,\bX_i)\}_{1\le i\le n}$, 
\begin{align}
S_{\bY,\bX}= \sqrt{{\rm Tr }(\widehat{{\rm Var}}_{\text{init}}( \hbtheta_n(k_{\max}) \vert \bY,\bX))}\, ,
\end{align}
where the variance is taken over the random initializations $\btheta_s$.
In words, $S_{\bY,\bX}$ is the spread of the limit points of projected gradient descent, for the instance $(\bY,\bX)$. 
We then define the empirical success probability as
\begin{align}
\Psucc \equiv \widehat{\prob}(S_{\bY,\bX}\le \eps)\, .
\end{align}

In Figure \ref{fig:success_rate}, we plot our results for the empirical success rate, for several values of $n$, $d$. In this experiment, we take $\Vert \btheta_0 \Vert_2 = 3$. 
For each pair $(n,d)$, we generate $100$ instances $(Y_i,\bX_i)$ and run projected 
gradient descent from $10$ random initializations. We use $k_{\max}=10^4$ iterations and tolerance $\eps = 10^{-2}$ though results seem to be fairly insensitive
 to these parameters. For each dimension $d$, the success rate goes rapidly from $0$ to $1$ as the number of samples $n$ crosses a threshold. 
We plot the success probability as function of the rescaled number of samples $n/(d\log d)$. On this scale, curves for different dimension 
cross each other, and become steeper as $d$ increases. This is consistent with Theorem \ref{thm:MainClassification}. This also suggests a sharp phase transition at $n_*(d)$ which is roughly of order $d\log d$. It is a fascinating open question whether a sharp threshold actually exists\footnote{When convergence to a single global minimum fails, we observe that often 
projected gradient actually convergence to the boundary of $\Ball^d(\rad)$.}. 

Figure \ref{fig:error_bc} illustrates the behavior of the estimation error $\|\hbtheta_n-\btheta_0\|_2$ achieved by gradient descent. 
In all the following experiments, we will take $\Vert \btheta_0 \Vert_2 = 1$. We plot the estimation error (averaged over $100$ random instances) 
$\widehat{\rm E} [ \|\hbtheta_n-\btheta_0\|_2]$ versus  $n/d$.
Curves for different dimensions collapse, and are consistent with the optimal  rate $\|\hbtheta_n-\btheta_0\|_2 = \Theta(\sqrt{d/n})$.

\begin{figure}
\begin{subfigure}{0.45\linewidth}
\centering
\includegraphics[width=0.95\linewidth]{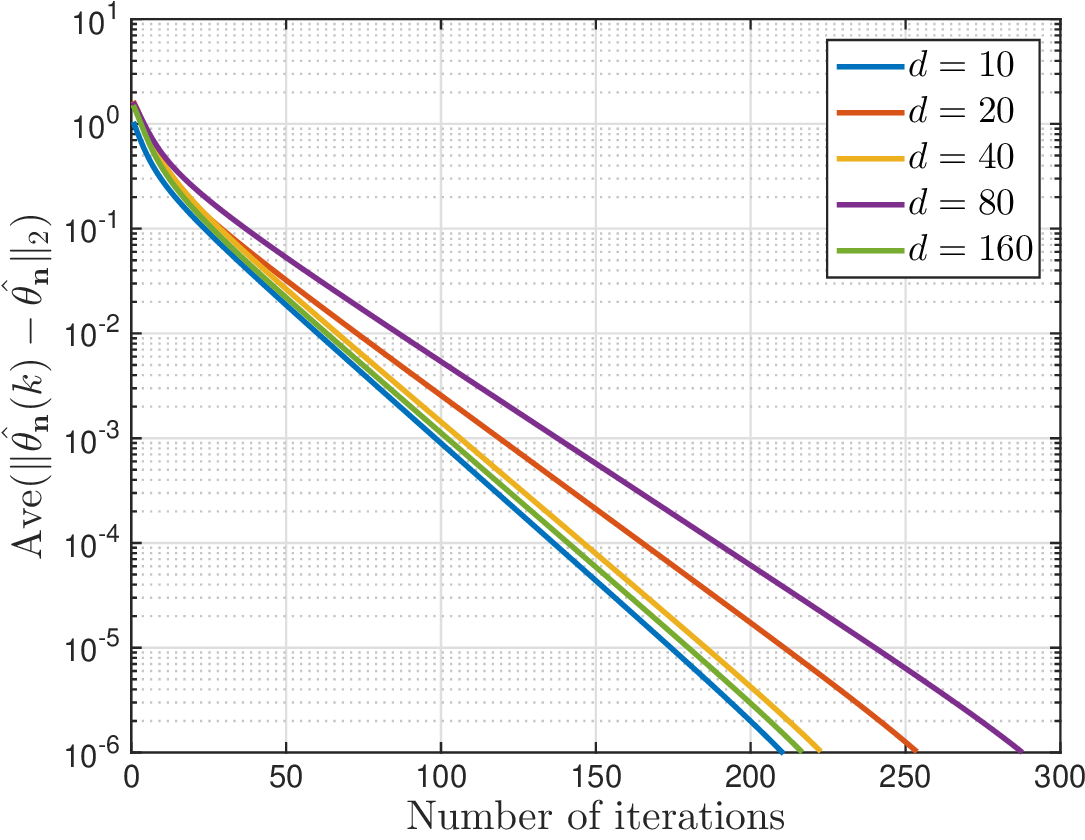}
\caption{}\label{fig:conv_bc}
\end{subfigure}\hspace{0.1cm}
\begin{subfigure}{0.45\linewidth}
\centering
\includegraphics[width=0.95\linewidth]{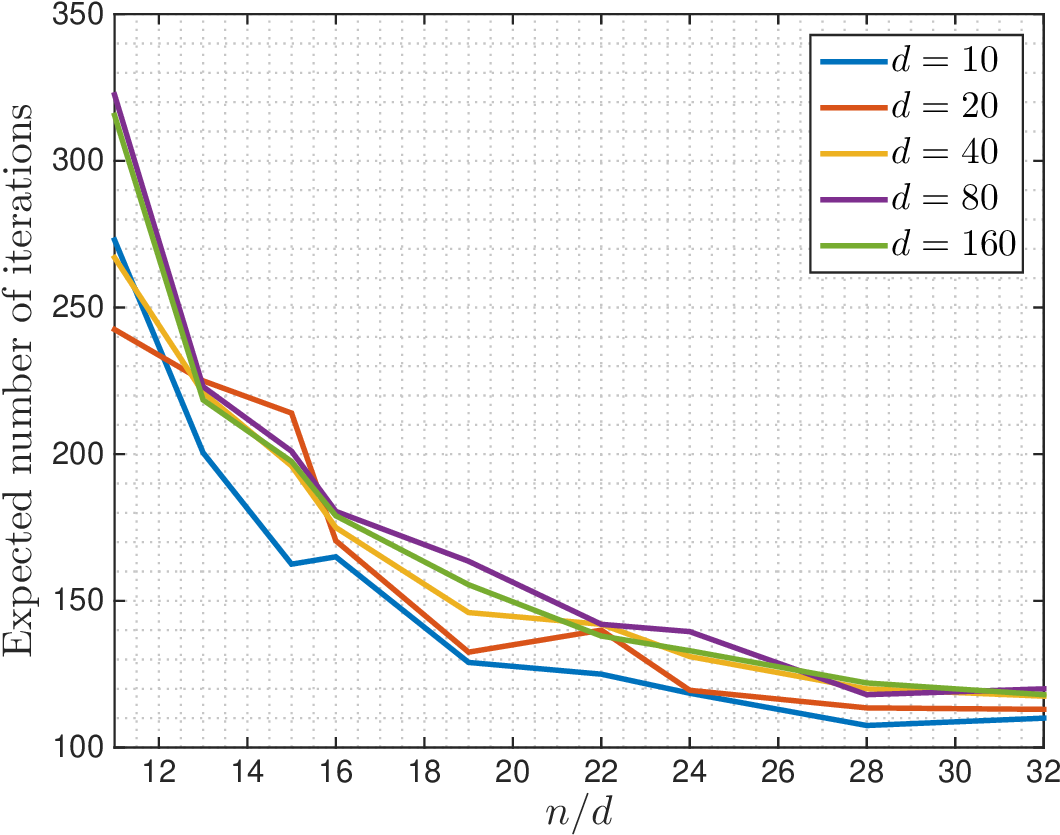}
\caption{}\label{fig:gditer}
\end{subfigure}
\caption{Binary linear classification, high dimensional: $(a)$ The convergence of the gradient descent algorithm.  Here
$\Vert \btheta_0 \Vert_2 = 1$, $n/d = 20$. The y-axis is on a log-scale; $(b)$ Minimum number of iterations needed to achieve average distance $10^{-4}$ from the global optimizer. }
\end{figure}

Figure \ref{fig:conv_bc} shows the convergence of gradient descent for several values of $n$ and $d$, for fixed $n/d = 20$. Namely, we plot the distance from the global minimizer as a function of the number of iterations $k$, estimated using $100$ realizations $(\bY,\bX)$. Since there is a  small probability that gradient descent fails to find unique minimizer, we average the distance from the global 
minimizer
over the results between the $(0.05, 0.95)$ quantiles of these $100$ instances. Convergence to the global minimizer appears to be exponential as predicted by Theorem \ref{thm:MainClassification}. Also,  convergence is fairly independent
of the dimension for fixed $n/d$.  

Finally, Figure \ref{fig:gditer} shows the number of iterations needed to achieve the $\varepsilon = 10^{-4}$ optimization error. We run $100$ instances, and 
we plot the expected number of iteration, by averaging the results between the $(0.05, 0.95)$ quantiles of these $100$ instances. When $n/d$ is small, the landscape is 
not very smooth, and convergence is slower. When $n/d$ grows, the number of iterations decreases and converges  to a constant. 
This is also predicted by Theorem \ref{thm:MainClassification}: the landscape of empirical risk will be as smooth as the landscape of population risk, as 
$n \geq C\, d\log d$. 

\subsection{Binary linear classification: very high-dimensional regime }

In Figures \ref{fig:var_bcl1}, \ref{fig:conv_bcl1}, \ref{fig:error_bcl1}, we present our results on non-convex binary linear classification in the very high-dimensional regime.  
Data $(Y_i,\bX_i)$ were generated as in the previous section, with $\btheta_0$ a vector $k$ non-zero entries all of size $1/\sqrt{k}$. 
We use proximal gradient descent to solve problem (\ref{eq:RiskConstrained}) with $\rad=10$.

\begin{figure}
\centering
\includegraphics[width=0.6\linewidth]{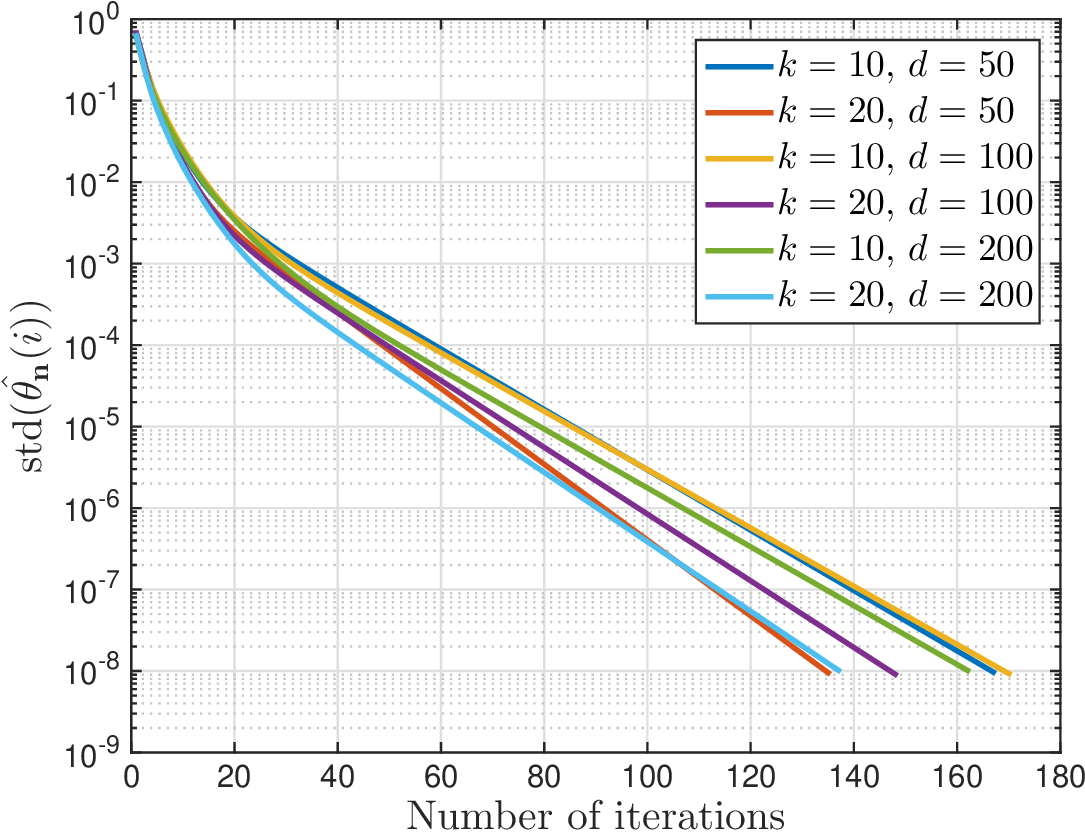}
\caption{Binary linear classification, very high-dimensional. The standard deviation of each iteration point with 
  respect to random initialization.
}\label{fig:var_bcl1}
\end{figure}
In Figure \ref{fig:var_bcl1}, we use random initializations $\btheta_s\sim \normal(\bzero,\id_{d\times d}/d)$, and plot
the empirical standard deviation of the resulting iterates ${\rm std}(\hbtheta_n(i)) = \Trace(\hVar(\hbtheta_n(i)))^{1/2}$. 
Note that the variance is taken over the random initializations, for a same realization of the data $(\bY,\bX)$, and hence captures smoothness 
(or roughness) of the empirical risk landscape.
The standard deviation appears to converge exponentially fast to $0$,
confirming that indeed  proximal gradient is converging to the unique local minimizer, as anticipated by 
Theorem \ref{thm:SparseClass}.

\begin{figure}
\begin{subfigure}{0.45\linewidth}
\centering
\includegraphics[width=0.95\linewidth]{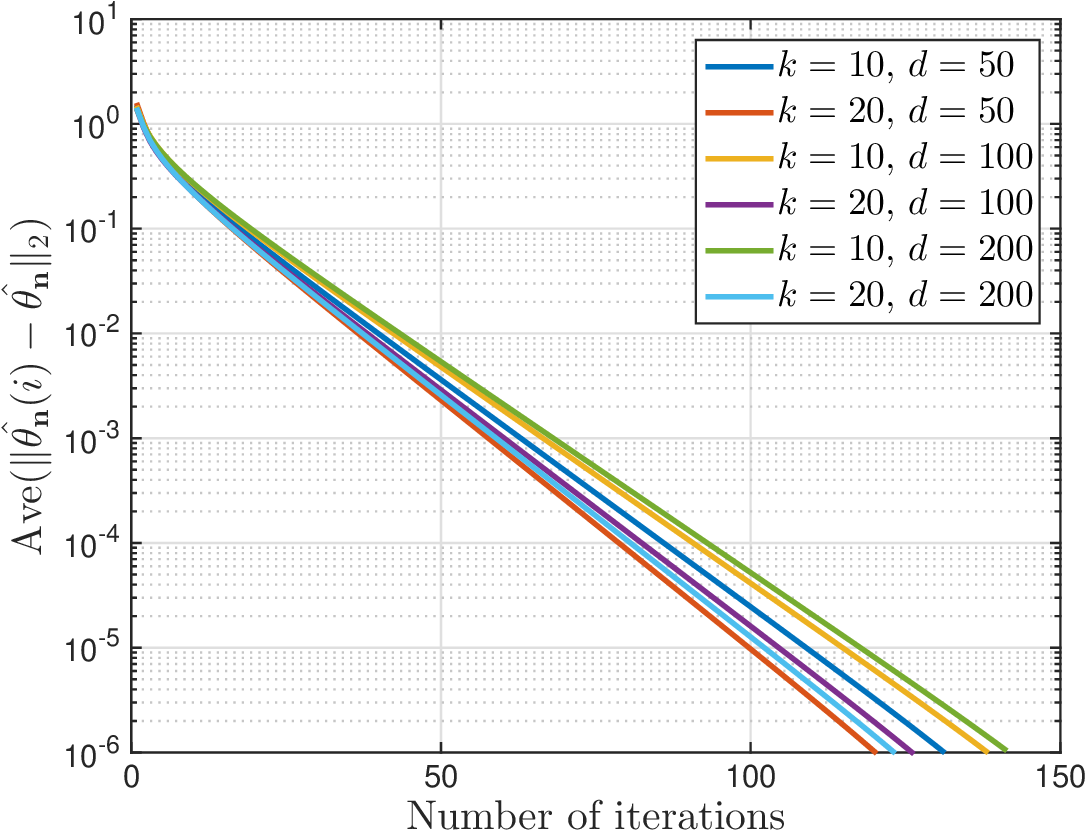}
\caption{}\label{fig:conv_bcl1}
\end{subfigure}\hspace{0.1cm}
\begin{subfigure}{0.45\linewidth}
\centering
\includegraphics[width=0.95\linewidth]{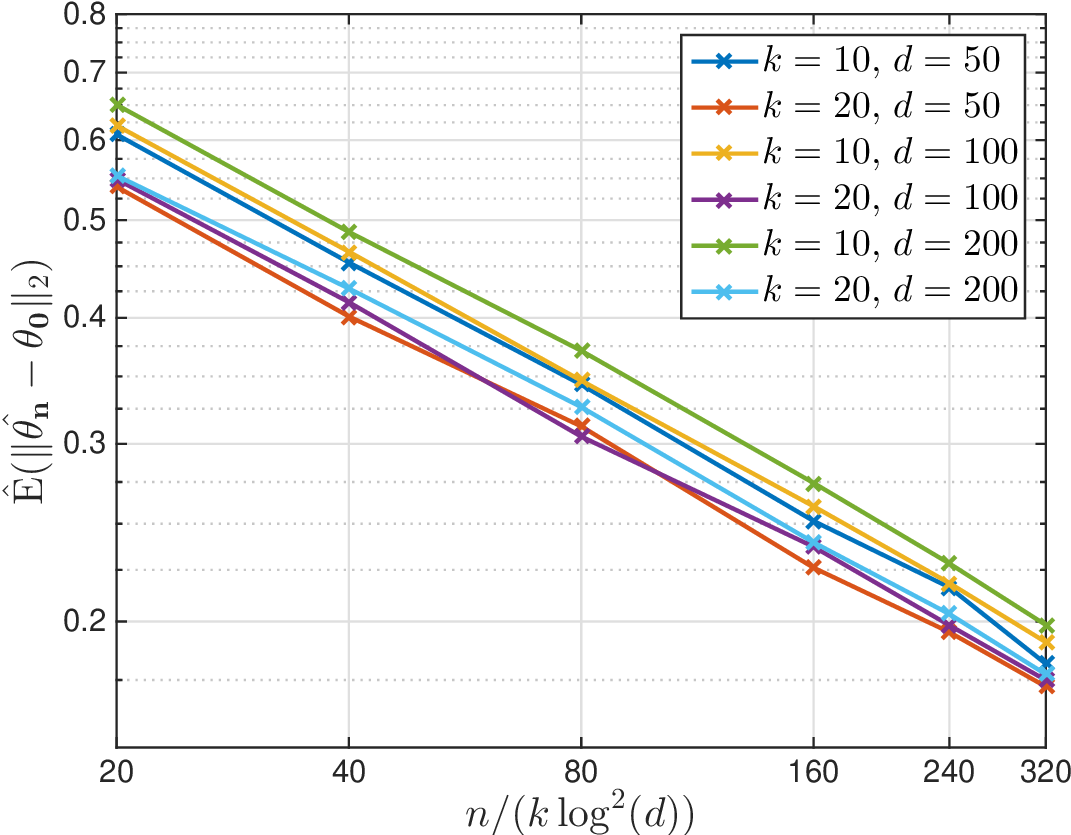}
\caption{}\label{fig:error_bcl1}
\end{subfigure}
\caption{Binary linear classification, very high-dimensional regime: $(a)$ The convergence of proximal gradient descent.  Here
$\Vert \btheta_0 \Vert_2 = 1$, and $n/(k\log^2(d))= 20$, and $\lambda_n = 1/100 \cdot \sqrt{\log^2(d)/n}$
. $(b)$ Convergence of the statistical error.  }
\end{figure}
In Figure \ref{fig:conv_bcl1}, we plot the expected distance from the global minimizer $\hbtheta_n$ for each iterates.
Proximal gradient appears to converge exponentially fast 
 for $n\gg k\log^2(d)$.

\subsection{Robust linear regression }

In  Figures \ref{fig:var_rr}, \ref{fig:var_outlier_rr}, \ref{fig:error_cmp_rr} we present simulations for robust regression.
We generated random covariates $\bX_i\sim \normal(\bzero,\id_{d\times d})$ and responses $Y_i = \<\btheta_0,\bX_i\>+\eps_i$,
where $\|\btheta_0\|_2 = 1$. Again, we used projected gradient descent to solve the optimization problem (\ref{eq:RobustDefinition})
with $\rad=10$. 
For the loss function we used Tukey's  loss (\ref{eq:Tukey}) with $t_0=4.685$.

\begin{figure}
\centering
\includegraphics[width=0.6\linewidth]{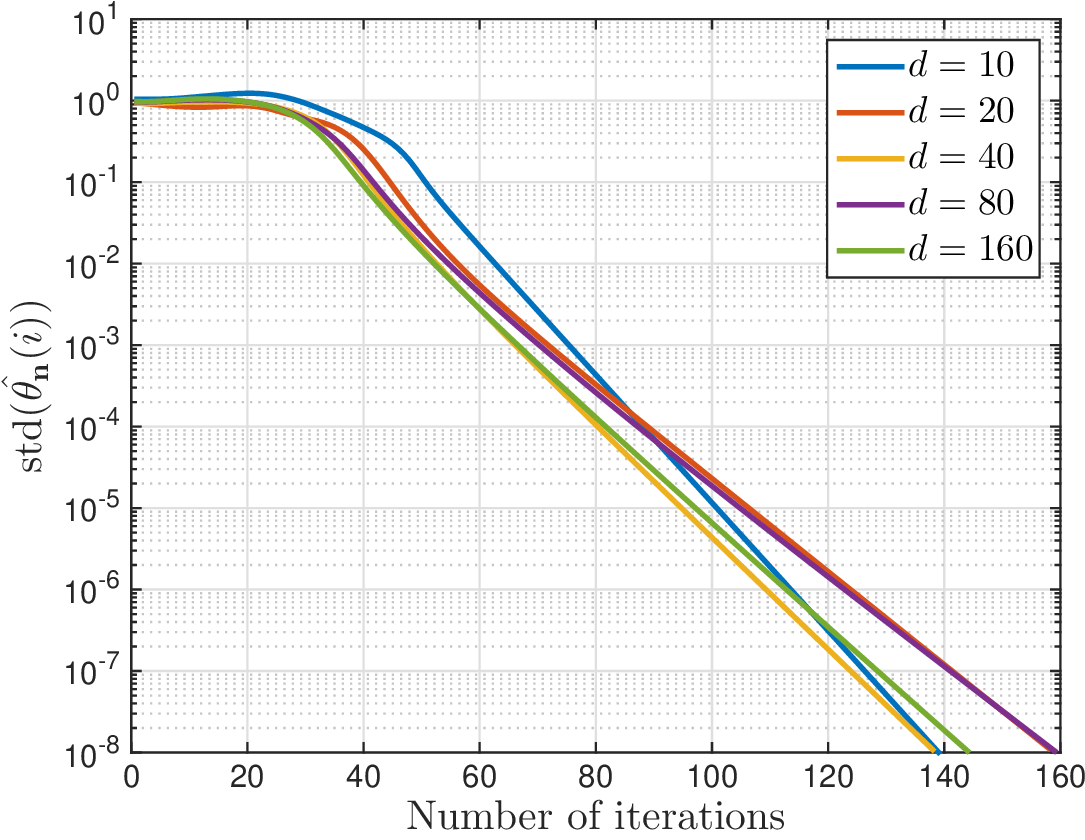}
\caption{Robust regression. The standard deviation of each iteration point with respect to random initialization. }\label{fig:var_rr}
\end{figure}
In Fig. \ref{fig:var_rr}, we plot the standard deviation of the iterates ${\rm std} (\hbtheta_n(i))=\Trace(\hVar(\hbtheta_n(i)))^{1/2}$ over random initializations
$\btheta_s\sim \normal(\bzero,25\, \id_{d\times d}/d)$. In this case $\eps_i \sim \normal(0,1)$. Again, this standard deviation converges exponentially fast to $0$ 
supporting the claim that proximal gradient descent converges to a unique global minimum irrespective of the initialization. 

\begin{figure}
\begin{subfigure}{0.45\linewidth}
\centering
\includegraphics[width=0.95\linewidth]{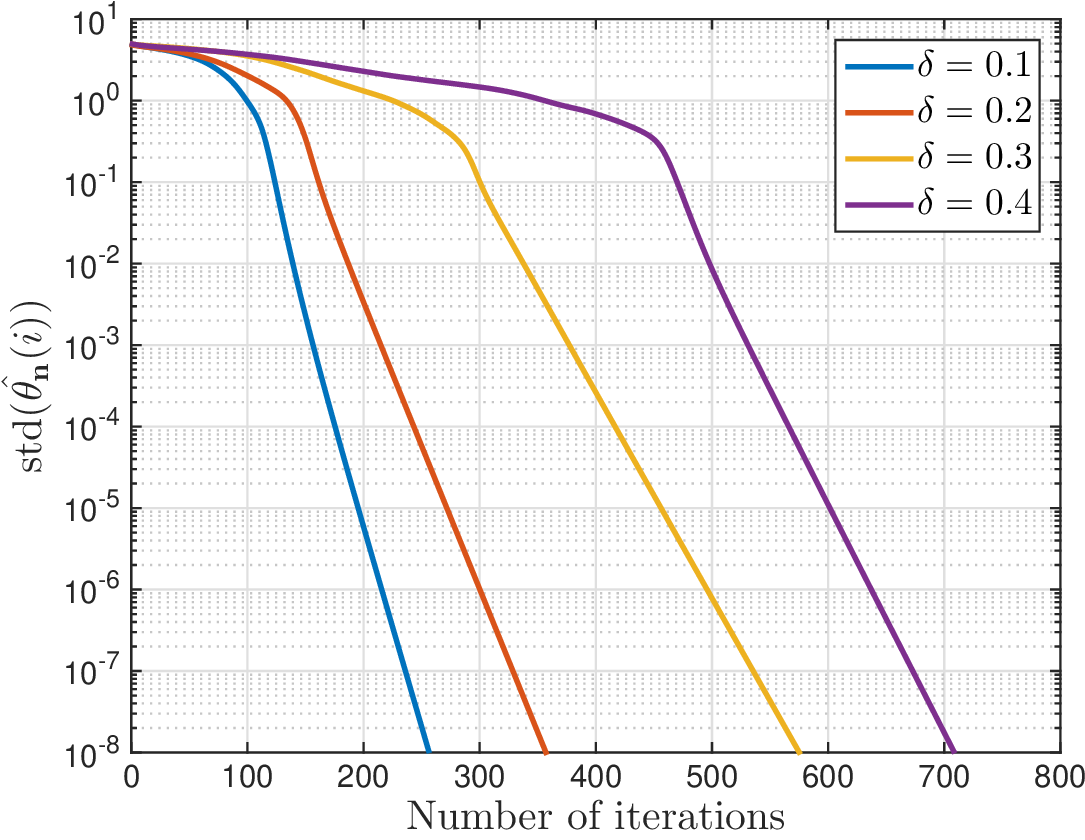}
\caption{}\label{fig:var_outlier_rr}
\end{subfigure}\hspace{0.1cm}
\begin{subfigure}{0.45\linewidth}
\centering
\includegraphics[width=0.95\linewidth]{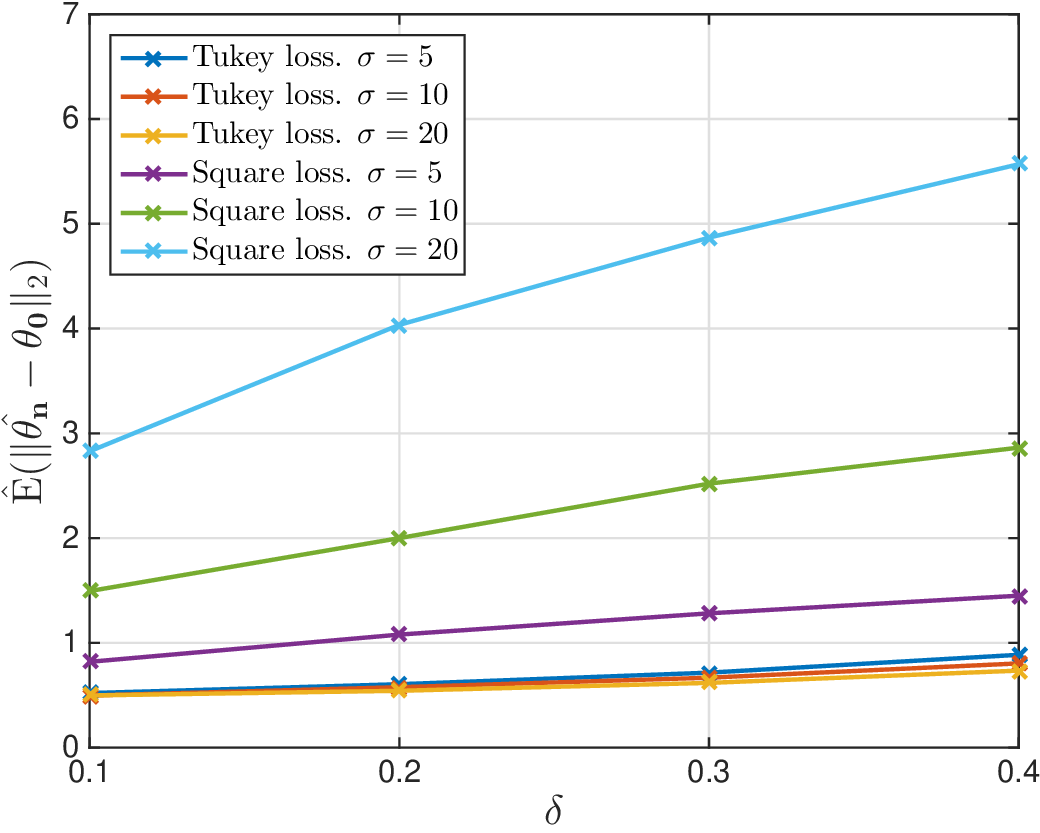}
\caption{}\label{fig:error_cmp_rr}
\end{subfigure}
\caption{Robust regression: $(a)$ The standard deviation of each iteration point with respect to random initialization, for different proportion of contamination.
$(b)$ The robustness of the global minimum between linear regression and Tukey regression.  }
\end{figure}

In Figures \ref{fig:var_outlier_rr}, \ref{fig:error_cmp_rr} we study the a contaminated  model for the noise, namely 
$\eps_i\sim (1-\delta)\normal(0,1) +\delta \normal(0,\sigma^2)$.
In Figure \ref{fig:var_outlier_rr} we plot the standard deviation of the estimates obtained with random initializations 
$\btheta_s\sim \normal(\bzero,25\,\id_{d\times d}/d)$,
for $n=480$, $d=80$. Convergence rate remains exponential even for large contamination fraction. In Figure \ref{fig:error_cmp_rr} we investigated
the dependence of the estimation error on the contamination fraction, and the scale of outliers. Tukey's regression is fairly insensitive
to outliers, while the least squares regression deteriorates as expected.

\subsection{Gaussian mixture model }
\label{sec:MixtureSimulations}

\begin{figure}
\begin{subfigure}{0.45\linewidth}
\centering
\includegraphics[width=0.95\linewidth]{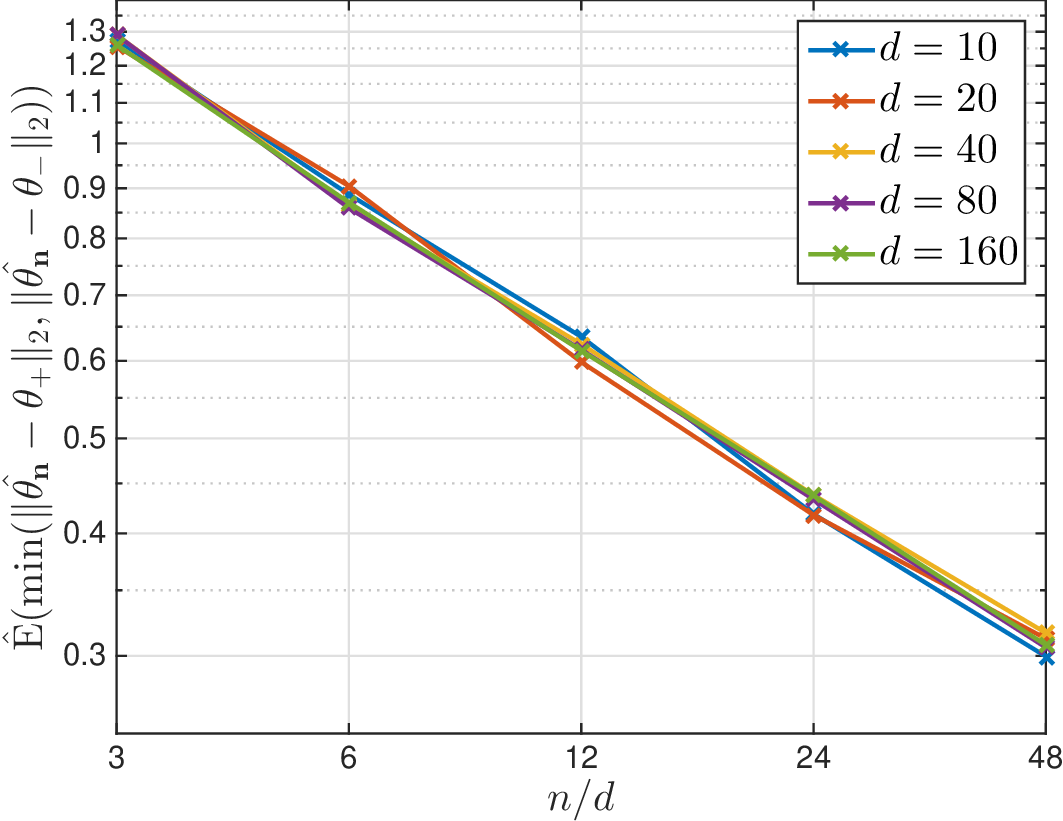}
\caption{}\label{fig:error_gmm}
\end{subfigure}\hspace{0.1cm}
\begin{subfigure}{0.45\linewidth}
\centering
\includegraphics[width=0.95\linewidth]{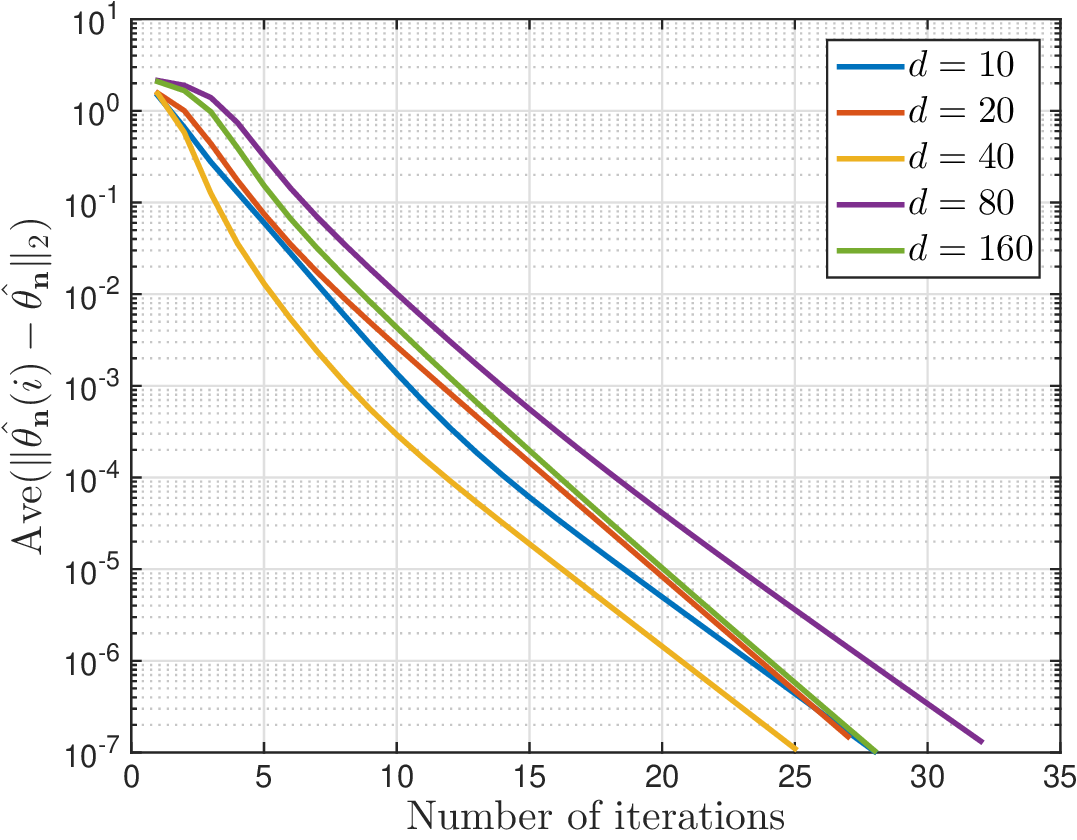}
\caption{}\label{fig:conv_gmm}
\end{subfigure}
\caption{Gaussian mixture model: $(a)$ The convergence of statistical error. Here we use 
$\Vert \btheta_{0,1} - \btheta_{0,2} \Vert_2 = 3$, and $n/d = 6$; $(b)$ The convergence of the gradient descent algorithm. }
\end{figure}

In Figures \ref{fig:error_gmm} and \ref{fig:conv_gmm} we consider the Gaussian mixture model of Section \ref{sec:GaussianMixture}.
We use an equal mixture proportion with $\Vert \btheta_{0,1} - \btheta_{0,2} \Vert_2 = 3$ and compute the maximum likelihood
estimator  (\ref{eqn:erm-mixture}). Instead of using trust region method as suggested in Theorem \ref{thm:Mixture}, we used gradient descent here. We observed that there are only two local minimizers. In Figure  \ref{fig:error_gmm} we plot the convergence of the statistical error, and in Figure \ref{fig:conv_gmm} the convergence of the gradient descent algorithm to one of the only two local minimizers. These results are consistent with Theorem \ref{thm:Mixture}.

\subsection{Australian credit data}
\label{sec:AustralianData}

\begin{figure}
\centering
\includegraphics[width=0.5\linewidth]{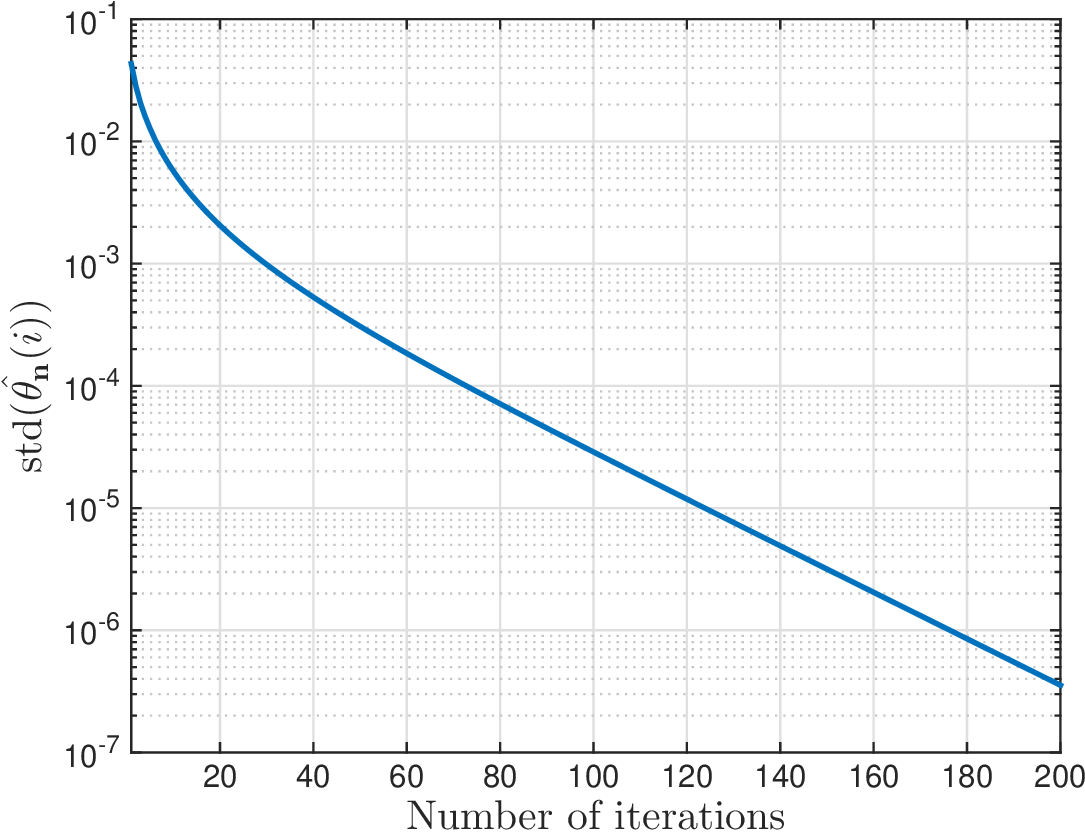}
\caption{Australian credit data: the standard deviation of each iteration point with respect to random initialization. }\label{fig:australian}
\end{figure}

In Figure \ref{fig:australian} we consider the Australian credit dataset from Statlog \cite{Lichman:2013}. The data set contains $n = 690$ entries. Each entries has a binary label,
with $307$ entries labeled $1$ and $383$ labeled $0$.  Each entry also comprises $d = 14$ attributes including both categorical and continuous variables. These variables are normalized with zero mean and unit standard deviation. 

We fit a model of the form $\P(Y_i = 1 \vert \bX_i = \bx) = \sigma(\langle \btheta_0, \bx \rangle)$ with $\sigma(u) = \sigma_L(u)$ the logistic function, by using the non-convex approach (\ref{eq:nonconvex1}) and gradient descent. We also used logistic regression for comparison. Let us emphasize here that our focus here is not on the accuracy of the predictive model, but rather on showing that the non-convex approach is a viable alternative to the standard logistic regression. In particular, the M-estimator appears to be efficiently computable. 

In Figure \ref{fig:australian}, we plot the standard deviation of the estimate $\hat \btheta_n(i)$, over random initializations $\btheta_s \sim \normal(\bzero, \id_{d\times d}/d)$. As for the simulations in the previous section, the standard deviation decreases exponentially fast, suggesting that indeed the optimization problem has a unique local minimum.

\subsection{Colon cancer data}
\label{sec:ColonCancerData}

\begin{figure}
\begin{subfigure}{0.45\linewidth}
\centering
\includegraphics[width=0.95\linewidth]{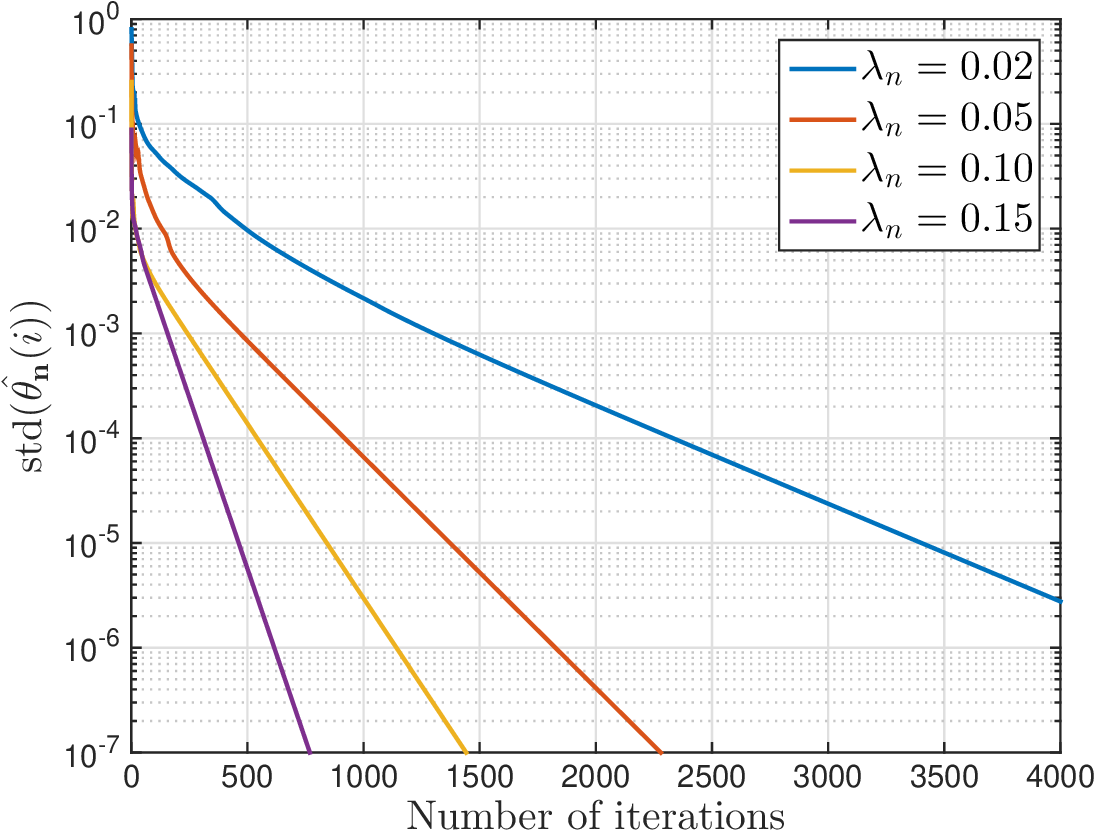}
\caption{}\label{fig:var_coloncancer}
\end{subfigure}\hspace{0.1cm}
\begin{subfigure}{0.45\linewidth}
\centering
\includegraphics[width=0.95\linewidth]{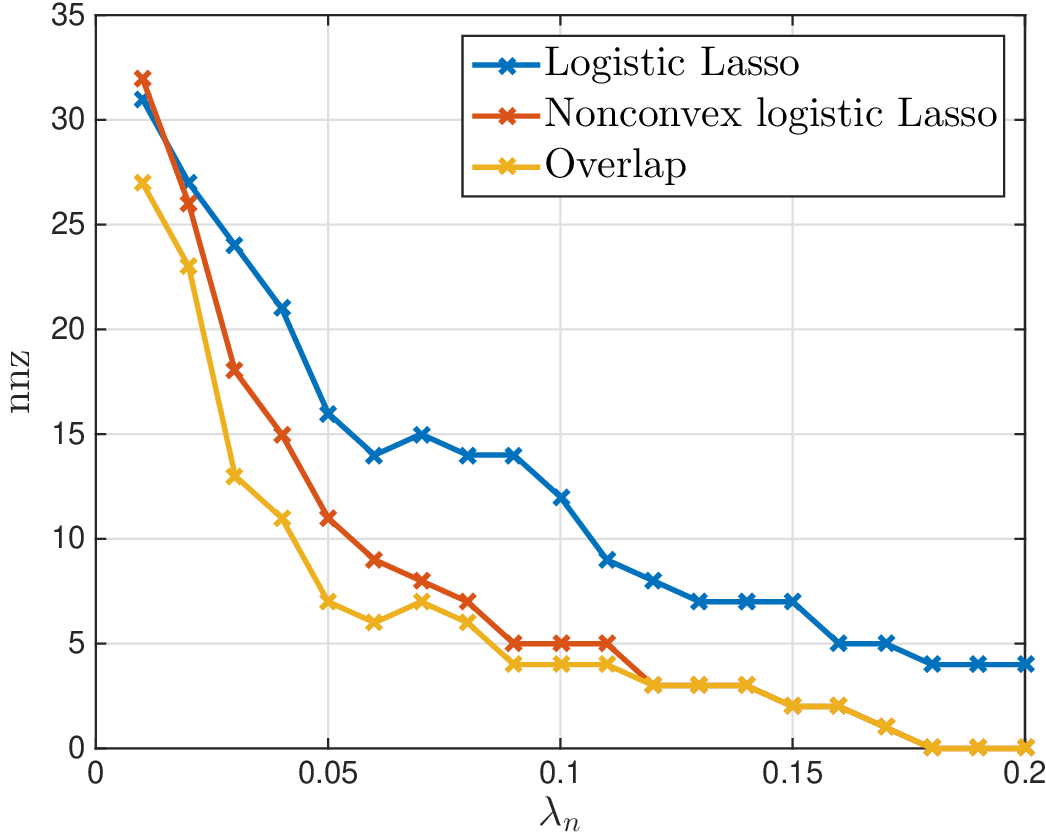}
\caption{}\label{fig:cmp_coloncancer}
\end{subfigure}
\caption{Colon cancer data: $(a)$ The standard deviation of each iteration point with respect to random initialization, for different regularization parameter. $(b)$ Number of non-zero elements of logistic Lasso and non-convex logistic Lasso.  }
\end{figure}

In Figures \ref{fig:var_coloncancer}, \ref{fig:cmp_coloncancer} we consider a gene-expression dataset from  \cite{alon1999broad}. 
The data set contains expression levels of  of $2,000$ genes in $22$
normal and $40$ tumor colon tissues, hence $n=62$ data points. 
Expression levels are normalized as in \cite{alon1999broad}  to have zero mean and unit standard 
deviation. We use the expression levels to form  feature vectors $\bx_i\in\reals^{d}$,  $d = 2000$ and encode the type of tissue 
using a binary label $y_i=1$ (tumor) or $y_i=0$ (no tissue).

We fit a model of the form $\prob(Y_i=1|\bX_i=\bx) = \sigma(\<\btheta_0,\bx\>)$ with $\sigma(u) = \sigma_L(u)$ the logistic function,
by using the non-convex approach (\ref{eq:SparseClass_1}) and proximal gradient. We also used $\ell_1$-regularized logistic regression, for comparison.
Let us emphasize here that our focus here is not on the accuracy of the predictive model, but rather on showing that the non-convex approach 
is a viable alternative to the more standard regularized logistic regression.

In Figure \ref{fig:var_coloncancer}, we plot the standard deviation of the estimate $\hbtheta_n(i)$, over
random initializations $\btheta_s\sim\normal(\bzero,\id_{d\times d}/d)$. As for the simulations in the previous section, the
standard deviation decreases exponentially fast, suggesting that indeed the optimization problem has a unique local minimum.
In Figure \ref{fig:cmp_coloncancer} we compare the model selected by the non-convex approach (\ref{eq:SparseClass_1}) to the one
from $\ell_1$-regularized logistic regression, and also plot the number of overlaps of their selected variables. Note that most of the covariates selected by the non-convex regression method also 
appear in logistic regression. This suggests that the model produced by the non-convex approach is comparable to that
produced by $\ell_1$-regularized logistic regression.

\section*{Acknowledgments}

A.M. was partially supported by the NSF grant CCF-1319979. S.M. was supported by Office of Technology Licensing Stanford Graduate Fellowship. 

\appendix

\section{Some useful tools}

In this section we collect some well-known definitions and tools from high-dimensional probability, for the reader's convenience.

\subsection{Properties of sub-Gaussian and sub-exponential random variables}

Let us first recall the definition of (not necessarily mean zero) sub-Gaussian and sub-exponential random variables in $\R^d$:
\begin{definition}\label{def:subgaussian}
A random variable $\bX\in\R^d$ is $\tau^2$-sub-Gaussian if for any $\blambda\in\R^d$,
\begin{equation}
\E[e^{\<\blambda,\bX - \E[\bX]\>}] \le e^{\frac{\|\blambda\|_2^2\tau^2}{2}}.
\end{equation}
\end{definition}
\begin{definition}\label{def:subexponential}
A random variable $\bX\in\R^d$ is $K$-sub-exponential if for any $\blambda\in\R^d$, $\|\<\blambda,\bX-\E[\bX]\>\|_{\psi_1}\le K\|\blambda\|_2$, where $\|\cdot\|_{\psi_1}$ is the Orlicz $\psi_1$-norm:
\begin{equation}
\|X\|_{\psi_1} \defeq \sup_{k\ge 1} \frac{1}{k} \E[|X-\E(X)|^k]^{1/k}.
\end{equation}
\end{definition}
Note: we can also define sub-Gaussian random variables via the Orlicz $\psi_2$-norm. We choose to follow the more classic definition there so as to make sub-Gaussian concentration inequalities clearer. 

Proofs of our main theorems rely on some properties about sub-Gaussian and sub-exponential random variables that are well known in the literature, for example \cite{boucheron2013concentration,vershynin2010introduction}. We summarize them here for reference. 

\begin{lemma}\label{lem:subgaussian}
Assume $\bX\in\R^d$ has mean zero and is $\tau^2$-sub-Gaussian, then
\begin{enumerate}[label=$(\alph*)$]
\item There exists numerical constants $C_{2k}\in(0,\infty)$ for all integers $k\ge 1$ such that
\begin{equation}
\E[|\langle \bu,\bX\rangle|^{2k}] \le C_{2k}\|\bu\|_2^{2k}\tau^{2k}
\end{equation}
for all $\bu\in\R^d$. In particular, $C_2=1$, and we can take $C_{2k}=2^{k+1}k!$.
\item Higher moments of $\|\bX\|_2$ are controlled, that is, for all integers $k\ge 1$,
\begin{equation}
\E[\|\bX\|_2^{2k}] \le C_{2k} d^k\tau^{2k},
\end{equation}
where $C_{2k}$ is the same as in $(a)$.
\item $\|\bX\|_2^2$ is $4\tau^2$-sub-exponential. In particular
\begin{equation}
\E[e^{\frac{\|\bX\|_2^2}{4\tau^2}}] \le 2^{\frac{d}{2}} < e^{\frac{d}{2}}.
\end{equation}
\item If $X\in\R$ is zero-mean and $\tau^2$-sub-Gaussian and $\alpha$ is a random variable (that can depend on $X$) with $|\alpha|\le 1$, then there exists some absolute constant $\bddsubg\le 64$ such that $\alpha X$ is $\bddsubg\tau^2$-sub-Gaussian.
\item If $X\in\R$ is zero-mean $K$-sub-exponential and $\beta$ is a random variable (that can depend on $X$) with $|\beta|\le 1$, then there exists some absolute constant $\bddsube\le 2$ such that $\beta X$ is $\bddsube K$-sub-exponential.
\item
\begin{equation}
	\E[\|\bX\|_\infty] \le \sqrt{2\tau^2\log(2d)}.
\end{equation}
\end{enumerate}
\end{lemma}
\begin{proof}
\begin{enumerate}[label=$(\alph*)$]
\item This is known in the literature, for example Theorem 2.1 in \cite{boucheron2013concentration}.

\item This is a direct consequence of part $(a)$. From the generalized mean inequality, we have
\begin{equation}
\E[\|\bX\|_2^{2k}] = \E\Big[ (\sum_{j=1}^{d}X_j^2)^{k} \Big] \le d^{k-1} \E\Big[ \sum_{j=1}^{d}X_j^{2k} \Big].
\end{equation}
Applying part $(a)$ with the standard basis $\be_j\in\R^d$, we get $\E[X_j^{2k}]\le C_{2k}\tau^{2k}$. Summing over $j$ gives $\E[\|\bX\|_2^{2k}]\le d^{k}C_{2k}\tau^{2k}$.

\item We can assume $\tau^2=1$ by scale invariance. Let $\bW\sim\normal(0,\id_{d\times d})$ be independent of $\bX$. It is known that $\E[e^{\lambda Z^2}]=\frac{1}{\sqrt{1-2\lambda}}$ for $Z\sim\normal(0,1)$ and $\lambda<\frac{1}{2}$, so we have
\begin{equation}
\E[e^{\lambda\|\bW\|_2^2}] = \prod_{i=1}^{d} \E[e^{\lambda W_i^2}] = \frac{1}{(1-2\lambda)^{d/2}}
\end{equation}
for any $\lambda<\frac{1}{2}$.

Now, for $\lambda>0$, we evaluate the quantity $\E[e^{\sqrt{2\lambda}\<\bX,\bW\>}]$ in two ways. We have
\begin{equation}\label{eqn:x}
\E[e^{\sqrt{2\lambda}\<\bX,\bW\>}] = \E\Big[ \E[e^{\sqrt{2\lambda}\<\bX,\bW\>}|\bX] \Big] = \E\Big[ e^{\frac{\|\sqrt{2\lambda}\bX\|_2^2}{2}} \Big] = \E[e^{\lambda\|\bX\|_2^2}].
\end{equation}
On the other hand, we have
\begin{equation}\label{eqn:w}
\E[e^{\sqrt{2\lambda}\<\bX,\bW\>}] = \E\Big[ \E[e^{\<\sqrt{2\lambda}\bW, \bX\>}|\bW] \Big] \le \E[e^{\frac{\|\sqrt{2\lambda}\bW\|_2^2}{2}}] = \E[e^{\lambda\|\bW\|_2^2}] = \frac{1}{(1-2\lambda)^{d/2}},
\end{equation}
the last equality holding for $\lambda<\frac{1}{2}$. Combining (\ref{eqn:x}) and (\ref{eqn:w}) and taking $\lambda=\frac{1}{4}$, we get
\begin{equation}
\E[e^{\frac{\|\bX\|_2^2}{4}}] \le 2^{d/2} < e^{d/2}.
\end{equation}

\item  By Theorem 2.1 in \cite{boucheron2013concentration}, we have $\E[|X|^{2k}]\le k!(4\tau^2)^{k}$. Consequently, $\E[|\alpha X|^{2k}]\le k!(4\tau^2)^k$, since $|\alpha|\le 1$. Now we introduce $(\alpha',X')$ that is an independent copy of $(\alpha,X)$, then
\begin{equation}
\E[|\alpha X - \alpha'X'|^{2k}] \le 2^{2k-1}\E[|\alpha X|^{2k} + |\alpha'X'|^{2k}] = 2^{2k}\E[|\alpha X|^{2k}] \le k!(16\tau^2)^{k}.
\end{equation}
We apply the converse statement in \cite[Theorem 2.1]{boucheron2013concentration}, to conclude that $\alpha X- \alpha'X'$ is $64\tau^2$-sub-Gaussian. Finally, as $e^{\lambda (a-t)}$ is convex in $t$, we have
\begin{equation}
\E[e^{\lambda(\alpha X - \E[\alpha X])}] = \E[e^{\lambda(\alpha X - \E[\alpha'X'])}] \le \E[e^{\lambda(\alpha X - \alpha'X')}],
\end{equation}
and thus $\alpha X$ is also $64\tau^2$-sub-Gaussian.

\item By Remark 5.18 in \cite{vershynin2010introduction}, we have
\begin{equation}
\|\beta X - \E[\beta X]\|_{\psi_1} \le 2\|\beta X\|_{\psi_1} \le 2\|X\|_{\psi_1}\, .
\end{equation}

\item It suffices to work with $\tau=1$. For any $\lambda>0$, we have
\begin{equation*}
	\E[e^{\lambda\|\bX\|_\infty}] = \E[e^{\lambda\max_{j\in[d]}\{X_j,-X_j\} }] \le \sum_{j=1}^{d} \E[e^{\lambda X_j}] + \sum_{j=1}^{d} \E[e^{-\lambda X_j}] \le 2de^{\frac{\lambda^2}{2}}.
\end{equation*}
Applying Jensen's inequality on the function $t\mapsto e^{\lambda t}$, we get
\begin{equation*}
	\E[\|X\|_\infty] \le \frac{1}{\lambda} \log\E[e^{\lambda \|\bX\|_\infty}] \le \frac{1}{\lambda} \Big( \frac{\lambda^2}{2} + \log(2d) \Big).
\end{equation*}
Optimizing the RHS gives $\lambda^*=\sqrt{2\log(2d)}$ and an upper bound $\sqrt{2\log(2d)}$.
\end{enumerate}
\end{proof}

\begin{lemma}\label{lem:sumofsub}
There exists a universal constant $C_s$, such that the sum of two dependent sub-Gaussian random variables with parameters $\tau_1^2$ and $\tau_2^2$ is $C_s (\tau_1^2 + \tau_2^2)$-sub-Gaussian, and the sum of two dependent sub-exponential random variables with parameters $\tau_1^2$ and $\tau_2^2$ is $C_s (\tau_1^2 + \tau_2^2)$-sub-exponential. 
\end{lemma}

\begin{proof}
This lemma follows directly from the equivalent form of definition of sub-Gaussian and sub-exponential random variables using Orlicz norms. 
\end{proof}

\begin{theorem}[Bernstein inequality for subexponential random variables]\label{thm:Bernstein}
Let $X_1,\dots X_n$ be independent sub-exponential random variables with $\|X_i\|_{\psi_1}\le b$, and define
$S_n\equiv \sum_{i=1}^n\big(X_i-\E X_i\big)$. Then there exists a universal constant $c$ such that, for all $t>0$,
\begin{align}
\prob\big(S_n\ge t\big)\le \exp\Big\{-c \min \Big(\frac{t^2}{n b^2}, \frac{t}{b}\Big)\Big\}\, .
\end{align}
\end{theorem}

\subsection{Bounding norms via $\eps$-covers}

In this section we state two simple technical lemmas that are useful for  our proofs. 
Their proofs of these lemmas can be found in \cite{vershynin2010introduction}.
\begin{lemma}\label{lem:2norm}
Let $a\in\R^d$ and $V_\eps=\{\bv_1,\dots,\bv_N\}$ be an $\eps$-cover of $\Ball^d(\bzero,1)$, then
\begin{equation}
\|\ba\|_2 \le \frac{1}{1-\eps} \sup_{\bv\in V_\eps} \langle \bv,\ba\rangle.
\end{equation}
\end{lemma}

\begin{lemma}\label{lem:opnorm}
Let $\bM\in\R^{d\times d}$ be a symmetric $d\times d$ matrix and $V_\eps=\{\bv_1,\dots,\bv_N\}$ be an $\eps$-cover of $\Ball^d(\bzero,1)$, then
\begin{equation}
\|\bM\|_\op \le \frac{1}{1-2\eps} \sup_{\bv\in V_\eps} |\langle \bv,\bM\bv\rangle|.
\end{equation}
\end{lemma}

\section{Proof of Theorem \ref{thm:uniformconvergence1}: High-dimensional regime}
\subsection{Proof of Theorem \ref{thm:uniformconvergence1}.$(a)$: Uniform convergence of gradient}

\noindent{\bf Step 1. } \emph{Decompose the bad events using $\eps$-nets.} 

Let $N_\eps$ be the $\eps$-covering number of the ball $\Ball^p(\bzero, \rad)$. It is known that $\log N_\eps \leq p \log (3 \rad/\eps)$ \cite{vershynin2010introduction}. Let $\Theta_\eps = \{ \btheta_1, \ldots, \btheta_N\}$ be a corresponding $\eps$-cover with $N = N_\eps$ elements. For any $\btheta \in \Ball^p(\bzero, \rad)$, let $j(\btheta) = \argmin_{j\in [N]} \Vert \btheta - \btheta_j \Vert_2$. Then $\Vert \btheta - \btheta_{j(\btheta)}\Vert_2 \leq \eps$ for all $\btheta \in \Ball^p(\bzero, \rad)$. 

Observe that $\nabla \what R_n(\btheta) = \frac{1}{n} \sum_{i=1}^n \nabla \ell(\btheta; \bZ_i)$, and $\nabla R(\btheta) = \E[\nabla \ell(\btheta; \bZ)]$. Thus, for any $\btheta \in \Ball^p(\bzero, \rad)$, we have
\begin{equation}
\begin{aligned}
\Big\Vert \nabla \what R_n(\btheta) - \nabla R(\btheta) \Big\Vert_2 \leq &\Big\Vert \frac{1}{n} \sum_{i=1}^n \Big[ \nabla \ell(\btheta; \bZ_i) - \nabla \ell(\btheta_{j(\btheta)}; \bZ_i) \Big] \Big\Vert_2 \\
&+ \Big\Vert \frac{1}{n} \sum_{i=1}^n \nabla \ell(\btheta_{j(\btheta)}; \bZ_i) - \E[\nabla \ell(\btheta_{j(\btheta)}; \bZ)] \Big\Vert_2 \\
&+\Big\Vert \E[\nabla \ell(\btheta_{j(\btheta)}; \bZ)] - \E[\nabla \ell (\btheta ; \bZ)]\Big\Vert_2.
\end{aligned}
\end{equation}

Hence, we have
\[
\P\left( \sup_{\btheta \in \Ball^p (\bzero, \rad)} \left \Vert \nabla \what R_n(\btheta) - \nabla R(\btheta) \right \Vert_2 \geq t \right) \leq \P(A_t) + \P(B_t) + \P(C_t),
\]
where the events $A_t$, $B_t$, and $C_t$ are defined as
\[
\begin{aligned}
A_t =& \left\{ \sup_{\btheta \in \Ball^p(\bzero, \rad)} \left \Vert \frac{1}{n}\sum_{i=1}^n  \Big[ \nabla \ell(\btheta;\bZ_i) - \nabla \ell(\btheta_{j(\btheta)};\bZ_i) \Big]\right\Vert_2 \geq \frac{t}{3} \right\},\\
B_t =& \left\{ \sup_{j \in [N]} \left\Vert \frac{1}{n} \sum_{i=1}^n \nabla \ell(\btheta_j; \bZ_i) - \E [\nabla \ell(\btheta_j ; \bZ)] \right\Vert_2 \geq \frac{t}{3} \right\},\\
C_t =& \left\{ \sup_{\btheta \in \Ball^p(\bzero, \rad)} \Big \Vert \E[\nabla \ell(\btheta_{j(\btheta)}; \bZ)] - \E[\nabla \ell(\btheta; \bZ)] \Big\Vert_2 \geq \frac{t}{3}\right\}.\\
\end{aligned}
\]

\noindent{\bf Step 2. }\emph{Upper bound $\P(B_t)$.}

Let $V_{1/2}$ be a $(1/2)$-cover of $\Ball^p(\bzero, 1)$ with $\log \vert V_{1/2} \vert \leq p \log 6$. From Lemma \ref{lem:2norm}  we know that 
\[
\left\Vert \frac{1}{n} \sum_{i=1}^n \nabla \ell(\btheta_j; \bZ_i ) - \E[\nabla \ell(\btheta_j; \bZ)] \right \Vert_2 \leq 2 \sup_{\bv \in V_{1/2}} \Big \langle \bv, \frac{1}{n} \sum_{i=1}^n \nabla \ell(\btheta_j; \bZ_i) - \E[\nabla \ell(\btheta_j; \bZ)] \Big \rangle.
\]

Taking union bounds over $\Theta_\eps$ and $V_{1/2}$ yields
\[
\begin{aligned}
\P(B_t) \leq & \P\left(\sup_{j \in [N], \bv \in V_{1/2}} \left\{ \frac{1}{n} \sum_{i=1}^n \Big\langle \nabla \ell(\btheta_j; \bZ_i) - \E[ \nabla \ell(\btheta_j ;\bZ)], \bv \Big \rangle \right\} \geq \frac{t}{6} \right) \\
\leq& e^{p \log\frac{3 \rad}{ \eps} + p \log 6} \sup_{j \in [N] , \bv \in V_{1/2}} \P\left(  \frac{1}{n} \sum_{i=1}^n \Big \langle \nabla \ell(\btheta_j; \bZ_i) - \E[ \nabla \ell(\btheta_j ;\bZ)], \bv \Big \rangle \geq \frac{t}{6} \right).
\end{aligned}
\]

Fixing any $j$ and $\bv$, according to Assumption \ref{ass:GradientSub}, we have $\langle\nabla \ell(\btheta_j; \bZ_i) - \E [\nabla \ell(\btheta_j; \bZ)], \bv \rangle$ is $\tau^2$-sub-Gaussian. Hence $\frac{1}{n} \sum_{i=1}^n \Big \langle \nabla \ell(\btheta_j; \bZ_i) - \E[ \nabla \ell(\btheta_j ;\bZ)], \bv \Big \rangle$ is $\tau^2/n$-sub-Gaussian random variable. This gives
\[
\P\left(  \frac{1}{n} \sum_{i=1}^n \Big \langle \nabla \ell(\btheta_j; \bZ_i) - \E[ \nabla \ell(\btheta_j ;\bZ)], \bv \Big \rangle \geq \frac{t}{6} \right) \leq e^{- \frac{nt^2}{ 144 \tau^2}}.
\]

As a result, 
\[
\P(B_t) \leq \exp\left( - \frac{nt^2}{144\tau^2} + p \log \frac{18\rad}{\eps} \right).
\]

Thus, 
\[
t >\sqrt \frac{144 \tau^2(p\log \frac{18\rad}{\eps} + \log \frac{2}{ \delta})}{n}
\]
ensures that $\P(B_t) \leq \delta/2$. 

\noindent{\bf Step 3. }\emph{Upper bound $\P(A_t)$ and $\P(C_t)$.}

Let us look at the deterministic event $C_t$ first. We have
\[
\begin{aligned}
&\sup_{\btheta \in \Ball^p (\bzero, \rad)} \left\Vert \E[\nabla \ell (\btheta; \bZ) - \nabla \ell (\btheta_{j(\btheta)}; \bZ)] \right\Vert_2 \\
\leq& \sup_{\btheta \in \Ball^p (\bzero, \rad)} \frac{\left\Vert \E[\nabla \ell (\btheta; \bZ) - \nabla \ell (\btheta_{j(\btheta)}; \bZ)] \right\Vert_2 }{\Vert \btheta - \btheta_{j(\btheta)} \Vert_2} \cdot \sup_{\btheta \in \Ball^p (\bzero, \rad)}\Vert \btheta - \btheta_{j(\btheta)} \Vert_2\\
\leq &  \E \left[\sup_{\btheta \in \Ball^p(\bzero, \rad)} \Vert \nabla^2 \what R(\btheta) \Vert_{\op} \right]  \cdot \eps\\  \\
\leq & D_* \cdot \eps,
\end{aligned}
\]
where 
\[
\begin{aligned}
D_* =&  \E \left[\sup_{\btheta \in \Ball^p(\bzero, \rad)} \Vert \nabla^2 \what R(\btheta) \Vert_{\op}\right] \\
\leq& \E \left[\sup_{\btheta \in \Ball^p(\bzero, \rad)} \Vert \nabla^2 \what R(\btheta) - \nabla^2 R(\btheta_*) \Vert_{\op}\right ] + \Vert \nabla^2 R(\btheta_*) \Vert_\op \\ 
\leq& \E \left[\sup_{\btheta \in \Ball^p(\bzero, \rad)} \left\Vert \frac{1}{n}\sum_{i=1}^n \nabla^2 \ell(\btheta; \bZ_i) - \E[\nabla^2 \ell(\btheta_*;\bZ)] \right\Vert_{\op}\right] + \HUB \\ 
\leq& \E \left[\sup_{\btheta \in \Ball^p(\bzero, \rad)} \left\Vert  \nabla^2 \ell(\btheta; \bZ) - \E[\nabla^2 \ell(\btheta_*;\bZ)] \right\Vert_{\op}\right] + \HUB \\ 
\leq& \E \left[\sup_{\btheta \in \Ball^p(\bzero, \rad)} \left\Vert  \nabla^2 \ell(\btheta; \bZ) - \nabla^2 \ell(\btheta_*;\bZ) \right\Vert_{\op}\right] + \HUB \\ 
\leq & 2\rad J_* + \HUB. 
\end{aligned}
\] 
In this line of inequality, we used Assumption \ref{ass:BoundedHessian}.

We use Markov inequality to bound the probability of event $A_t$. 
\[
\begin{aligned}
\P\left( A_t \right) = & \P \left( \sup_{\btheta \in \Ball^p(\bzero, \rad)} \left \Vert \frac{1}{n}\sum_{i=1}^n \nabla \ell(\btheta;\bZ_i) - \nabla \ell(\btheta_{j(\btheta)};\bZ_i)\right\Vert_2 \geq \frac{t}{3} \right) \\
\leq & \frac{3}{t}\E \left [ \sup_{\btheta \in \Ball^p (\bzero, \rad)} \left\Vert \frac{1}{n}\sum_{i=1}^n \nabla \ell(\btheta;\bZ_i) - \nabla \ell(\btheta_{j(\btheta)};\bZ_i) \right \Vert_2\right] \\
\leq & \frac{3 \eps}{ t} \E \left[ \sup_{\btheta \in \Ball^p (\bzero, \rad)} \Vert \nabla^2 \what R(\btheta) \Vert_{\op} \right] \\
\leq & \frac{3 \eps D_*}{t}. 
\end{aligned}
\]

Taking $t \geq 6 \eps D_*/\delta$, we have 
\[
\P(A_t) \leq \frac{\delta}{2},
\]
and $C_t$ will never happen.

\noindent{\bf Step 4. }\emph{Conclusion.}

Using the above results, to ensure the probability of the bad event to be less than $\delta$, it is sufficient to take $\eps = \frac{\delta \tau}{6 (\HUB + 2 \rad J_*) \cdot np}$, and 
\begin{align}
t \geq \max\left\{ \frac{\tau}{np}, \sqrt \frac{144 \tau^2 (p \log \frac{108 \rad (\HUB + 2 \rad J_*) np}{\delta \tau} + \log \frac{4}{ \delta})}{n} \right\}. 
\end{align}

According to Assumption \ref{ass:BoundedHessian}, we have $\HUB\leq \tau^2 p^{\Ch}$ and $J_* \leq \tau^3 p^{\Ch}$. Thus, there exists a universal constant $C_0$, and letting $C_1 = C_0 \cdot (\Ch \vee \log(\rad\tau/\delta)\vee 1)$, such that as long as $n \geq C_1 p \log p$,
\begin{align}
\P\left(\sup_{\btheta \in \Ball^p (\bzero, \rad)} \left \Vert \nabla \what R_n(\btheta) - \nabla R(\btheta) \right \Vert_{2} \geq  \tau \sqrt{\frac{C_1 p \log n}{n}} \right)\leq \delta. 
\end{align}

\subsection{Proof of Theorem \ref{thm:uniformconvergence1}.$(b)$: Uniform convergence of Hessian}


\noindent{\bf Step 1. } \emph{Decompose the bad event using $\eps$-nets.}

Let $N_\eps$ be the $\eps$-covering number of the $p$ dimensional Euclidean ball $\Ball^p(\bzero, \rad) = \Ball_2^p (\bzero,\rad)$. It is known that $\log N_\eps \leq p \log (3\rad /\eps)$
\cite{vershynin2010introduction}. Let $\Theta_\eps = \{ \btheta_1, \ldots, \btheta_N\}$ be an $\eps$-cover with $N = N_\eps$ elements. For any $\btheta \in \Ball^p(\bzero, \rad)$, let $j(\btheta) = \argmin_{j\in [N]} \Vert \btheta - \btheta_j \Vert_2$. Then $\Vert \btheta - \btheta_{j(\btheta)}\Vert_2 \leq \eps$ for all $\btheta \in \Ball^p(\bzero, \rad )$. 

Observe that $\nabla^2 \what R_n(\btheta) = \frac{1}{n} \sum_{i=1}^n \nabla^2 \ell(\btheta; \bZ_i)$, and $\nabla^2 R(\btheta) = \E[\nabla^2 \ell(\btheta; \bZ)]$. Thus, for any $\btheta \in \Ball^p(\bzero, \rad)$, we have
\begin{equation}
\begin{aligned}
\Big\Vert \nabla^2 \what R_n(\btheta) - \nabla^2 R(\btheta) \Big\Vert_{\op} \leq & \Big\Vert \frac{1}{n} \sum_{i=1}^n \Big[\nabla^2 \ell(\btheta; \bZ_i) - \nabla^2 \ell(\btheta_{j(\btheta)}; \bZ_i)\Big]\Big\Vert_{\op} \\
&+ \Big\Vert \frac{1}{n} \sum_{i=1}^n \nabla^2 \ell(\btheta_{j(\btheta)}; \bZ_i) - \E[\nabla^2 \ell(\btheta_{j(\btheta)}; \bZ)] \Big\Vert_{\op} \\
&+ \Big\Vert \E[\nabla^2 \ell(\btheta_{j(\btheta)}; \bZ)] - \E[\nabla^2 \ell (\btheta ; \bZ)]\Big\Vert_{\op}. 
\end{aligned}
\end{equation}

Hence, we have
\[
\P\left( \sup_{\btheta \in \Ball^p (\bzero, \rad )} \left \Vert \nabla^2 \what R_n(\btheta) - \nabla^2 R(\btheta) \right \Vert_{\op} \geq t \right) \leq \P(A_t) + \P(B_t) + \P(C_t),
\]
where the events $A_t$, $B_t$, and $C_t$ are defined as
\[
\begin{aligned}
A_t =& \left\{ \sup_{\btheta \in \Ball^p(\bzero, \rad)} \left \Vert \frac{1}{n}\sum_{i=1}^n \Big[ \nabla^2 \ell(\btheta;\bZ_i) - \nabla^2 \ell(\btheta_{j(\btheta)};\bZ_i) \Big] \right\Vert_{\op} \geq \frac{t}{3} \right\},\\
B_t =& \left\{ \sup_{j \in [N]} \left\Vert \frac{1}{n} \sum_{i=1}^n \nabla^2 \ell(\btheta_j; \bZ_i) - \E [\nabla^2 \ell(\btheta_j ; \bZ)] \right\Vert_{\op} \geq \frac{t}{3} \right\},\\
C_t =& \left\{ \sup_{\btheta \in \Ball^p(\bzero, \rad)} \Big \Vert \E[\nabla^2 \ell(\btheta_{j(\btheta)}; \bZ)] - \E[\nabla^2 \ell(\btheta; \bZ)] \Big\Vert_{\op} \geq \frac{t}{3}\right\}.\\
\end{aligned}
\]

\noindent{\bf Step 2. } \emph{Upper bound $\P(B_t)$.}

Let $V_{1/4}$ be a $(1/4)$-cover of $\Ball^p(\bzero, 1)$ with $\log \vert V_{1/4} \vert \leq p \log 12$. From Lemma \ref{lem:opnorm} we know that 
\[
\left\Vert \frac{1}{n} \sum_{i=1}^n \nabla^2 \ell(\btheta_j; \bZ_i ) - \E[\nabla^2 \ell(\btheta_j; \bZ)] \right \Vert_{\op} \leq 2 \sup_{\bv \in V_{1/4}} \left\vert \Big\langle \bv, \Big( \frac{1}{n} \sum_{i=1}^n \nabla^2 \ell(\btheta_j; \bZ_i) - \E[\nabla^2 \ell(\btheta_j; \bZ)]\Big) \bv \Big\rangle \right\vert .
\]

Taking union bound over $\Theta_\eps$ and $V_{1/4}$ yields
\[
\begin{aligned}
\P(B_t) \leq & \P\left(\sup_{j \in [N], \bv \in V_{1/4}} \left\vert \frac{1}{n} \sum_{i=1}^n \Big\langle \bv, \Big( \nabla^2 \ell(\btheta_j; \bZ_i) - \E[ \nabla^2 \ell(\btheta_j ;\bZ)] \Big) \bv \Big \rangle \right\vert \geq \frac{t}{6} \right) \\
\leq& e^{p \log\frac{3 \rad }{ \eps} + p \log 12} \sup_{j \in [N] , \bv \in V_{1/4}} \P\left(  \left\vert \frac{1}{n} \sum_{i=1}^n \Big\langle \bv, \Big( \nabla^2 \ell(\btheta_j; \bZ_i) - \E[ \nabla^2 \ell(\btheta_j ;\bZ)] \Big) \bv \Big \rangle\right\vert  \geq \frac{t}{6} \right).
\end{aligned}
\]

Fixing any $j$ and $\bv$, according to Assumption \ref{ass:Hessianub}, $\Big\langle \bv, \Big( \nabla^2 \ell(\btheta_j; \bZ_i) - \E[ \nabla \ell^2(\btheta_j ;\bZ)] \Big) \bv \Big \rangle$ is $\tau^2$-sub-exponential. Hence by Bernstein inequality in Theorem \ref{thm:Bernstein}, we have
\[
\P\left(  \left\vert \frac{1}{n} \sum_{i=1}^n \Big\langle \bv, \Big( \nabla^2 \ell(\btheta_j; \bZ_i) - \E[ \nabla^2 \ell(\btheta_j ;\bZ)] \Big) \bv \Big \rangle \right\vert  \geq \frac{t}{6} \right)  \leq 2e^{- \tilde C_1 n \min \{ \frac{t^2}{ \tau^4}, \frac{t}{\tau^2} \}},
\]
for some universal constant $\tilde C_1$. As a result, 
\[
\P(B_t) \leq 2 \exp\left( - \tilde C_1 n \min\{ \frac{t^2}{\tau^4}, \frac{t}{\tau^2} \}+ p \log \frac{36 \rad }{\eps}\right).
\]

Thus, 
\[
t >\tilde C_2 \max \left \{ \sqrt \frac{ \tau^4(p\log \frac{36 \rad}{\eps} + \log \frac{4}{ \delta})}{n}, \frac{\tau^2 (p \log \frac{36 \rad }{\eps} + \log\frac{4}{ \delta})}{n}  \right \}
\]
for some universal constant $\tilde C_2$ ensures that $\P(B_t) \leq \delta/2$. 

\noindent{\bf Step 3. } \emph{Upper bound $\P(A_t)$ and $\P(C_t)$.}

Let us look at the deterministic event $C_t$ first. According to Assumption \ref{ass:BoundedHessian}, we have
\[
\begin{aligned}
&\sup_{\btheta \in \Ball^p (\bzero, \rad)} \left\Vert \E[\nabla^2 \ell(\btheta; \bZ) - \nabla^2 \ell (\btheta_{j(\btheta)}; \bZ)] \right\Vert_{\op} \\
\leq& \sup_{\btheta \in \Ball^p (\bzero, \rad)} \frac{\left\Vert \E[\nabla^2 \ell (\btheta; \bZ) - \nabla^2 \ell (\btheta_{j(\btheta)}; \bZ)] \right\Vert_{\op} }{\Vert \btheta - \btheta_{j(\btheta)} \Vert_{2}} \cdot \sup_{\btheta \in \Ball^p (\bzero, \rad)}\Vert \btheta - \btheta_{j(\btheta)} \Vert_2\\
\leq & \E \left[ \sup_{\btheta_1\neq \btheta_2 \in \Ball^p(\bzero, \rad)} \frac{\Vert  \nabla^2 \ell(\btheta_1; \bZ) - \nabla^2 \ell(\btheta_2; \bZ)\Vert_{\op}}{\Vert \btheta_1 - \btheta_2 \Vert_2} \right] \cdot \eps\\  \\
\leq & J_* \cdot \eps.
\end{aligned}
\]

We use Markov inequality to bound the event $A_t$. 
\[
\begin{aligned}
\P\left( A_t \right) = & \P \left( \sup_{\btheta \in \Ball^p(\bzero, \rad)} \left \Vert \frac{1}{n}\sum_{i=1}^n  \Big[\nabla^2 \ell(\btheta;\bZ_i) - \nabla^2 \ell(\btheta_{j(\btheta)};\bZ_i) \Big] \right\Vert_{\op} \geq \frac{t}{3} \right) \\
\leq & \frac{3}{t}\E \left [ \sup_{\btheta \in \Ball^p (\bzero, \rad)} \left\Vert \frac{1}{n}\sum_{i=1}^n\Big[ \nabla^2 \ell(\btheta;\bZ_i) - \nabla^2 \ell(\btheta_{j(\btheta)};\bZ_i)  \Big] \right \Vert_{\op} \right] \\
\leq & \frac{3}{t}\E \left [ \sup_{\btheta \in \Ball^p (\bzero, \rad)} \left\Vert  \nabla^2 \ell(\btheta;\bZ) - \nabla^2 \ell(\btheta_{j(\btheta)};\bZ) \right \Vert_{\op} \right] \\
\leq& \frac{3}{t}\E \left[ \sup_{\btheta \in \Ball^p (\bzero, \rad)} \frac{\left\Vert \nabla^2 \ell (\btheta; \bZ) - \nabla^2 \ell (\btheta_{j(\btheta)}; \bZ) \right\Vert_{\op} }{\Vert \btheta - \btheta_{j(\btheta)} \Vert_{2}}\right] \cdot \sup_{\btheta \in \Ball^p (\bzero, \rad)}\Vert \btheta - \btheta_{j(\btheta)} \Vert_2\\
\leq & \frac{3}{t}\E \left[ \sup_{\btheta_1\neq \btheta_2 \in \Ball^p (\bzero, \rad)} \frac{\Vert  \nabla^2 \ell(\btheta_1; \bZ) - \nabla^2 \ell(\btheta_2; \bZ)\Vert_{\op}}{\Vert \btheta_1 - \btheta_2 \Vert_2} \right] \cdot \eps\\  \\
\leq & \frac{3J_* \eps}{t}.
\end{aligned}
\]

Taking $t \geq 6 \eps J_*/\delta$ yields
\[
\P(A_t) \leq \frac{\delta}{2}
\]
and $C_t$ never happens. 

\noindent{\bf Step 4.} \emph{Conclusion.}

Using the above inequalities, to bound the probability of the bad event less than $\delta$, noting that $J_* \leq \tau^3 p^{\Ch}$ by Assumption \ref{ass:BoundedHessian}, it is sufficient to take $\eps = \delta\tau^2 /(6 J_*\cdot np)$ and taking
\begin{align}\label{eq:hessiandelta1}
t \geq  \tau^2 \cdot \max \left\{ \frac{1}{np},  \tilde C_2\sqrt \frac{(p\log \frac{36 \rad \tau np\cdot p^{\Ch}}{\delta} + \log \frac{4}{ \delta})}{n}, \tilde C_2 \frac{ (p \log \frac{36 \rad \tau np \cdot p^{\Ch}}{\delta} + \log\frac{4}{ \delta})}{n}  \right\}
\end{align}
for some universal constant $\tilde C_2$. 

Thus, there exists a universal constant $C_0$, such that letting $C_1 = C_0 \cdot (\Ch \vee \log(\rad \tau /\delta)\vee 1)$,  as long as $n \geq C_1 p \log p$, we have
\begin{align}
\P\left(\sup_{\btheta \in \Ball^p (\bzero, \rad)} \left \Vert \nabla^2 \what R_n(\btheta) - \nabla^2 R(\btheta) \right \Vert_{\op} \geq  \tau^2 \sqrt{\frac{C_1 p \log n}{n}}\right)\leq \delta. 
\end{align}

\section{Proof of Theorem \ref{thm:morse}}

\subsection{Two structural lemmas}

In the following, the index of a symmetric non-degenerate matrix is the number of its negative eigenvalues,
and the index of a non-degenerate critical point $\bx$ of a smooth function $F$ is simply the index of its Hessian 
$\nabla^2 F(\bx)$. 
\begin{lemma}\label{lemma:Stability}
Let $D\subseteq \reals^m$ be a compact set with a $C^2$ boundary $\partial D$, and $f,g:A\to\reals$ be $C^2$ functions
defined on an open set $A$, with $D\subseteq A$. Assume that, for all $\bx\in \partial D$, and all $t\in [0,1]$, 
$t\nabla f(\bx)+(1-t)\nabla g(\bx)\neq \bzero$. Finally, assume that the Hessian $\nabla^2f(\bx)$ is non-degenerate and 
has index equal to $r$ for all $\bx\in D$. Then the following hold:
\begin{enumerate}
\item[$(a)$] If $g$ has no critical point in $D$, then $f$ has no critical point in $D$.
\item[$(b)$] If $g$ has a  unique critical $\bx_0$ point in $D$,  that is non-degenerate with index $r$,
then $f$ also has a unique critical point $\bx_1$ in $D$, with index equal to $r$.
\end{enumerate}
\end{lemma}

The proof is based on a classical result in differential topology (restated here from \cite{dubrovin2012modern}).
Recall that given a smooth vector field $\bxi:D\to\reals^m$, defined on $D\subseteq \reals^m$, a critical point is a point $\bx_0\in D$
such that $\bxi(\bx_0) = \bzero$. If $\bx_0$ is non-degenerate (i.e. the Jacobian matrix of $\bxi$ at $\bx_0$ is full rank), then  index of $\bx_0$ can be defined 
as the sign of the Jacobian determinant\footnote{Note a possible source of confusion with this standard terminology. The index of a critical point of a vector field is in $\{+1,-1\}$, while the index of a
critical point of a scalar function is in $\{0,1,\dots, m\}$.}
\begin{align}
\inde_{\bx_0}(\bxi) = \sign\det\left(\, \frac{\partial\bxi}{\partial\bx}(\bx_0)\,\right)\, . \label{eq:IndexDef}
\end{align}
\begin{lemma}[Theorem 14.4.4 in \cite{dubrovin2012modern}]\label{lem:Index}
Let $D\subseteq \reals^m$ be a compact set with a $C^2$ boundary $\partial D$, and $\bxi:D\to\reals^m$ a $C^1$ vector field,
with a finite number of critical points $\bx_1,\dots,\bx_k\in D$ (possibly $k=0$) and no critical point on the boundary $\partial D$. Define
the Gauss map $\hbxi(\bx) = \bxi(\bx)/\|\bxi(\bx)\|_2$ wherever $\bxi(\bx)\neq \bzero$.

Then 
\begin{align}
\sum_{i=1}^k\inde_{\bx_i}(\bxi) = \deg\big(\hbxi\big|_{\partial D}\big)\, ,
\end{align}
where $\deg(\hbxi|_{\partial D})$ is the degree of the Gauss map restricted to the boundary of $D$.
\end{lemma}
For what follows it is not needed to recall the definition of degree $\deg\big(\hbxi\big|_{\partial D}\big)$. It is only important to note that this depends only
on the restriction of $\bxi$ to $\partial D$.

We are now in position to prove Lemma \ref{lemma:Stability}.
\begin{proof}[Proof of Lemma \ref{lemma:Stability}.]
As a preliminary remark, since the Hessian of $f$ is non-degenerate on $D$, it follows that the 
critical points of $f$ are all isolated. Since $D$ is compact, there can be only a finite number of them, call
them $\bx_1,\bx_2,\dots,\bx_k$.

For $\eps>0$, let $D_{-\eps} = \{\bx\in D:\; d(\bx,D^c)\ge \eps\}$, where $d(\bx,S)\equiv \inf\{\|\bx-\by\|_2:\; \by\in S\}$. 
Also, let $w:A\to[0,1]$ be a $C^1$ function such that $w(\bx)=0$ for $\bx\in A\setminus D$ and $w(\bx) = 1$ for $\bx\in D_{-\eps}$.
Define the following $C^1$ vector fields
\begin{align}
\bxi_0(\bx) & = \nabla g(\bx)\, ,\\
\bxi(\bx) & = (1-w(\bx))\, \nabla g(\bx)+ w(\bx) \, \nabla f(\bx)\, .
\end{align}
Note that $\bxi_0|_{\partial D}  =\bxi|_{\partial D}$. Further, since by assumption 
\begin{align}
\inf_{\bx\in\partial D}\inf_{t\in [0,1]}\|(1-t) \nabla g(\bx)+t \nabla f(\bx)\|_2>0\, ,
\end{align}
by a continuity argument, we can (and will) take $\eps>0$ small enough so that  $\bxi(\bx) \neq \bzero$, $\nabla f(\bx) \neq \bzero$
for all $\bx\in D\setminus D_{-\eps}$. This implies that the critical points of of $\bxi$ are in $D_{-\eps}$
and coincide with the critical points of $f$, i.e. $\bx_1$, \dots, $\bx_k$. Further,  by Eq.~(\ref{eq:IndexDef}),
and since $\bxi(\bx)=\nabla f(\bx)$ for $\bx\in D_{-\eps}$, we have $\inde_{\bx_i}(\bxi)=(-1)^r$.
By applying Lemma \ref{lem:Index}, we get
\begin{align}
(-1)^rk &=\sum_{i=1}^k\inde_{\bx_i}(\bxi) = \deg\big(\hbxi\big|_{\partial D}\big)\\
&= \deg\big(\hbxi_0\big|_{\partial D}\big)\, ,
\end{align}
where the last equality follows because $\bxi|_{\partial D} = \bxi_0|_{\partial D}$. Applying  Theorem \ref{lem:Index}
once more to $\bxi_0$, we get
\begin{align}
(-1)^rk = \deg\big(\hbxi_0\big|_{\partial D}\big) =
\begin{cases}
0 & \mbox{ if $g$ has no critical points,}\\
\inde_{\bx_0}(\bxi_0) = (-1)^r& \mbox{ if $g$ has a unique critical point with index $r$.}
\end{cases}
\end{align}
In the first case we conclude $k=0$, i.e. $f$ has no critical points.
In the second case $k=1$, i.e. $f$ has exactly one critical point. Its index is $r$, because the Hessian
of $f$ has index $r$ at all points in $D$.
\end{proof}

\begin{lemma}\label{lemma:Decomposition}
Let $F:\Ball^m(\rad)\to\reals$ be an $(\eps,\eta)$-strongly Morse function. Denote by $\bx_1,\dots, \bx_k$ its critical points and let
$D = \{\bx\in \Ball^m(\rad):\; \Vert \nabla F(\bx) \Vert_2 < \eps\}$. Then $D$ decomposes into (at most) countably many open connected components,
with each component containing either exactly one critical point, or no critical point.

Explicitly, there exists disjoint open sets $\{D_i\}_{i\in \naturals}$, with $D_i$ possibly empty for $i\ge k+1$, such that
\begin{align}
D = \cup_{i=1}^{\infty} D_i\, .
\end{align}
Further, $\bx_i\in D_i$ for $1\le i\le k$, and each $D_i$, $i\ge k+1$ contains no stationary points.
\end{lemma}
\begin{proof}
Let $D = \cup_{\alpha\in \cA} Q_{\alpha}$ be the decomposition of $D$ into maximal connected components (i.e. the $Q_{\alpha}$ are disjoint and connected).
Note that each $Q_{\alpha}$ is open (because, if $\bx\in Q_{\alpha}$, then there exists a ball centered at $\btheta$ that is also in $D$, and hence in $Q_{\alpha}$).
Hence there must be at most countably many components $Q_{\alpha}$, because each $Q_{\alpha}$ contains a ball of non-zero radius, and $\Ball^m(\rad)$ can contain
only a countable number of balls of non-zero radius. 

We claim that each connected component $Q_{\alpha}$ can contain at most one critical point. Indeed consider the function
$G: \text{Closure}(Q_\alpha) \to\reals_{\ge 0}$ defined by
\begin{align}
G(\bx) =\frac{1}{2}\big\| \nabla F(\bx)\|_2^2\, ,
\end{align}
Note that, by continuity of $\nabla F(\bx)$, we have $\|\nabla F(\bx)\|_2=\eps$  and hence $G(\bx) = \eps^2/2$ for $\bx\in\partial Q_{\alpha}$.
Each critical point of $F(\bx)$ corresponds to a non-degenerate
local minimum of $G(\bx)$ with value $G(\bx) = 0$. If $G$ has
at least two critical  points on $Q_\alpha$, by Morse theory
it must have at least one saddle $\bx_*$ with value
$0<G(\bx_*)<\eps^2/2$. Indeed if this is the case, for all $t$ small enough, the level set $Q_{\alpha}(t) \equiv\{\bx\in Q_{\alpha}: G(\bx)\le t\}$ must have as many
connected components as critical points of $F$, say $m_{\alpha}\ge 2$. However $Q_{\alpha}(\eps^2/2)=Q_\alpha$ has only one connected component.  At the lowest $t$
such that $Q_{\alpha}(t)$ has less than $m_{\alpha}$ components, a critical point must exist on the border of $Q_{\alpha}(t)$ \cite{milnor1963morse}.

At $\bx_*$ we have  $\nabla G(\bx_*) =\bzero$ and $\nabla F(\bx_*)\neq \bzero$. On the other hand, $\nabla G(\bx) = \nabla^2 F(\bx)\, \nabla F(\bx)$, which implies
that $\nabla^2 F(\bx_*)$ has a zero eigenvalue, in contradiction with the 
assumption that implies $\inf_{\bx \in Q_\alpha} \vert \lambda_i(\nabla^2 F(\bx))\vert \geq \eta$.
\end{proof}

\subsection{Proof of Theorem \ref{thm:morse}}

Let $\btheta^{(1)}, \ldots, \btheta^{(k)}$ be the $k$ critical points of $R(\btheta)$, and define
$D = \{\btheta\in \Ball^p(\rad):\; \Vert \nabla R(\btheta) \Vert_2 < \eps\}$. By Lemma \ref{lemma:Decomposition},
$D = \cup_{i=1}^{\infty} D_i$ where each $D_i$ is an open connected component with $\btheta^{(i)}\in D_i$ for $i\le k$  and 
$D_i$ does not contain any critical point of $R(\btheta)$ for $i\ge k+1$.
By continuity of $\nabla R(\btheta)$, we have $\Vert \nabla R(\btheta) \Vert_2 = \eps$ for $\btheta \in \partial D_i$. 

Due to Theorem \ref{thm:uniformconvergence1}, and under the stated condition on $n$ with probability at least $1-\delta$, we have
\begin{align}
\sup_{\btheta\in\Ball^p(\rad)} \big\|\nabla \hR_n(\btheta)- \nabla R(\btheta)\big\|_2& \le \frac{\eps}{2}\, ,\\
\sup_{\btheta\in\Ball^p(\rad)} \big\|\nabla^2 \hR_n(\btheta)- \nabla^2 R(\btheta)\big\|_2&\le \frac{\eta}{2}\, .
\end{align}
We will hereafter assume that this event holds. In particular, we have 
\begin{align}
\inf_{\btheta \in \partial D_i}\Vert t\nabla \hR_n(\btheta) + (1-t) \nabla R(\btheta) \Vert_2 \geq \eps/2,&\quad \forall t \in [0,1],\\
\inf_{\btheta \in D_i} \vert \lambda_i(\nabla^2\what R_n(\btheta))\vert \geq \eta/2.&\label{eq:UniformHessianDi}
\end{align}

By the strong Morse property, the Hessian $\nabla^2 R(\btheta)$, is non-degenerate and has the same index for all $\btheta\in D_i$. 
Denote this index by  $r_i$. By Eq.~(\ref{eq:UniformHessianDi}) the Hessian $\nabla^2 \hR_n(\btheta)$ is also non-degenerate and has index equal to $r_i$ 
for all $\btheta \in D_i$. Due to Lemma \ref{lemma:Stability}, $\hR_n(\btheta)$ the same number of critical points as $R(\btheta)$ in $D_i$.
Namely:
\begin{itemize}
\item For $1\le i\le k$, $\hR_n(\btheta)$ has a unique critical point $\hbtheta^{(i)}$ in $D_i$, with index equal to $r_i$. 
\item For $i\ge k+1$, $\hR_n(\btheta)$ has no critical points in $D_i$.
\end{itemize}
This concludes the proof of part $(a)$ of the theorem.

In order to prove part $(b)$, let $D(\eps_n) =  \{\btheta\in \Ball^p(\rad):\; \|\nabla R(\btheta)\|_2\le \eps_n\}$, with 
$\eps_n = \tau\sqrt{(C p\log n)/n}$ and $C = C_1\cdot (\Ch\vee\log(\rad\tau/\delta)\vee 1)$, with $C_1$ a suitably large absolute constant.
Without loss of generality $\eps_n\le \eps$.
We can repeat the argument of part $(a)$, yielding the decomposition $D(\eps_n)=\cup_{i=1}^{\infty} D_i(\eps_n)$, with $\btheta^{(i)} ,\hbtheta^{(i)} \in D_i(\eps_n)$,
we will next bound the radius of $D_i(\eps_n)$.
For any fixed unit vector $\bu \in \partial \Ball^p(1)$, by Taylor expansion with third order remainder,
for each $\bv\in \reals^p$, with $\bv+\btheta^{(i)} \in\Ball^p(\rad)$, there exists $t\in[0,1]$ such that
\begin{align}
\langle \bu, \nabla R(\btheta^{(i)}+\bv)\rangle = \langle \bu, \nabla R(\btheta^{(i)})\rangle +\langle \bu, \nabla^2 R(\btheta^{(i)})\bv \rangle +\frac{1}{2}\langle\nabla^3 R(\btheta^{(i)}+t\,\bv), \bv\otimes \bv\otimes \bu\rangle\, .
\end{align}
Using $\nabla R(\btheta^{(i)}) =0$ and  $\|\nabla^3 R(\btheta)\|_{\op}\le L$, we get
\begin{align}
\big\|\nabla R(\btheta^{(i)}+\bv)\big\|_2 &\ge \<\bv,\big[\nabla^2 R(\btheta^{(i)})\big]^2\bv\rangle^{1/2}- \frac{1}{2} L\|\bv\|_2^2
\ge \eta \|\bv\|_2 -  \frac{1}{2} L \|\bv\|_2^2\, .
\end{align}
Hence 
\begin{align}
D_i(\eps_n)\subseteq\overline{D}_i(\eps_n)\equiv\Big\{\btheta\in\reals^p:\;\; \eta \|\bv\|_2 -  \frac{1}{2} L \|\bv\|_2^2\le \eps_n\Big\}
\end{align}
Note that for $\eps_n<\eta^2/(2L)$, we have $\overline{D}_i(\eps_n) = \Ball^p(r_0)\cup \Ball^p(r_1)^c$, with $r_0<r_1$. 
Since $D_i(\eps_n)$ is connected by construction, we must have $D_i(\eps_n)\subseteq  \Ball^p(r_0)$. The thesis follows by recalling that
$\hbtheta^{(i)}\in D_i(\eps_n)$, and noting that $r_0\le 2\eps_n/\eta$.

\section{Proof of Theorem \ref{thm:uniformconvergence2}: Very high-dimensional regime}\label{sec:proof_uc2}

\subsection{Proof of Theorem \ref{thm:uniformconvergence2}.$(a)$: Uniform convergence of directional gradient}\label{sec:proof_uc2_a}

We adopt the following general strategy: we identify a radius $r_b>0$ and an associated $\ell_1$-ball $\Ball_1^p(\btheta_0,r_b)$. Outside the ball we use the peeling method, that is, we decompose the set $\Ball_2^p(\bzero,\rad) \setminus \Ball_1^p(\btheta_0, \rb)$ into a finite sequence of regions $\mathbb{K}_l$. In each region, we first bound the expectation of the quantities of interest, and then use concentration inequalities to bound deviations from the expectation. Inside $\Ball_1^p(\btheta_0, \rb)$, we take $\rb$ small enough to ensure a small discretization error. 

\prg{Constructing $\mathbb{K}_l$.} We first bound the quantity outside $\Ball_1^p(\btheta_0,r_b)$. For any integer $l$, define
\begin{equation}
\mathbb K_l = \{ \btheta : 2^{l-1} < \Vert \btheta - \btheta_0
\Vert_1 \leq 2^{l}\}.
\end{equation}
The above definition implies that $\Ball_2^p(\bzero,\rad) \setminus \Ball_1^p(\btheta_0, \rb) \subset \cup_{l=N_-}^{N_+} \mathbb K_l$, where $N_- =\lfloor \log_2(\rb) \rfloor$ and $N_+ = \lceil \log_2(2 \rad \sqrt p) \rceil$. The quantity of interest is
\begin{equation}
D_l = \sup_{\btheta \in \mathbb K_l }\Big| \langle \nabla \what R_n(\btheta) - \nabla R(\btheta), \btheta - \btheta_0 \rangle \Big|
\end{equation}
for all  $l\in \{N_-, N_- + 1,\ldots, N_+ \}$.

\prg{Upper bounding $\E[D_l]$.}
To upper bound $\E[D_l]$ we apply symmetrization techniques and the
Rademacher contraction inequality. 
More formally, let $\bZ_i'$ be independent copies of $\bZ_i$ and let $\eps_i$ be i.i.d. Rademacher variables. We have 
\begin{eqnarray*}
&& \E[D_l] = \E \left[\sup_{\btheta \in \mathbb{K}_l} \left\vert \langle \nabla \what R_n(\btheta)-\nabla R(\btheta), \btheta - \btheta_0  \rangle \right\vert \right]\\
&=& \E_{\bZ} \left[\sup_{\btheta \in \mathbb K_l} \Big\vert \frac{1}{n} \sum_{i=1}^n \langle \nabla \ell(\btheta; \bZ_i) - \E \nabla \ell(\btheta; \bZ), \btheta - \btheta_0 \rangle \Big\vert \right] \\
&\leq& \E_{\bZ,\bZ'} \left[ \sup_{\btheta \in \mathbb K_l} \Big\vert \frac{1}{n} \sum_{i=1}^n \langle \nabla \ell(\btheta; \bZ_i) - \grad\ell(\btheta;\bZ_i') , \btheta - \btheta_0 \rangle \Big\vert \right] \\
&=& \E_{\bZ,\bZ',\eps} \left[ \sup_{\btheta \in \mathbb K_l} \Big\vert \frac{1}{n} \sum_{i=1}^n \langle \eps_i(\nabla \ell(\btheta; \bZ_i) - \grad\ell(\btheta;\bZ_i')) , \btheta - \btheta_0 \rangle \Big\vert \right] \\
&\leq & 2 \E_{\bZ,\eps} \left[\sup_{\btheta \in \mathbb K_l} \Big\vert \frac{1}{n} \sum_{i=1}^n \eps_i \langle \nabla \ell(\btheta; \bZ_i), \btheta - \btheta_0 \rangle \Big\vert \right]\\
&\leq & 2 \E_\bZ \left[ \E_\eps \left[\sup_{\btheta \in \mathbb K_l} \Big\vert \frac{1}{n} \sum_{i=1}^n \eps_i \langle \nabla \ell(\btheta; \bZ_i), \btheta - \btheta_0 \rangle \Big\vert  \Bigg \vert \bZ\right] \right]\\
& = & 2 \E_\bZ \left[ \E_\eps \left[\sup_{\btheta \in \mathbb K_l} \Big\vert \frac{1}{n} \sum_{i=1}^n \eps_i g(\langle\btheta-\btheta_0, \bpsi(\bZ_i) \rangle; \bZ_i) \Big\vert  \Bigg \vert \bZ\right] \right].
\end{eqnarray*}
The last equality is from Assumption \ref{ass:GGL}. 

Now we apply the Rademacher contraction inequality \cite{ledoux2013probability} to bound this quantity, which says that for any set $\mathbb{T}\subset \R^n$ and any family of $L$-Lipschitz functions $\{\phi_i\}_{i\in[n]}$, $\phi_i:\reals\to\reals$, with $\phi_i(0) = 0$,
\begin{equation}
	\E\Big[ \sup_{\bt\in\mathbb{T}} \sum_{i=1}^{n}\eps_i\phi_i(t_i) \Big] \le 2L \cdot \E\Big[ \sup_{\bt \in\mathbb{T}} \sum_{i=1}^{n} \eps_i t_i \Big],
\end{equation}
where $\eps_i$ are i.i.d. Rademacher variables. Since we would like to bound the expectation of supremum of the \emph{absolute value} of the empirical process, we have
\begin{equation}\label{eq:rademacher}
\begin{aligned}
	\E\Big[ \sup_{\bt\in\mathbb{T}} \Big\vert \sum_{i=1}^{n}\eps_i\phi_i(t_i) \Big\vert \Big] =& \E\Big[ \sup_{\bt\in\mathbb{T}}\Big\{ \max\big\{\sum_{i=1}^{n}\eps_i\phi_i(t_i), -\sum_{i=1}^{n}\eps_i\phi_i(t_i)\big\}\Big\} \Big]\\
	\le & \E\Big[ \sup_{\bt\in\mathbb{T}} \sum_{i=1}^{n}\eps_i[\phi_i(t_i)] +  \sup_{\bt \in \mathbb{T}} \sum_{i=1}^{n}\eps_i [-\phi_i(t_i)]  \Big]\\
	\le & 4L \cdot \E\Big[ \sup_{\bt \in\mathbb{T}} \sum_{i=1}^{n} \eps_i t_i \Big].
	\end{aligned}
\end{equation}

Applying the Rademacher contraction inequality, we get
\begin{eqnarray*}
&& \E[D_l] \leq  8 L_* \E_{\bZ} \left[ \E_{\eps}\left[ \sup_{\btheta \in \mathbb K_l} \Big\vert \frac{1}{n} \sum_{i=1}^n \eps_i  \langle  \btheta - \btheta_0 , \bpsi(\bZ_{i})\rangle \Big\vert \Bigg \vert \bZ \right]\right] \\ 
&= & 8 L_* \E_{\bZ} \left[ \E_{\eps}\left[ \sup_{\btheta \in \mathbb K_l} \Big \vert  \langle  \frac{1}{n} \sum_{i=1}^n \eps_i \bpsi(\bZ_i), \btheta - \btheta_0 \rangle \Big \vert  \Bigg \vert \bZ \right]\right]\\
&= & 8 L_* \E  \left[ \left\Vert \frac{1}{n}\sum_{i=1}^n \eps_i \bpsi(\bZ_i) \right\Vert_\infty \right]  \cdot \sup_{\bw \in \mathbb K_l} \Vert \btheta - \btheta_0 \Vert_1 \\
&\leq & 2^{l+3} L_* \E \left[ \left\Vert \frac{1}{n}\sum_{i=1}^n \eps_i \bpsi(\bZ_i) \right\Vert_\infty \right]  \\
&\leq & 2^{l+4} L_* \tau  \sqrt{\frac{\log p}{n}} .
\end{eqnarray*}
The last inequality is due to the fact that $\bpsi(\bZ_i)$ are independent $\tau^2$-sub-Gaussian.

\prg{Concentrating $D_l$ around $\E[D_l]$.} From Assumption \ref{ass:BoundedGradient}, each $\Vert  \grad\ell(\btheta;\bZ_i) \Vert_\infty$ is bounded by $T_*$, thus for $\btheta \in \mathbb K_l$ we have $\vert \langle \grad \ell(\btheta; \bZ_i), \btheta- \btheta_0 \rangle \vert$ is bounded by $T_*\cdot 2^l$. Hence $D_l$ is a $2^l \cdot T_*/n$-bounded variation function in $(\bZ_1,\dots,\bZ_n)$. Applying McDiarmid's inequality, we get
\begin{equation}
\P\left( D_l \geq \E [D_l] + t \right)  \leq \exp\Big( -\frac{2nt^2}{(T_* \cdot 2^l)^2} \Big). 
\end{equation}

Taking $t_l = 2^l  \cdot T_*\sqrt{(\log(N/\delta))/(2n)}$ guarantees that $\P(D_l \ge \E[D_l] + t_l) \le \delta/N$ where $N = N_+ - N_- + 1 \le \log_2(2\rad \sqrt{p}/r_b)+2$.

\prg{Taking union bound over $l$.} Define the event
\begin{equation}
E_l = \left\{ D_l \ge 2^{l+4} L_* \tau  \sqrt{\frac{\log p}{n}} + 2^l T_*\sqrt{\frac{\log\frac{N}{\delta}}{2n}} \right\}.
\end{equation}
We have already shown that $\P(E_l) \le \delta/N$ and so $\P(\cup_{l=N_-}^{N_+} E_l) \le \delta$. On the event $(\cup_{l=N_-}^{N_+} E_l)^c$, for any $\btheta\in\mathbb{K}_l$, we have $\|\btheta - \btheta_0\|_1\ge 2^{l-1}$ and $\vert \langle \grad\what{R}_n(\btheta) - \grad R(\btheta), \btheta-\btheta_0 \rangle \vert \le 2^{l+4} L_* \tau  \sqrt{(\log p)/n} + 2^l T_*\sqrt{(\log(N/\delta))/(2n)}$. Consequently, for all $l$,
\begin{equation}
\sup_{\btheta\in\mathbb{K}_l} \frac{\vert \langle \grad\what{R}_n(\btheta) - \grad R(\btheta), \btheta-\btheta_0 \rangle \vert}{\|\btheta - \btheta_0\|_1} \le 32L_*\tau\sqrt{\frac{\log p}{n}} + 2 T_*\sqrt{\frac{\log\frac{N}{\delta}}{2n}}.
\end{equation}

Noticing that $\Ball_2^p(\bzero,\rad) \setminus \Ball_2^p(\btheta_0, \rb) \subset  \cup_{l=N_-}^{N_+} \mathbb{K}_l$, we see that there exists a universal constant $C_0$ such that letting $C=C_0 \cdot \sqrt{\log(1/\delta)} \cdot (T_*+L_*\tau)$, we have
\begin{equation}
\P\left( \sup_{\btheta \in\Ball_2^p(\bzero,\rad) \setminus \Ball_1^p(\btheta_0, \rb)} \frac{\vert \langle \nabla \what R_n(\btheta)-\nabla R(\btheta), \btheta - \btheta_0  \rangle \vert}{\Vert \btheta - \btheta_0 \Vert_1} \geq C \sqrt{\frac{\log p + \log N}{n}}\right) \leq \delta. 
\end{equation}

\prg{Convergence inside the ball $\Ball_1(\btheta_0,r_b)$.}
We  relate any point in $\Ball_1(\btheta_0,r_b)\setminus\{0\}$ to its projection onto the sphere $\partial\Ball_1(\btheta_0,r_b)$. Notice that
\begin{equation}
  \frac{\< \grad\what{R}_n(\btheta) - \grad R(\btheta), \btheta-\btheta_0 \>}{\|\btheta - \btheta_0\|_1} = \< \grad\what{R}_n(\btheta_0+s\bn) - \grad R(\btheta + s\bn), \bn\>,
\end{equation}
where $s=\|\btheta-\btheta_0\|_1$ and $\bn=(\btheta-\btheta_0)/s \in\partial \Ball_1^p(1)$.

For any vector $\bn \in \partial\Ball_1^p(1)$, we have, for
$r_1,r_2\ge 0$, $\tilde r\in [r_1,r_2]$, $\tilde\btheta = \btheta_0+{\tilde r} \bn$,
\begin{equation}
\begin{aligned}
&\vert  \langle \nabla \what R_n(\btheta_0 + r_1 \bn)-\nabla R(\btheta_0 + r_1 \bn), \bn  \rangle - \langle \nabla \what R_n(\btheta_0 + r_2 \bn)-\nabla R(\btheta_0 + r_2 \bn), \bn \rangle\vert  \\
=& \left\vert \langle \bn, \left(\nabla^2 \what R_n (\btheta_0 + \tilde r \bn) - \nabla^2 R (\btheta_0 + \tilde r \bn)\right)  \bn \rangle\right\vert \cdot \vert r_1 - r_2 \vert\\
\leq& \left\|  \grad^2\what{R}_n(\tilde\btheta) - \grad^2 R(\tilde\btheta) \right\|_\op \cdot \Big(\sup_{\|\bn\|_1=1} \|\bn\|_2^2\Big) \cdot |r_1 - r_2| \\
=& \left\|  \grad^2\what{R}_n(\tilde\btheta) - \grad^2 R(\tilde\btheta) \right\|_\op \cdot |r_1 - r_2|.
\end{aligned} 
\end{equation}
Now, for any $r_1< r_b$ and $r_2=r_b$, the intermediate value $\tilde\btheta\in\Ball_1(\btheta_0,r_b)\in\Ball_2(\rad)$. According to the uniform convergence of Hessians in
 Theorem \ref{thm:uniformconvergence1}.$(b)$, (and more precisely, using Eq.~(\ref{eq:hessiandelta1}) in the proof, which does not use the sample size assumption), 
there exists a constant $C_{\rm hess}$ depending on $(\rad, \tau^2, \Ch, \delta)$, such that
\begin{equation}
  \sup_{\btheta\in\Ball_2^p(\rad)} \Big\| \grad^2\what{R}_n(\btheta) - \grad^2 R(\btheta) \Big\|_\op \le \tau^2  \max\Big\{ \frac{1}{np}, C_{\rm hess}\sqrt{\frac{p\log(np)}{n}}, C_{\rm hess} \frac{p\log(np)}{n} \Big\}
\end{equation}
holds with probability at least $1-\delta$. When this event happens, we can take $r_b=1/(\tau C_{\rm hess}p^2)$ and get
\begin{eqnarray*}
  && \Big\| \grad^2\what{R}_n(\tilde\btheta) - \grad^2 R(\tilde\btheta) \Big\|_\op \cdot |r_1-r_2| \\
  &\le& \tau^2  \max\Big\{ \frac{1}{np}, C_{\rm hess}\sqrt{\frac{p\log(np)}{n}}, C_{\rm hess}\frac{p\log(np)}{n} \Big\} \cdot \frac{1}{ \tau C_{\rm hess}p^2} \\
  &=& \tau  \max\Big\{ \frac{1}{C_{\rm hess} np^3 },\sqrt{\frac{\log(np)}{np^3}}, \frac{\log(np)}{np} \Big\} \\
& \le & \tau\sqrt{\frac{\log(np)}{n}}.
\end{eqnarray*}
It follows that
\begin{eqnarray*}
  && \sup_{\btheta\in\Ball_1(\btheta_0,r_b)\setminus \{0\}} \frac{\vert \< \grad\what{R}_n(\btheta) - \grad R(\btheta), \btheta-\btheta_0 \> \vert}{\|\btheta - \btheta_0\|_1} \le \tau \sqrt{\frac{\log(np)}{n}} + \sup_{\btheta\in\partial\Ball_1(\btheta_0,r_b)} \frac{\vert \< \grad\what{R}_n(\btheta) - \grad R(\btheta), \btheta-\btheta_0 \>\vert}{\|\btheta - \btheta_0\|_1} \\
  &\le& \tau \sqrt{\frac{\log(np)}{n}} + C\sqrt{\frac{\log(p) + \log(N)}{n}} \\
 & \le & \tau \sqrt{\frac{\log(np)}{n}} + C_0 \cdot \sqrt{\log(1/\delta)} \cdot (T_*+L_*\tau)\sqrt{\frac{\log(p) + \log (\log_2(2\rad p^{5/2} \tau C_{\rm hess})+2)}{n}} \\
&\le &(T_* + L_* \tau)\cdot \sqrt{\frac{C_1 \log (np)}{n}},
\end{eqnarray*}
where $C_1$ depends on $(\rad, \tau^2, \Ch, \delta)$. Thus, we get the desired bound with probability at least $1-2 \delta$.


\subsection{Proof of Theorem \ref{thm:uniformconvergence2}.$(b)$: Uniform of convergence of restricted Hessian}\label{sec:proof_uc2_b}

We proceed almost the same as the proof of Theorem \ref{thm:uniformconvergence1}.$(b)$. 

\noindent{\bf Step 1. } \emph{Decompose the bad event using $\eps$-nets.}

Let $\Omega_1 = \Ball_2(\rad)\cap \Ball_0(s_0)$, $\Omega_2 = \Ball_2(1) \cap \Ball_0(s_0)$, and $\Omega = \Omega_1 \times \Omega_2$.

Let $\Theta_{\eps} = \{ \btheta_{1},\ldots, \btheta_{N_{\eps}}\}$ be a minimal $\eps$-covering of the set $\Omega_1 = \Ball_2(\rad)\cap \Ball_0(s_0)$. The size $N_{\eps}$ of the $\eps$-covering set $\Theta_{\eps}$ is bounded by 
\begin{equation}
\begin{aligned}
N_{\eps} \leq {p \choose s_0} \left(\frac{3 \rad}{\eps}\right)^{s_0} 
\leq \exp\left(s_0 \log\left(\frac{3p \rad}{\eps}\right)\right),
\end{aligned}
\end{equation}

For any $\btheta \in \Omega_1$, let $j(\btheta) = \argmin_{j\in [N_{\eps}]} \Vert \btheta - \btheta_j \Vert_2$. Then $\Vert \btheta - \btheta_{j(\btheta)}\Vert_2 \leq \eps$ for all $\btheta \in \Omega_1$. Thus, for any $\btheta \in \Omega_1$, we have
\begin{equation}
\begin{aligned}
\left \vert \left\langle \bv, \left(\nabla^2 \what R_n(\btheta) - \nabla^2 R(\btheta)\right) \bv \right\rangle\right \vert  \leq & \left \vert \left\langle \bv, \left(\frac{1}{n}\sum_{i=1}^n \nabla^2 \ell(\btheta; \bZ_i) - \nabla^2 \ell(\btheta_{j(\btheta)}; \bZ_i) \right)\bv \right\rangle\right \vert  \\
&+ \left \vert \left\langle \bv, \left(\frac{1}{n} \sum_{i=1}^n \nabla^2 \ell(\btheta_{j(\btheta)}; \bZ_i) - \nabla^2 R(\btheta_{j(\btheta)})\right) \bv \right\rangle\right \vert  \\
&+\left \vert  \left\langle \bv, \left(\nabla^2 R(\btheta_{j(\btheta)}) - \nabla^2 R (\btheta)\right)\bv \right\rangle\right \vert . 
\end{aligned}
\end{equation}

Hence, we have
\[
\P\left( \sup_{(\btheta,\bv)\in \Omega}\left\vert  \left\langle \bv, \left(\nabla^2 \what R_n(\btheta) - \nabla^2 R(\btheta)\right) \bv \right \rangle\right \vert  \geq t \right) \leq \P(A_t) + \P(B_t) + \P(C_t),
\]
where the events $A_t$, $B_t$, and $C_t$ are defined as
\[
\begin{aligned}
A_t =& \left\{ \sup_{\btheta \in \Ball^p(\bzero, \rad)} \left \Vert \frac{1}{n}\sum_{i=1}^n \Big[ \nabla^2 \ell(\btheta;\bZ_i) - \nabla^2 \ell(\btheta_{j(\btheta)};\bZ_i) \Big] \right\Vert_{\op} \geq \frac{t}{3} \right\},\\
B_t =& \left\{ \sup_{j \in [N_\eps]}\sup_{\bv \in \Omega_2}  \left\vert \frac{1}{n} \sum_{i=1}^n \Big\langle \bv, \Big( \nabla^2 \ell(\btheta_j; \bZ_i) -  \nabla^2 R(\btheta_j) \Big) \bv \Big \rangle \right\vert \geq \frac{t}{3} \right\},\\
C_t =& \left\{ \sup_{\btheta \in \Ball^p(\bzero, \rad)} \Big \Vert \E[\nabla^2 \ell(\btheta_{j(\btheta)}; \bZ)] - \E[\nabla^2 \ell(\btheta; \bZ)] \Big\Vert_{\op} \geq \frac{t}{3}\right\}.\\
\end{aligned}
\]

\noindent{\bf Step 2. } \emph{Upper bound $\P(B_t)$.}

Let $V_{1/4}$ be a minimal $(1/4)$-covering of the set $\Omega_2 = \Ball_2(1) \cap \Ball_0(s_0)$ with the following property: for any $\bv \in \Omega_2$, there exits an $\bv_j \in V_{1/4}$ such that $\text{supp}(\bv_j) = \text{supp}(\bv)$, and $\Vert \bv - \bv_j \Vert_2 \leq \eps$. The size of the $\eps$-covering set $V_{1/4}$ is bounded by 
\begin{equation}
\begin{aligned}
\vert V_{1/4}\vert \leq& {p \choose s_0} \left(\frac{3}{1/4}\right)^{s_0} 
\leq \exp\left(s_0 \log(12p)\right). 
\end{aligned}
\end{equation}
By Lemma \ref{lem:opnorm}, we have that for any $j \in [N_{\eps}]$, 
\[
\sup_{\bv \in \Omega_2}\left\vert \frac{1}{n} \sum_{i=1}^n \Big\langle \bv, \Big( \nabla^2 \ell(\btheta_j; \bZ_i) -  \nabla^2 R(\btheta_j) \Big) \bv \Big \rangle \right \vert \leq 2 \sup_{\bv \in V_{1/4}} \left\vert \frac{1}{n} \sum_{i=1}^n \Big\langle \bv, \Big( \nabla^2 \ell(\btheta_j; \bZ_i) -  \nabla^2 R(\btheta_j) \Big) \bv \Big \rangle \right\vert .
\]

Taking union bound over $\Theta_\eps$ and $V_{1/4}$ yields
\[
\begin{aligned}
\P(B_t) \leq & \P\left(\sup_{j \in [N_\eps], \bv \in V_{1/4}} \left\vert \frac{1}{n} \sum_{i=1}^n \Big\langle \bv, \Big( \nabla^2 \ell(\btheta_j; \bZ_i) -  \nabla^2 R(\btheta_j) \Big) \bv \Big \rangle \right\vert \geq \frac{t}{6} \right) \\
\leq& e^{s_0 \log\frac{3p \rad}{ \eps} + s_0 \log (12p)} \sup_{j \in [N_\eps] , \bv \in V_{1/4}} \P\left( \left\vert \frac{1}{n} \sum_{i=1}^n \Big\langle \bv, \Big( \nabla^2 \ell(\btheta_j; \bZ_i) - \nabla^2 R(\btheta_j) \Big) \bv \Big \rangle \right\vert \geq \frac{t}{6} \right).
\end{aligned}
\]

Fixing any $j$ and $\bv$, according to Assumption \ref{ass:Hessianub}, $\Big\langle \bv, \Big( \nabla^2 \ell(\btheta_j; \bZ_i) - \E[ \nabla \ell^2(\btheta_j ;\bZ)] \Big) \bv \Big \rangle$ is $\tau^2$-sub-exponential. Hence by Bernstein inequality in Theorem \ref{thm:Bernstein}, we have
\[
\P\left(  \left\vert \frac{1}{n} \sum_{i=1}^n \Big\langle \bv, \Big( \nabla^2 \ell(\btheta_j; \bZ_i) - \E[ \nabla^2 \ell(\btheta_j ;\bZ)] \Big) \bv \Big \rangle \right\vert \geq \frac{t}{6}  \right) \leq 2 e^{- \tilde C_1 n \min \{ \frac{t^2}{ \tau^4}, \frac{t}{\tau^2} \}},
\]
for some universal constant $\tilde C_1$. As a result, 
\[
\P(B_t) \leq  2 \exp\left( - \tilde C_1 n \min\{ \frac{t^2}{\tau^4}, \frac{t}{\tau^2} \}+ s_0 \log \frac{36 \rad p^2}{\eps}\right).
\]

Thus, 
\[
t >\tilde C_2 \max \left \{ \sqrt \frac{ \tau^4(s_0\log \frac{36  \rad p^2}{\eps} + \log \frac{4}{ \delta})}{n}, \frac{\tau^2 (s_0 \log \frac{36 \rad p^2}{\eps} + \log\frac{4}{ \delta})}{n}  \right \}
\]
for some universal constant $\tilde C_2$ ensures that $\P(B_t) \leq \delta/2$. 

\noindent{\bf Step 3. } \emph{Upper bound $\P(A_t)$ and $\P(C_t)$.}

Note the definition of events $A_t$ and $C_t$ is exactly the same as in the proof of Theorem \ref{thm:uniformconvergence1}.$(b)$. Thus, taking $t \geq 6 \eps J_*/\delta$ yields
\[
\P(A_t) \leq \frac{\delta}{2}
\]
and $C_t$ never happens. 

\noindent{\bf Step 4.} \emph{Conclusion.}

Using the above inequalities, to bound the probability of the bad event, it is sufficient to take $\eps = \delta\tau^2/(6 J_*\cdot np)$ and 
\begin{align}\label{eq:hessiandelta}
t \geq \max\left\{ \frac{\tau^2}{np},  \tilde C_2 \sqrt \frac{ \tau^4(s_0\log \frac{36 \rad p^2}{\eps} + \log \frac{4}{ \delta})}{n}, \tilde C_2 \frac{\tau^2 (s_0 \log \frac{36 \rad p^2}{\eps} + \log\frac{4}{ \delta})}{n}  \right\}
\end{align}
for some universal constant $\tilde C_2$. 

According to Assumption \ref{ass:BoundedHessian}, we have $J_* \leq \tau^3 p^{\Ch}$. Thus, there exists a universal constant $C_0$, such that letting $C_2 = C_0 \cdot ( \Ch \vee {\log(\rad \tau /\delta)}\vee 1)$,  as long as $n \geq C_2 s_0 \log (np)$, we have
\begin{align}
\P \left( \sup_{\btheta\in\Ball_2^p(\rad) \cap \Ball_0^p(s_0), \bv \in \Ball_2^p(1) \cap \Ball_0^p(s_0)} \left \vert \left\langle \bv, \left(\nabla^2\hR_n(\btheta)-\nabla^2 R(\btheta)\right) \bv\right \rangle \right \vert \ge  \tau^2 \sqrt{\frac{  C_2 s_0 \log (np)}{n}} \right) \leq \delta\, .
\end{align}


\section{Proofs for binary linear classification}

\subsection{Proof of Theorem \ref{thm:MainClassification}: High-dimensional regime}
\label{app:Class_High}

\subsubsection{Landscape of population risk}
\begin{lemma}\label{lemma:popriskClassification}
Assume $\|\btheta_0\|_2\le \rad/3$ together with Assumption \ref{ass:classification}. Then we have the following:
\begin{enumerate}[label=$(\alph*)$, leftmargin=0.5cm]
\item {\bf Unique minimizer. } The population risk $R(\btheta)$ is minimized at $\btheta=\btheta_0$ and has no other stationary points. 
\item {\bf Bounds on the Hessian. } There exist an $\eps_0>0$ and some constants 
$0<\heslb<\hesub<\infty$ such that
\begin{equation}
\label{eq:PopHessian} 
\inf_{\btheta\in \Ball^d(\btheta_0,\eps_0)} \lambda_{\min}\Big( \grad^2 R(\btheta) \Big) \ge \heslb,\;\;\;\;\;\; \sup_{\btheta\in \Ball^d(\bzero,\rad)} \Big\Vert \grad^2 R(\btheta) \Big\Vert_{\op} \le \hesub.
\end{equation}
\item {\bf Bounds on the gradient. } For the same $\eps_0$ as in part $(b)$, there exist some constants $0<\gradlb<\gradub<\infty$ and $\gradlip\in(0,\infty)$ such that
\begin{equation}
\inf_{\btheta\in \Ball^d(\bzero, \rad)\setminus \Ball^d(\btheta_0,\eps_0)} \Big\| \grad R(\btheta) \Big\|_2 \ge \gradlb,\;\;\;\;\;\;\sup_{\btheta\in \Ball^d(\bzero,\rad)} \Big\| \grad R(\btheta) \Big\|_2 \le \gradub,
\end{equation}
and for all $\btheta\in \Ball^d(\bzero,\rad)$,
\begin{equation}
\langle \btheta-\btheta_0, \grad R(\btheta)\rangle\ge \gradlip \|\btheta-\btheta_0\|_2^2.   \label{eq:GradientDirection} 
\end{equation}
\end{enumerate}
All constants $\eps_0,\heslb,\hesub,\gradlb,\gradub,\gradlip$ are functions of $(\sigma(\cdot),\rad,\tau^2,\actlip,\covlb)$ but do not depend on $d$ and the distribution of $\bX$.
\end{lemma}

\begin{proof}

The proof consists of five parts. Lower bounds of gradient and Hessian are a little involved, and upper bounds are relatively easy to obtain.

\prg{Part $(a)$. No stationary points other than $\btheta_0$.} Fix $\btheta\in \Ball^d(\bzero,\rad)$, then
$\|\btheta\|_2\le \rad$. Let $\bU\in\R^{2\times d}$ be an orthogonal
transform ($\bU\bU^{\sT}=\id_{2\times 2}$) from 
$\reals^d$ to $\R^2$ whose row space contains $\{\btheta,\btheta_0\}$. Define the event $A_s=\{\|\bU\bX\|_2\le2s/(3\rad)\}$. Recall that $\|\btheta_0\|_2\le\rad/3$. Then on the event $A_s$, we have $\max\{|\langle \btheta,\bX\rangle|,|\langle \btheta_0,\bX\rangle|,|\langle \btheta-\btheta_0,\bX\rangle|\}\le s$.

It is easily seen that $R(\btheta)$ is minimized at $\btheta_0$ from the bias-variance decomposition. Moreover,
\begin{eqnarray*}
\langle \btheta-\btheta_0,\grad R(\btheta)\rangle &=& \langle \btheta-\btheta_0, \E[\grad_\btheta(Y-\sigma(\btheta^{\sT}\bX))^2]\rangle \\
&=& \E[2(\sigma(\btheta^{\sT}\bX)-Y)\sigma'(\btheta^{\sT}\bX)\cdot \langle \btheta-\btheta_0,\bX\rangle ] \\
&=& \E[2(\sigma(\btheta^{\sT}\bX)-\sigma(\btheta_0^{\sT}\bX))\sigma'(\btheta^{\sT}\bX)\cdot \langle \btheta-\btheta_0,\bX\rangle].
\end{eqnarray*}

Notice that $(\sigma(t_1)-\sigma(t_2))(t_1-t_2)\ge 0$ for all $t_1,t_2\in\R$, so the quantity inside the above expectation is always nonnegative. In addition,
since by Assumption \ref{ass:classification}.$(a)$, $\sigma'$ is positive on $\R$, so for any $s>0$ there exists some $L(s)>0$ such that $\inf_{t\in[-s,s]}\sigma'(t)\ge L(s)$. Hence, by the intermediate value theorem,
\begin{eqnarray*}
\langle \btheta-\btheta_0,\grad R(\btheta) \rangle &\ge& \E[2(\sigma(\btheta^{\sT}\bX)-\sigma(\btheta_0^{\sT}\bX))\sigma'(\btheta^{\sT}\bX)\cdot \langle \btheta-\btheta_0,\bX\rangle \ones_{A_s}] \\
&\ge& 2L^2(s) \E[\langle \btheta-\btheta_0,\bX\rangle^2 \ones_{A_s}].
\end{eqnarray*}
From assumption \ref{ass:classification}$, \E[\bX\bX^{\sT}]\succeq
\covlb\tau^2\id_{d\times d}$, so $\E[\langle
\btheta-\btheta_0,\bX\rangle^2] \ge \covlb \tau^2\|\btheta-\btheta_0\|_2^2$. Hence, we can always find a sufficiently large $s$ such that $\E[\langle \btheta-\btheta_0,\bX\rangle^2\ones_{A_s}]\ge (\covlb \tau^2/2)\cdot \|\btheta-\btheta_0\|_2^2$. For this $s$, the above lower bound is greater than 0, so the gradient $\grad R(\btheta)$ cannot be zero. Hence, the risk $R(\btheta)$ has no other stationary points.

\prg{Part $(b)$. Lower bounding the Hessian.} Recall that $\grad^2 R(\btheta)=\E[\beta(\btheta)\bX\bX^{\sT}]$ with
\begin{equation}\label{eqn:beta_def}
\beta(\btheta) = 2\Big( \sigma'(\btheta^{\sT}\bX)^2 + (\sigma(\btheta^{\sT}\bX)-\sigma(\btheta_0^{\sT}\bX))\sigma''(\btheta^{\sT}\bX) \Big).
\end{equation}
(Note we changed $Y$ to $\sigma(\btheta_0^{\sT}\bX)$ using the tower property.) Our general strategy to lower bound the minimum eigenvalue of $\grad^2 R(\btheta)$ is to first lower bound the $\lambda_{\min}(\grad^2 R(\btheta_0))$ and then upper bound $\lambda_{\max}(\grad^2 R(\btheta) - \grad^2 R(\btheta_0))$.

Let us first consider $\lambda_{\min}(\grad^2 R(\btheta_0))$. We have that $\grad^2 R(\btheta_0)=\E[2\sigma'(\btheta_0^{\sT}\bX)^2\bX\bX^{\sT}]$. Fix any $\bu\in\R^d$, $\|\bu\|_2=1$. Similar to part $(a)$, let $A_s=\{|\langle \btheta_0/\|\btheta_0\|_2, \bX\rangle| \le 2s/\rad\}$, then
\begin{eqnarray*}
\langle \bu,\grad^2 R(\btheta_0)\bu\rangle &\ge& \E[2\sigma'(\btheta_0^{\sT}\bX) \langle \bu,\bX\rangle^2\ones_{A_s}] \\
&\ge& 2L^2(s)\E[\langle \bu,\bX\rangle^2\ones_{A_s}] \\
&\ge& 2L^2(s)\Big( \E[\langle \bu,\bX\rangle^2] - \E[\langle \bu,\bX\rangle^2\ones_{A_s^c}] \Big).
\end{eqnarray*}
Note that $\E[\langle \bu,\bX\rangle^2]\ge \covlb \tau^2$, $\E[\langle \bu,\bX\rangle^4]\le C_4\tau^4$, and $\P(A_s^c)\le 2\exp(-2s^2/(\rad^2\tau^2))$. By the Cauchy-Schwarz inequality, we have
\begin{equation}
\langle \bu,\grad^2 R(\btheta_0)\bu\rangle \ge 2L^2(s)\tau^2(\covlb - \sqrt{2C_4}e^{-\frac{s^2}{\rad^2\tau^2}}).
\end{equation}
Choosing $s=\tilde{c}\rad\tau$ for some constant $\tilde{c}$ gives us a lower bound $\covlb\tau^2L^2(\tilde{c}\rad\tau)$ on $\lambda_{\min}(\grad^2 R(\btheta_0))$.

Now let's turn to the difference $\grad^2 R(\btheta) - \grad^2 R(\btheta_0)$. Observe that
\begin{equation}
\grad^2 R(\btheta) - \grad^2 R(\btheta_0) = \E[(\beta(\btheta) - \beta(\btheta_0))\bX\bX^{\sT}].
\end{equation}
Since $\beta$ is $ L_\beta$-Lipschitz ($L_\beta$ only depends on $\actlip$) with respect to $\btheta^{\sT}\bX$, we have that, for any unit vector $\bu\in\R^d$,
\begin{eqnarray*}
\Big| \langle \bu,(\grad^2 R(\btheta) - \grad^2 R(\btheta_0))\bu \>\Big| &\le& L_\beta\E[|\langle \btheta-\btheta_0,\bX\rangle \cdot \langle \bu,\bX\rangle^2|] \\
&\le& L_\beta\Big( \E[\langle \btheta-\btheta_0,\bX\rangle^2] \cdot \E[\langle \bu,\bX\rangle^4] \Big)^{1/2} \\
&\le& L_\beta\Big( \|\btheta-\btheta_0\|_2^2\tau^2 \cdot C_4\tau^4 \Big)^{1/2} \\
&=& L_\beta\sqrt{C_4} \cdot \|\btheta-\btheta_0\|_2\tau^3.
\end{eqnarray*}

Hence, whenever $\|\btheta-\btheta_0\|_2\le \eps_0\defeq \tilde{\tilde{c}}L^2(\tilde{c}\rad\tau)/(L_\beta\tau)$ for some universal constant $\tilde{\tilde{c}}$ guarantees that $\lambda_{\max}(\grad^2 R(\btheta) - \grad^2 R(\btheta_0))\le(1/2)\cdot\lambda_{\min}(\grad^2 R(\btheta_0))$. Consequently, for all $\|\btheta-\btheta_0\|_2\le \eps_0$,
\begin{equation}
\lambda_{\min}(\grad^2 R(\btheta)) \ge \heslb = \frac{\covlb}{2}\tau^2L^2(\tilde{c}\rad\tau).
\end{equation}

\prg{Part $(b)$. Upper bounding the Hessian.} For any $\btheta\in\R^d$, we have
\begin{eqnarray*}
\| \grad^2 R(\btheta) \|_\op &=& \| \E[\beta(\btheta)\bX\bX^{\sT}] \|_\op = \sup_{\|\bv\|_2=1} \Big| \langle \bv, \E[\beta(\btheta)\bX\bX^{\sT}]\cdot \bv\rangle \Big| = \sup_{\|\bv\|_2=1} \Big| \E[\beta(\btheta)\langle \bv,\bX\rangle^2] \Big| \\
&\le& \sup_{\|\bv\|_2=1} \E[|\beta(\btheta)| \cdot \langle \bv,\bX\rangle^2] \le C_\beta \tau^2.
\end{eqnarray*}
where $C_\beta$ only depends on $\actlip$. Hence, $\hesub=C_\beta \tau^2$ is a global upper bound for Hessian.

\prg{Part $(c)$. Lower bounding the gradient.} In part $(a)$, the lower bound of the gradient depends on the distribution of $\bX$, so it is not distribution-free. Now we give a distribution free lower bound, for any $\btheta \in \Ball^d(\bzero, \rad) \setminus \Ball^d(\btheta_0, \eps_0)$. We have
\begin{eqnarray*}
\langle \btheta-\btheta_0,\grad R(\btheta)\rangle &\ge& 2L^2(s)\Big( \E[\langle \btheta-\btheta_0,\bX\rangle^2] - \E[\langle \btheta-\btheta_0,\bX\rangle^2 \ones_{A_s^c}] \Big) \\
&\ge& 2L^2(s)\Big( \covlb \tau^2\|\btheta-\btheta_0\|_2^2 - \Big( \E[\langle \btheta-\btheta_0,\bX\rangle^4] \cdot \P(A_s^c) \Big)^{1/2} \Big) \\
&\ge& 2L^2(s) \|\btheta-\btheta_0\|_2^2\tau^2 \Big( \covlb - \sqrt{C_4\cdot \P(A_s^c)} \Big).
\end{eqnarray*}
In addition,
\begin{equation}
\P(A_s^c) = \P\Big( \|\bU\bX\|_2>\frac{2s}{3\rad} \Big) \le \sum_{j=1}^{2} \P\Big( |\langle \bU_j,\bX\rangle| \ge \frac{\sqrt{2}s}{3\rad} \Big) \le 4\exp\Big( -\frac{s^2}{9\rad^2\tau^2} \Big),
\end{equation}
giving us
\begin{equation}
\langle \btheta-\btheta_0,\grad R(\btheta) \rangle \ge 2L^2(s)\|\btheta-\btheta_0\|_2^2\tau^2\Big( \covlb - 2\sqrt{C_4}e^{-\frac{s^2}{18\rad^2\tau^2}} \Big).
\end{equation}
So choosing $s\ge \tilde{c}\rad\tau$ for some constant $\tilde{c}>0$
and $\gradlip=\covlb L^2(\tilde{c}\rad\tau)\tau^2$ ensures that
\begin{equation}
\langle \btheta-\btheta_0,\grad R(\btheta)\rangle \ge \gradlip \|\btheta-\btheta_0\|_2^2
\end{equation}
and also $\|\grad R(\btheta)\|_2\ge \gradlip\|\btheta-\btheta_0\|_2$ from the Cauchy-Schwarz inequality. 

Finally, for the $\eps_0$ chosen in part $(b)$, choosing $\gradlb=\eps_0\gradlip$ ensures that $\|\grad R(\btheta)\|_2\ge\gradlb$ for all $\btheta\in \Ball^d(\bzero,\rad)\setminus \Ball^d(\btheta_0,\eps_0)$.

\prg{Part $(c)$. Upper bounding the gradient.} For any $\btheta\in\R^d$, we have
\begin{eqnarray*}
\| \grad R(\btheta) \|_2 &=& \| \E[\alpha(\btheta)\bX] \|_2 = \sup_{\|\bv\|_2=1} \langle \bv, \E[\alpha(\btheta)\bX] \rangle = \sup_{\|\bv\|_2=1} \E[\alpha(\btheta)\langle \bv,\bX\rangle] \\
&\le& \sup_{\|\bv\|_2=1} \E[|\alpha(\btheta)| \cdot |\langle \bv,\bX\rangle|] \le 2 \actlip \tau.
\end{eqnarray*}
Thus, $\gradub=2 \actlip \tau$ upper bounds $\|\grad R(\btheta)\|_2$ for all $\btheta\in\R^d$.

\prg{Dependence on model parameters.} Notice that all constants $\gradlb,\gradub,\gradlip,\heslb,\hesub,\eps_0$ does not depend on $d$ and the distribution of $\bX$. This completes the proof of all of our statements.

\end{proof}

\subsubsection{Landscape of empirical risk}
\begin{lemma}\label{lemma:empriskClassification}
Under Assumption \ref{ass:classification}, let 
$\eps_0$, $\heslb$, $\hesub$, $\gradlb$, $\gradub$, $\gradlip$ be the constants defined in Lemma \ref{lemma:popriskClassification}.$(b)$ depending on $(\sigma(\cdot),\rad,\tau^2,\actlip,\covlb)$, then there exists a large positive constants $C$ depending on $(\sigma(\cdot),\rad,\tau^2,\actlip,\covlb,\delta)$, such that as $n  \geq C d \log d$, the following hold with probability at least $1 - \delta$:
\begin{enumerate}[label=$(\alph*)$, leftmargin=0.5cm]
\item {\bf Bounds on the Hessian. } 
\begin{equation}
\label{eq:PopHessian} 
\inf_{\btheta\in \Ball^d(\btheta_0,\eps_0)} \lambda_{\min}\Big( \grad^2 \what R_n(\btheta) \Big) \ge \heslb/2,\;\;\;\;\;\; \sup_{\btheta\in \Ball^d(\bzero,\rad)} \Big\Vert \grad^2 \what R_n(\btheta) \Big\Vert_{\op} \le 2\hesub.
\end{equation}
\item {\bf Bounds on the gradient. } 
\begin{equation}
\inf_{\btheta\in \Ball^d(\bzero, \rad)\setminus \Ball^d(\btheta_0,\eps_0)} \Big\| \grad \what R_n(\btheta) \Big\|_2 \ge \gradlb/2,\;\;\;\;\;\;\sup_{\btheta\in \Ball^d(\bzero,\rad)} \Big\| \grad \what R_n(\btheta) \Big\|_2 \le 2\gradub,
\end{equation}
and for all $\btheta\in \Ball^d(\bzero,\rad) \setminus \Ball^d(\btheta, \eps_0/2)$,
\begin{equation}
\langle \btheta-\btheta_0, \grad \what R_n(\btheta)\rangle\ge \frac{\gradlip}{4} \eps_0 \|\btheta-\btheta_0\|_2.   \label{eq:GradientDirection} 
\end{equation}
\item {\bf Unique minimizer. } The empirical risk $\what R_n(\btheta)$ is minimized at $\hat \btheta_n \in \Ball^d(\btheta_0, C\sqrt{d \log n/n})$ 
\end{enumerate}
\end{lemma}

\begin{proof}
Let $\eps_0$, $\gradlip$, $\gradub$, $\gradlb$, $\hesub$, and $\heslb$ be the constants defined in Lemma \ref{lemma:popriskClassification}. We begin by verifying the conditions for Theorem \ref{thm:uniformconvergence1}, i.e., Assumptions \ref{ass:GradientSub}, \ref{ass:Hessianub}, and \ref{ass:BoundedHessian}.

\vskip 0.1cm
\noindent{\bf Assumption \ref{ass:GradientSub}. }
We would like to verify that the directional gradient of the loss is sub-Gaussian. The directional gradient of the loss gives
\begin{align}\label{eq:gradsubgaussian1}
\langle \nabla \ell(\btheta;\bZ), \bv \rangle = \alpha(\btheta) \langle \bX, \bv \rangle, 
\end{align}
where $\alpha(\btheta) = -2(Y - \sigma(\langle \btheta, \bX \rangle))\sigma'(\langle \btheta,\bX \rangle)$ whose absolute value is bounded by $2 \actlip$. By Assumption \ref{ass:classification}.$(b)$, $\langle \bX, \bv \rangle$ is mean zero and $\tau^2$ sub-Gaussian. Due to Lemma \ref{lem:subgaussian}.$(d)$, there exists a universal constant $C_1$, such that $\langle \nabla \ell(\btheta;Z), \bv \rangle$ is $C_1 \actlip \tau^2$-sub-Gaussian. 

\vskip 0.1cm
\noindent{\bf Assumption \ref{ass:Hessianub}. }
We would like to verify that the directional Hessian of the loss is sub-exponential. The directional Hessian of the loss gives
\begin{align}
\langle \bv, \nabla^2 \ell(\btheta;\bZ)\bv\rangle  = \beta(\btheta) \langle \bX, \bv \rangle^2,
\end{align}
where $\beta(\btheta)$ is given in equation (\ref{eqn:beta_def}) whose absolute value is by $2(\actlip^2 + \actlip)$. Since $\langle \bX, \bv \rangle$ is mean-zero and $\tau^2$-sub-Gaussian, according to Lemma \ref{lem:subgaussian}.$(c)$, $\langle \bv, \bX \rangle^2$ is $\tau^2$-sub-exponential. Due to Lemma \ref{lem:subgaussian}.$(e)$, there exists a universal constant $C_2$, such that $\langle \bv, \nabla^2 \ell(\btheta;\bZ)\bv\rangle$ is $C_2 (\actlip^2 + \actlip) \tau^2$-sub-exponential. 

\vskip 0.1cm
\noindent{\bf Assumption \ref{ass:BoundedHessian}.}
We need to verify that there exists a constant $\Ch$ which does not depend on $d$, such that $\HUB\leq \tau^2 d^{\Ch}$ and $J_* \leq \tau^3 d^{\Ch}$ (as $d \geq 2$).  
\[
\begin{aligned}
\HUB =& \Vert \nabla^2 R(\btheta_0)\Vert_{\op} \\
=& \sup_{\Vert \bv \Vert_2 = 1}\E[\beta(\btheta_0) \langle \bX, \bv \rangle^2 ]\\
\leq& 2 (\actlip^2 + \actlip) \sup_{\Vert \bv \Vert_2 = 1} \E [\langle X, \bv \rangle^2] =  2 (\actlip^2 + \actlip) \tau^2, \\
J_* =& \E\left[\sup_{\btheta_2 \neq \btheta_2} \frac{\Vert (\beta(\btheta_1) - \beta(\btheta_2) )\bX \bX^\sT \Vert_{\op}}{\Vert \btheta_1 -\btheta_2 \Vert_2}\right]\\
\leq&  \E \left[ \sup_{\btheta} \vert \gamma(\btheta) \vert \cdot \sup_{\btheta_1\neq \btheta_2}\frac{\langle \btheta_1 - \btheta_2, \bX \rangle}{\Vert \btheta_1 - \btheta_2 \Vert_2} \cdot \sup_{\Vert v \Vert_2 = 1} \langle \bv, \bX \rangle^2 \right] \\
\leq &  (6 \actlip^2 + 2 \actlip) \cdot \E \Vert \bX \Vert_2^3\\
\leq & C_3 (\actlip^2 + \actlip)  (d \tau^2)^{3/2},
\end{aligned}
\]
where $\gamma(\btheta) = 2(3\sigma'( \btheta^\sT \bX ) \sigma''( \btheta^\sT \bX) + (\sigma( \btheta^\sT \bX) -\sigma( \btheta_0
^\sT \bX)) \sigma'''( \btheta^\sT \bX ))$ which is bounded by $6 \actlip^2 + 2 \actlip$, and $C_3$ is a universal constant.

Therefore, in Theorem \ref{thm:uniformconvergence1}, the Assumptions \ref{ass:GradientSub} and \ref{ass:Hessianub} are satisfied with sub-Gaussian and sub-exponential parameters $\max\{C_1, C_2\} \cdot (\actlip^2 + \actlip)\tau^2$, and the Assumption \ref{ass:BoundedHessian} is satisfied with parameter $\Ch = \max\{ \log_2(2(\actlip^2 + \actlip)), 3/2 + \log_2(C_3(\actlip^2 + \actlip))\}$. Now we take $\eps_g = \min\left\{ \gradlb/2,  \gradlip \eps_0/4\right\}$, $\eps_h \leq \heslb/2$ depending on $(\rad, \tau^2, \actlip, \covlb)$ but independent of $(n, d)$. 
According to the uniform convergence of the gradient and the Hessian in Theorem \ref{thm:uniformconvergence1}, there exists a constant $C$ depending on $(\sigma(\cdot),\rad, \tau^2, \actlip, \covlb,\delta)$ but independent of $(n,d)$, such that as $n$ is large enough when $n \geq C d \log {d}$, with probability at least $1-\delta$, the following good event happens: 
\begin{equation}\label{eq:good_event_bc}
E_{\rm good} = \left\{
\begin{aligned}
&\sup_{\btheta \in \Ball^d(\bzero,\rad)} \left\Vert \nabla \what R_n(\btheta) - \nabla R(\btheta) \right\Vert_2 \leq \tau \sqrt{\frac{C \cdot d \log n}{n}} \leq \eps_g,\\
&\sup_{\btheta \in \Ball^d(\bzero,\rad)} \left\Vert \nabla^2 \what R_n(\btheta) - \nabla^2 R(\btheta) \right\Vert_{\op} \leq \tau^2 \sqrt{\frac{C \cdot d  \log n}{n}} \leq \eps_h.
\end{aligned}
\right\}.
\end{equation}

All the following arguments are deterministic on the good event $E_{\rm good}$. 

\prg{Part $(a)$.} For the the least eigenvalue of the empirical Hessian in $\Ball^d(\btheta_0, \eps_0)$, we have
\[
\begin{aligned}
\inf_{\btheta \in \Ball^d(\btheta_0, \eps_0)} \lambda_{\min}( \nabla^2 \what R_n(\btheta) )\geq &  \inf_{\btheta \in \Ball^d(\btheta_0, \eps_0)} \lambda_{\min} ( \nabla^2 R(\btheta) ) - \sup_{\btheta \in \Ball^d(\btheta_0, \eps_0)} \Vert \nabla^2 \what R_n(\btheta) - \nabla^2 R(\btheta) \Vert_{\op} \\
\geq & \heslb - \eps_h \geq \frac{1}{2} \heslb >0.
\end{aligned}
\]
This leads to the conclusion that, $\what R_n(\btheta)$ is strongly convex inside the region $\Ball^d(\btheta_0, \eps_0)$. 

For the operator norm of the empirical Hessian in $\Ball^d(\bzero, \rad)$, we have
\[
\begin{aligned}
\sup_{\btheta \in \Ball^d(\bzero, \rad)}\Big \Vert \nabla^2 \what R_n(\btheta) \Big \Vert_{\op} \leq & \sup_{\btheta \in \Ball^d(\bzero, \rad)} \Vert \nabla^2 R(\btheta) \Vert_{\op} + \sup_{\btheta \in \Ball^d(\bzero, \rad)} \Vert \nabla^2 \what R_n(\btheta) - \nabla^2 R(\btheta) \Vert_{\op} \\
\leq & \hesub + \eps_h \leq 2 \hesub.
\end{aligned}
\]

\prg{Part $(b)$.}  
For the lower bound of the gradient in $\Ball^d(\bzero, \rad) \setminus \Ball^d(\btheta_0, \eps_0)$, we have
\[
\begin{aligned}
\inf_{\btheta \in \Ball^d(\bzero, \rad)\setminus \Ball^d(\btheta_0, \eps_0)} \Big \Vert \nabla \what R_n(\btheta) \Big \Vert_2 \geq& \inf_{\btheta \in \Ball^d(\bzero, \rad)\setminus \Ball^d(\btheta_0, \eps_0)} \Big \Vert \nabla R(\btheta) \Big \Vert_2 - \sup_{\btheta \in \Ball^d(\bzero, \rad)} \Big \Vert \nabla \what R_n(\btheta) -  \nabla R_n(\btheta) \Big\Vert_2 \\
\leq & \gradlb - \eps_g \geq \gradlb/2. 
\end{aligned}
\]

For the upper bound of the gradient in $\Ball^d(\bzero, \rad)$, we have
\[
\begin{aligned}
\sup_{\btheta \in \Ball^d(\bzero, \rad)} \Big \Vert \nabla \what R_n(\btheta) \Big \Vert_2 \leq& \sup_{\btheta \in \Ball^d(\bzero, \rad)} \Big \Vert \nabla R(\btheta) \Big \Vert_2 + \sup_{\btheta \in \Ball^d(\bzero, \rad)} \Big \Vert \nabla \what R_n(\btheta) -  \nabla R_n(\btheta) \Big\Vert_2 \\
\leq & \gradub + \eps_g \leq 2 \gradub. 
\end{aligned}
\]

For the lower bound of the directional empirical gradient $\langle \nabla \what R_n(\btheta), \btheta - \btheta_0 \rangle/\Vert \btheta - \btheta_0 \Vert_2$ in $\Ball^d(\bzero,\rad) \setminus \Ball^d(\btheta_0,  \eps_0/2)$, we have
\[
\begin{aligned}
&\inf_{\btheta \in \Ball^d(\bzero,\rad) \setminus \Ball^d(\btheta_0,  \frac{1}{2}\eps_0) } \frac{\langle \nabla \what R_n(\btheta), \btheta-\btheta_0 \rangle}{\Vert \btheta - \btheta_0 \Vert_2} \\ 
\geq&  \inf_{\btheta \in \Ball^d(\bzero,\rad) \setminus \Ball^d(\btheta_0,  \frac{1}{2}\eps_0) } \frac{\langle \nabla R(\btheta), \btheta-\btheta_0 \rangle}{\Vert \btheta - \btheta_0 \Vert_2} - \sup_{\btheta \in \Ball^d(\bzero,\rad) \setminus \Ball^d(\btheta_0,  \frac{1}{2}\eps_0) } \Vert \nabla \what R_n(\btheta) - \nabla R(\btheta) \Vert_2\\
\geq & \inf_{\btheta \in \Ball^d(\bzero,\rad) / \Ball^d(\btheta_0, \frac{1}{2}\eps_0) } \gradlip \Vert \btheta - \btheta_0 \Vert_2 - \frac{1}{4}\gradlip \eps_0 \\
\geq & \frac{1}{2}\gradlip \eps_0 - \frac{1}{4}\gradlip \eps_0 = \frac{1}{4} \gradlip \eps_0 >0.
\end{aligned}
\]

\prg{Part $(c)$. }
Note that by part $(b)$, there is no local minimizer in the interior of $\Ball^d(\bzero, \rad) \setminus \Ball^d(\btheta_0, \eps_0/2)$ (because otherwise, the directional gradient would vanish there). Also, there is no local minimizer on the
boundary of $\Ball^d(\bzero, \rad)$. Indeed, if $\btheta$ was such a minimizer, we would gave $\nabla\hR_n(\btheta) = \alpha\btheta$
for some $\alpha\ge 0$, whence $\<\nabla \hR_n(\btheta),\btheta-\btheta_0 \>\le 0$ contradicting the above.
Hence any local minimizer of $\hR_n(\btheta)$ must be in $\Ball^d(\btheta_0, \eps_0/2)$. By strong convexity there can be at most one such point.

Let that $\hat \btheta_n \in \Ball^d(\btheta_0, \eps_0/2)$ denotes the unique local minimizer. Note that, by the intermediate value theorem, there exists $\btheta' \in\Ball^d(\btheta_0,\eps_0/2)$ such that
\begin{align}
\hR_n(\hat \btheta_n) = \hR_n(\btheta_0)+\<\nabla\hR_n(\btheta_0),\hat \btheta_n-\btheta_0\> + \frac{1}{2}\<\nabla^2\hR_n(\btheta'),(\hat \btheta_n-\btheta_0)^{\otimes 2}\> \le  \hR_n(\btheta_0)\, .
\end{align}
where the inequality follows by optimality of $\hR_n(\hat \btheta_n)$. Using Cauchy-Schwarz inequality, the lower bound on the Hessian
in point $(b)$, and the uniform convergence of the gradient, we get 
\[
\begin{aligned}
\|\hat \btheta_n-\btheta_0\|_2\le& \frac{4\|\nabla\hR_n (\btheta_0)\|_2}{\heslb}\\
\leq & \frac{4 \tau}{ \heslb} \sqrt{\frac{C \cdot d \log n}{n}}.
\end{aligned}
\]

\end{proof}

\subsubsection{Gradient descent algorithm}

\begin{lemma}\label{lemma:GDalgorithm}
Under Assumption \ref{ass:classification}, and $\Vert \btheta_0 \Vert_2 \leq \rad /3$, there exist constants $C$
and $h_{\max}$ depending on $(\sigma(\cdot),\rad,\tau^2,\actlip, \covlb, \delta)$, such that as $n  \geq C d \log d$, with probability at least $1 - \delta$, gradient descent with fixed  step size $h_k=h \leq h_{\max}$ converges exponentially fast to the global minimizer, for any initialization $\btheta_s \in \Ball^d(\btheta_0, 2 \rad /3)$:
$\|\hbtheta_n (k)-\hbtheta_n\|_2\le C\|\btheta_s-\hbtheta_n\|_2\, (1- h/C)^k$.  
\end{lemma}

\begin{proof}

We have already shown that there is an area where the empirical risk is strongly convex inside and the directional gradient is lower bounded outside. Convergence of gradient descent is established by considering the two phases accordingly: outside the area, the problem is non-convex, and we get an exponential convergence using a strong quasi-convexity type argument; inside the area, we get an exponential convergence as we are essentially minimizing a strongly convex function. Connecting the two arguments gives a global convergence result. 

\noindent
{\bf Step 1.} Conditioning on the good event. 

Notice that we are making the same assumptions as in Lemma \ref{lemma:empriskClassification}. Hence event $E_{\rm good}$ in Equation (\ref{eq:good_event_bc}) holds with probability at least $1-\delta$ provided $n \geq C d \log d$ where $C$ is defined in Lemma \ref{lemma:empriskClassification}. All the conclusions in Lemma \ref{lemma:empriskClassification} holds, and all the following arguments are deterministic on $E_{\rm good}$. Note that we already proved that, on $E_{\rm good}$, there is a unique minimizer of empirical risk which is inside $\Ball^d(\btheta_0, \eps_0/2)$.

\noindent{\bf Step 2. } Establish an exponential convergence outside the ball $\Ball_2^d(\btheta_0, \eps_0/2)$. 

Let $\btheta_n(k)$ be the $k$-th iterate of gradient descent, defined via
\[
\btheta_n(k+1) = \btheta_n(k) - h \nabla \what R_n(\btheta_n(k)).
\]
In this part we assume that we initialize at $\btheta_n(0)\notin \Ball_2^d(\btheta_0,\eps_0)$ and all the iterates up to $\btheta_n(k)$ are outside the ball $\Ball_2^d(\btheta_0,\eps_0/2)$ and show that the gradient descent will converge exponentially to the ball $\Ball_2^d(\btheta_0,\eps_0/2)$.

By simple algebraic manipulations, we have 
\begin{eqnarray}
  && \Vert \btheta_n(k+1) - \btheta_0\Vert_2^2 - \Vert \btheta_n(k) - \btheta_0 \Vert_2^2 \nonumber\\
  &=& \|\btheta_n(k)-h\grad\what{R}_n(\btheta_n(k))-\btheta_0\|_2^2 - \|\btheta_n(k)-\btheta_0\|_2^2 \nonumber \\
  &=& - 2 h \langle \nabla \what R_n(\btheta_n(k)), \btheta_n(k) - \btheta_0\rangle + h^2 \Vert \nabla \what R_n(\btheta_n(k)) \Vert_2^2 \label{eqn:gd_algebra}.
\end{eqnarray}
First, we are going to lower bound the inner product term. By Lemma \ref{lemma:popriskClassification}.$(c)$ we have the inequality
\begin{equation}
  \< \grad R(\btheta), \btheta-\btheta_0 \> \ge \gradlip \|\btheta - \btheta_0\|_2^2
\end{equation}
for all $\btheta\in \Ball_2^d(\btheta_0, 2r/3) \subset \Ball_2^d(\bzero,\rad)$. Applying this inequality and using our uniform convergence result, we get
\begin{eqnarray*}
  && \< \grad\what{R}_n(\btheta_n(k)), \btheta_n(k)-\theta_0 \> \\
  &=& \< \grad R(\btheta_n(k)), \btheta_n(k)-\theta_0 \> + \< \grad\what{R}_n(\btheta_n(k)) - \grad R(\btheta_n(k)), \btheta_n(k) - \theta_0\> \\
  &\ge& \gradlip \|\btheta_n(k) - \btheta_0\|_2^2 - \eps\|\btheta_n(k) - \btheta_0\|_2 \ge \Big( \gradlip - \frac{2\eps}{\eps_0} \Big)\|\btheta_n(k) - \btheta_0\|_2^2.        
\end{eqnarray*}
The last inequality uses the fact that $\btheta_n(k) \notin \Ball_2^d(\btheta_0,\eps_0/2)$. Note that in Lemma \ref{lemma:empriskClassification} we have chosen $\eps\le \gradlip\eps_0/4$ to guarantee
that $\gradlip- 2\eps/\eps_0\ge \gradlip/2$. Plugging this back into (\ref{eqn:gd_algebra}), we get
\begin{eqnarray*}
  && \|\btheta_n(k+1) - \btheta_0\|_2^2 \\
  &\le& \|\btheta_n(k) - \btheta_0\|_2^2 -2h \cdot \frac{\gradlip}{2} \|\btheta_n(k) - \btheta_0\|_2^2 + h^2(2\gradub)^2 \\
  &=& (1-h\gradlip) \|\btheta_n(k) - \btheta_0\|_2^2 + 4h^2\gradub^2.
\end{eqnarray*}
Note that the second line uses Lemma \ref{lemma:empriskClassification}.$(b)$ to upper bound $\|\grad\what{R}_n(\btheta_n(k))\|_2$ by $2\gradub$. 

Next, choosing $h\le h_{\max, 1} \defeq \gradlip\eps_0^2/(8\gradub^2)$, we have
\begin{equation}
  4h^2\gradub^2 \le \frac{1}{2}h\gradlip\eps_0^2 \le \frac{1}{2}h\gradlip \|\btheta_n(k) - \btheta_0\|_2^2,
\end{equation}
which gives us
\begin{equation}
  \|\btheta_n(k+1) - \btheta_0\|_2^2 \le (1-h\gradlip)\|\btheta_n(k) - \btheta_0\|_2^2 + \frac{1}{2}h\gradlip\|\btheta_n(k) - \btheta_0\|_2^2 = \Big( 1-\frac{1}{2}h\gradlip \Big)\|\btheta_n(k) - \btheta_0\|_2^2.
\end{equation}

Finally, we are going to convert the above result to an exponential convergence of the optimization error $\|\btheta_n(k)-\btheta_n^*\|_2^2$. Define $r_1=1- h\gradlip/2<1$. We have the following chain of inequalities
\begin{eqnarray*}
  && \|\btheta_n(k) - \btheta_n^*\|_2 \le \|\btheta_n(k) - \btheta_0\|_2 + \|\btheta_0 - \btheta_n^*\|_2 \le \|\btheta_n(k) - \btheta_0\|_2 + \frac{1}{2}\eps_0\\
   &\le& 2 \|\btheta_n(k) - \btheta_0\|_2  \le 2r_1^{k/2} \|\btheta_n(0) - \btheta_0\|_2 \le 2 r_1^{k/2} \Big( \|\btheta_n(0) - \btheta_n^*\|_2 + \frac{1}{2}\eps_0 \Big) \\
  &\le& 4r_1^{k/2} \|\btheta_n(0) - \btheta_n^*\|_2.
\end{eqnarray*}
The last inequality is since $\|\btheta_n(0) - \btheta_n^*\|_2 \ge \eps_0/2$.
Consequently, we have
\begin{equation}\label{eqn:gdnotinball}
  \|\btheta_n(k) - \btheta_n^*\|_2^2 \le 16 \left(1 - \frac{1}{2} \gradlip h\right)^k \cdot \|\btheta_n(0) - \btheta_n^*\|_2^2.
\end{equation}
for any $\btheta_n(k)$ such that $\|\btheta_n(k) - \btheta_0\|_2 \geq \eps_0/2$.

\noindent{\bf Step 3. } Establish an exponential convergence inside $\Ball_2^d(\btheta_0, \eps_0)$. 


As shown in Lemma \ref{lemma:empriskClassification}.$(b)$, we have
\[
\inf_{\btheta\in\Ball^d(\btheta_0,\eps_0)}\lambda_{\min} (\nabla^2 \what R_n(\btheta)) \geq \frac{1}{2} \heslb, ~~ \sup_{\btheta\in\Ball^d(\btheta_0,\eps_0)} \lambda_{\max} (\nabla^2 \what R_n(\btheta)) \leq 2\heslb.
\]
Consequently, $\what R_n(\btheta)$ is $1/(2\heslb)$-strongly convex in $\Ball^d(\btheta_0, \eps_0)$. According to standard convex optimization results, if we start from a point inside $\Ball^d(\btheta_0, \epsilon_0)$, and take $h \le h_{\max,2} \defeq 1/(2\hesub)$, we have
\[
\what R_n(\btheta_n(k)) - \what R_n(\btheta_n^*) \leq  \Big(1 - \frac{1}{2} \heslb h\Big)^k \cdot \left(\what R_n(\btheta_n(0)) - \what R_n(\btheta_n^*)\right).
\]
Strongly convexity ensures that optimization error of iteration points is bounded by
\begin{equation}\label{eqn:gdinball}
\Vert \btheta_n(k) - \btheta_n^*\Vert_2^2 \leq \frac{4\hesub}{\heslb} \Big(1 - \frac{1}{2} \heslb h \Big)^{k} \cdot \Vert \btheta_n(0) - \btheta_n^* \Vert_2^2.
\end{equation}

\noindent{\bf Step 4. } Concatenate the two exponential convergences. 

Now we have the exponential convergence of gradient descent as $\btheta \in \Ball^d(\btheta_0, 2\rad/3) \setminus \Ball^d(\btheta_0, \eps_0/2)$ given by equation (\ref{eqn:gdnotinball}), and exponential convergence in $\Ball^d(\btheta_0, \eps_0)$ given by equation (\ref{eqn:gdinball}). Concatenating the two results, we get that for any initialization $\btheta_n(0)$, running gradient descent gives
\begin{equation}
  \|\btheta_n(k) - \btheta_n^*\|_2^2 \le \frac{64\hesub}{\heslb} \cdot s^k \|\btheta_n(0) - \btheta_n^*\|_2^2,
\end{equation}
where $s = \max\{1-h\gradlip/2, 1- h \heslb/2\}$, and the step size $h$ satisfies
\begin{equation}
  h \le h_{\max} = \min\{h_{\max,1}, h_{\max,2}\} = \min\left\{ \frac{\gradlip\eps_0^2}{8\gradub^2}, \frac{1}{2\hesub} \right\}.
\end{equation}

\end{proof}

\subsection{Proof of Theorem \ref{thm:SparseClass}: Very high-dimensional regime}

In this section we prove Theorem \ref{thm:SparseClass}. Similar to the high-dimensional regime, we proceed by first applying uniform convergence results in Theorem \ref{thm:uniformconvergence2} and then studying the regularized empirical risk more carefully. 

To fix notations, let $L(\btheta)=R(\btheta)+\lambda_n\|\btheta\|_1$ be the regularized population risk and $L_n(\btheta)=\what{R}_n(\btheta)+\lambda_n\|\btheta\|_1$ be the regularized empirical risk. Let $\partial L_n(\btheta)$ be the set of subgradient of $L_n$ at $\btheta$:
\begin{equation}
	\partial L_n(\btheta) =\left\{ \grad \what{R}_n(\btheta) + \lambda_n\bv : \bv\in \partial\Vert \btheta\Vert_1 \right\}.
\end{equation}
The optimality condition says that $\btheta$ is a stationary point of $L_n$ if and only if $\bzero \in\partial L_n(\btheta)$.

We decompose the proof into four lemmas. First, in Lemma \ref{lem:assSparseClass} we verify the assumptions in Theorem \ref{thm:uniformconvergence2} for the very high-dimensional binary classification model. Then, in Lemma \ref{lem:coneSparseClass} we argue that there cannot be any stationary points outside the region $\Ball_2^d(\btheta_0, r_s) \cap \mathbb C$, where $r_s$ is the statistical radius with $r_s = \Cs \sqrt{(M^2 s_0 \log d)/n + s_0 \lambda_n^2}$ and $\mathbb C$ is a cone with $\mathbb C = \{\btheta_0+\Delta: \Vert \Delta_{S_0^c}\Vert_1 \leq 3 \Vert \Delta_{S_0} \Vert_1 \}$. Next, in Lemma \ref{lem:supportSparseClass}, we argue that all the stationary points should have support size less to equal to $\Csp s_0 \log d$. Finally, in Lemma \ref{lem:restricted_hessian_sparseclass}, uniform convergence of restricted Hessian implies that there cannot be two stationary points in $
\Ball_2^d(\btheta_0, r_s) \cap \mathbb C$. 

Since we assumed $n\leq d^{100}$, sometimes we will implicitly use the bound $\log (dn) \leq 101 \log (d)$ in the proof. 

\subsubsection{Technical lemmas}

We provide a couple of technical lemmas to characterize the properties of the regularized empirical risk $L_n(\btheta)$.

\begin{lemma}\label{lem:assSparseClass}
For the very high-dimensional binary classification problem, under Assumptions \ref{ass:derivdecay} and \ref{ass:continuous}, there exist constants $T_0$ and $L_0$ depending on $(\tau^2,L_\sigma, \derivdecay, \rad)$ such that Assumptions \ref{ass:BoundedGradient} and \ref{ass:GGL} are satisfied with the parameters $T_* = T_0 \cdot M $ and $L_* = L_0 \cdot M$. 
\end{lemma}

\begin{proof}
We give a bound for $T_*$ in part $(a)$, and give a bound for $L_*$ in part $(b)$. 

\noindent{ \bf Part $(a)$.} The gradient of the loss is
\begin{equation}
	\grad_\btheta \ell(\btheta;\bz) = 2(\sigma(\<\btheta,\bx\>) - y)\sigma'(\<\btheta,\bx\>)\bx.
\end{equation}
Assumption \ref{ass:classification} guarantees that $\vert 2(\sigma(\<\btheta,\bx\>-y)\sigma'(\btheta,\bx)\vert \le 2L_\sigma$, and Assumption \ref{ass:continuous} guarantees that $\Vert \bx \Vert_\infty \leq M \tau$. So we have $\|\grad_\btheta\ell(\btheta;\bz)\|_\infty\le 2L_\sigma M\tau$. Assumption \ref{ass:BoundedGradient} is satisfied with parameter $T_*=2L_\sigma M \tau $.

\noindent{\bf Part $(b)$.}
%
We have
\begin{equation}
	\< \grad_\btheta\ell(\btheta;\bz), \btheta-\btheta_0 \> = 2(\sigma(\<\btheta,\bx\>)-y)\sigma'(\<\btheta,\bx\>)\<\bx,\btheta-\btheta_0\>.
\end{equation}
We take $t = \langle \btheta - \btheta_0, \bpsi(\bz)\rangle$, $\bpsi(\bz)=\bx$ and $g(t;\bz) = 2[(\sigma(t+t_0)-y) \sigma'(t+t_0) t]$  where $t_0 = \langle \btheta_0, \bx \rangle$. We have
\begin{equation}
\begin{aligned}
\vert g'(t;\bz) \vert =& \vert 2 \sigma'(t+t_0)^2 t  + 2(\sigma(t+t_0) -y) \sigma''(t+t_0) t +  2(\sigma(t+t_0) -y) \sigma'(t+t_0)\vert \\
\leq & \vert 2 \sigma'(t + t_0) \left(\sigma'(t+t_0)(t+t_0)\right) \vert + \vert 2 \sigma'(t + t_0)^2 t_0 \vert   \\
&+\vert 2(\sigma(t+t_0) - y) \left( \sigma''(t+t_0) (t+t_0) \right) \vert+\vert 2(\sigma(t+t_0) - y) \sigma''(t+t_0) t_0 \vert  + 2\actlip\\
\leq & 2 \actlip \derivdecay + 2 \actlip^2 M\tau \rad + 2 \derivdecay + 2 \actlip M \tau \rad + 2 \actlip \\
=&2 (\actlip \derivdecay +  \derivdecay + \actlip) +2M \cdot ( \actlip^2 + \actlip)( \tau \rad).
\end{aligned}
\end{equation}
Hence $g(t;\bz)$ is at most $2 (\actlip \derivdecay +  \derivdecay + \actlip) +2M \cdot ( \actlip^2 + \actlip)\tau \rad$-Lipschitz in its first argument, also satisfies $g(0;\bz) = 0$. So Assumption \ref{ass:GGL} is satisfied with $L_* = 2 [(\actlip \derivdecay +  \derivdecay + \actlip) + ( \actlip^2 + \actlip) ( \tau \rad)]M$ since we assumed $M \geq 1$. 

\end{proof}

From now on, we will not explicitly take account of the dependence on $(\sigma(\cdot), \actlip, \derivdecay, \tau^2, \rad, \covlb, \delta)$. We will write explicit dependence on $s_0$, $n$, $d$, and $M$. 

\begin{lemma}\label{lem:coneSparseClass}
Let $S_0={\rm supp}(\btheta_0)$ with $s_0 = \vert S_0\vert$, and define $\C=\{\btheta_0+\Delta:\|\Delta_{S_0^c}\|_1 \le 3\|\Delta_{S_0}\|_1\}\subset \R^d$. For any positive constant $\delta$, there exists constants $C_0$, $C_1$ depending on $(\sigma(\cdot),\actlip, \derivdecay, \tau^2, \rad, \covlb, \delta)$ such that letting $\lambda_n\ge C_1 M\sqrt{(\log d)/n}$, with probability at least $1-\delta$ the following two events happen:
\begin{enumerate}[label=$(\alph*)$]
\item $L_n$ has no stationary point in $\Ball_2(\rad) \cap \C^c$:
\begin{equation}
\< \bz(\btheta), \btheta-\btheta_0 \> > 0 ,~\forall \btheta\in\Ball_2(\rad)\cap \C^c,~\forall \bz(\btheta)\in\partial L_n(\btheta).
\end{equation}
\item $L_n$ has no stationary point in $\C\setminus \Ball_2(\btheta_0, r_s)$, where $r_s = C_0\sqrt{(M^2 s_0\log d)/n + s_0\lambda_n^2}$:
\begin{equation}
\< \bz(\btheta), \btheta-\btheta_0 \> > 0,~\forall \btheta\in\Ball_2(\rad)\cap \C \setminus \Ball_2( \btheta_0,r_s),~\forall \bz(\btheta)\in\partial L_n(\btheta).
\end{equation}
\end{enumerate}
\end{lemma}
\begin{proof}
	For any $\bz(\btheta) \in \partial L_n(\btheta)$, write $\bz(\btheta)=\grad \what{R}_n(\btheta) + \lambda_n \bv(\btheta)$ where $\bv(\btheta) \in \partial \Vert \btheta \Vert_1$. We have
	\begin{eqnarray}
		&& \< \bz(\btheta), \btheta-\btheta_0 \> \nonumber \\
		&=& \< \grad R(\btheta),\btheta-\btheta_0 \> + \lambda_n\<\bv(\btheta),\btheta-\btheta_0\> + \< \grad\what{R}_n(\btheta) - \grad R(\btheta), \btheta - \btheta_0 \> \nonumber \\
		&\ge& \gradlip \|\btheta-\btheta_0\|_2^2 + \lambda_n\<\bv(\btheta),\btheta-\btheta_0\> - |\< \grad\what{R}_n(\btheta) - \grad R(\btheta), \btheta - \btheta_0 \>|, \label{eqn:sparsegradlb}
	\end{eqnarray}
	where the final inequality follows from Lemma \ref{lemma:popriskClassification}.$(c)$ in which the constant $\gradlip>0$ depending on $(\sigma(\cdot ), \rad, \actlip, \tau^2,\covlb)$ but independent of $n$ and $d$.

	Our aim is to upper bound the third term above using our uniform convergence results. By Theorem \ref{thm:uniformconvergence2} and Lemma \ref{lem:assSparseClass}, there exists a constant $C$ depending on $(\sigma(\cdot ),\tau^2, \actlip, \derivdecay, \rad, \covlb, \delta)$ such that the event $E_1$ happens with probability at least $1- \delta$, where
	\begin{eqnarray}\label{eq:event_E1}
		&& E_1 = \bigg\{ \sup_{\btheta\in\Ball_2(\rad) \setminus \{\bzero\}} \frac{|\< \grad \what{R}_n(\btheta) - \grad R(\btheta), \btheta-\btheta_0 \>|}{\|\btheta-\btheta_0\|_1} \le CM \sqrt{\frac{\log d}{n}} \bigg\} 
	\end{eqnarray}
	

	Assume $E_1$ happens, we show claim $(a)$. By definition of $S_0$, we have $(\btheta_0)_{S_0^c}=0$, and so, letting $\Delta = \btheta - \btheta_0$, 
	\begin{equation}
		\< \bv(\btheta),\btheta-\btheta_0 \> = \< \bv(\btheta)_{S_0^c}, \btheta_{S_0^c} \> + \< \bv(\btheta)_{S_0}, \Delta_{S_0} \> \ge \|\Delta_{S_0^c}\|_1 - \|\Delta_{S_0}\|_1.
	\end{equation}
	Plugging this into (\ref{eqn:sparsegradlb}) gives
	\begin{equation}
		\< \bz(\btheta),\btheta-\btheta_0 \> \ge \gradlip \|\btheta-\btheta_0\|_2^2 + \lambda_n(\|\Delta_{S_0^c}\|_1 - \|\Delta_{S_0}\|_1) - CM\sqrt{\frac{\log d}{n}}(\|\Delta_{S_0^c}\|_1 + \|\Delta_{S_0}\|_1).
	\end{equation}
	Taking $C_1 = 2C$, and letting $\lambda_n\ge 2 C M\sqrt{(\log d)/n}$ and noticing that $\|\Delta_{S_0^c}\|_1 > 3\|\Delta_{S_0}\|_1$ when $\btheta\in\C^c$, we have
	\begin{equation}
		\< \bz(\btheta),\btheta-\btheta_0 \> \ge \gradlip \|\btheta-\btheta_0\|_2^2 + CM\sqrt{\frac{\log d}{n}} \Big( \|\Delta_{S_0^c}\|_1 - 3\|\Delta_{S_0}\|_1 \Big) > 0.
	\end{equation}

	Finally, we prove claim $(b)$. As $\btheta\in\C$, we have $\|\btheta-\btheta_0\|_1\le 4 \sqrt{s_0}\|\btheta-\btheta_0\|_2$. Plugging this into (\ref{eqn:sparsegradlb}) gives
	\begin{eqnarray*}
		&& \< \bz(\btheta),\btheta-\btheta_0 \> \\
		&\ge& \gradlip \|\btheta-\btheta_0\|_2^2 + \lambda_n\<\bv(\btheta),\btheta-\btheta_0\> - CM\sqrt{\frac{\log d}{n}}\|\btheta-\btheta_0\|_1 \\
		&\ge& \gradlip \|\btheta-\btheta_0\|_2^2 - \left( \lambda_n + CM\sqrt{\frac{\log d}{n}} \right) \|\btheta-\btheta_0\|_1 \\
		&\ge& \gradlip \|\btheta-\btheta_0\|_2^2 - \left( 4\sqrt{s_0}\lambda_n + CM\sqrt{\frac{s_0\log d}{n}} \right) \|\btheta-\btheta_0\|_2 \\
		&=& \left( \gradlip\|\btheta-\btheta_0\|_2 - \left( 4\sqrt{s_0}\lambda_n + CM\sqrt{\frac{s_0\log d}{n}} \right) \right) \|\btheta-\btheta_0\|_2.
	\end{eqnarray*}
	Consequently, when $\|\btheta-\btheta_0\|_2>(1/\gradlip)\cdot (4\sqrt{s_0}\lambda_n+CM\sqrt{(s_0\log d)/n})$, we get $\<\bz(\btheta),\btheta-\btheta_0\>>0$. Taking $C_0=(2/\gradlip)\cdot(C \vee 4)$ gives claim $(b)$. 
\end{proof}


\begin{lemma}\label{lem:supportSparseClass}
	For any positive constants $C_0$ and $\delta$, there exist positive constants $C_1$, $C_2$, and $C_3$ depending on $C_0$ and $(\sigma(\cdot),\tau^2, \rad, \actlip, \derivdecay, \covlb, \delta)$, such that as $n\ge C_1 s_0\log d$ and $\lambda_n\ge C_2 M \sqrt{(\log d)/n}$, then with probability at least $1-\delta$, any stationary point $\hat \btheta$ of $L_n$ in $\C\cap \Ball_2(\btheta_0,r_s)$ has support size $\vert S(\hat \btheta) \vert \le C_3 \, s_0\log d$, where $r_s = C_0\sqrt{(M^2 s_0\log d)/n+s_0\lambda_n^2}$.
\end{lemma}
\begin{proof}

We decompose the proof into the following steps.

\noindent{\bf Step 1. }
Let $\hat \btheta \in \Ball_2(\btheta_0, r_s) \cap \mathbb C$ be a stationary point of the optimization problem (\ref{eq:SparseClass_1}). The KKT condition for the stationary point gives
\[
\begin{aligned}
\nabla \what R_n(\hat \btheta) + \lambda_n \bv(\hat \btheta) = 0,
\end{aligned}
\]
where $\bv(\hat \btheta) \in \partial \Vert \hat \btheta \Vert_1$. Denote $S(\hat \btheta) = \text{supp}(\hat \btheta)$. Thus, we have
\begin{equation}\label{eq:kkt}
(\nabla \what R_n(\hat \btheta))_j = \pm \lambda_n, \quad \forall j \in S(\hat \btheta).
\end{equation}

Now define the event
\begin{equation}\label{eq:equ1}
  E_C = \left\{ \|\grad\what{R}_n(\btheta_0)\|_\infty \le C\sqrt{\frac{\log d}{n}}\right\}.
\end{equation}
As verified in Equation (\ref{eq:gradsubgaussian1}), there exists a constant $c_0$ depending on $\actlip$, such that $\nabla \what{R}_n(\btheta_0) = \frac{1}{n}\sum_{i=1}^{n}\grad\ell(\btheta_0;\bZ_i)$ is the average of $n$ i.i.d. mean zero $c_0 \tau^2$-sub-Gaussian random vectors. Thus we have
\begin{equation}
  \P\left( \|\grad\what{R}_n(\btheta_0)\|_\infty > t \right) \le d\sup_{j\in[d]} \P\left( \left| \frac{1}{n}\sum_{i=1}^{n}[\grad\ell(\btheta_0;\bZ_i)]_j \right| > t  \right) \le \exp\Big(\log (2d) - \frac{nt^2}{2c_0 \tau^2}\Big).
\end{equation}
Taking $t=\tau\sqrt{\frac{2c_0 (\log d + \log(6/\delta))}{n}}$ and correspondingly $C=2\tau\sqrt{c_0\log\frac{6}{\delta}}$ guarantees that $\P(E_C)\le\delta/3$.

Condition on $E_C$, taking $C_2 = 2 C$ and $\lambda_n \geq C_2 \sqrt{(\log d)/n}$, this gives
\[
\lambda_n/2 \geq C\sqrt{\frac{\log d}{n}}\geq \Vert \nabla \what R_n(\btheta_0)\Vert_\infty.
\]
Combining with equation (\ref{eq:kkt}), we have
\[
\lambda_n/2 \leq \left \vert \Big(  \nabla \what R_n(\hat \btheta) - \nabla \what R_n(\btheta_0) \Big)_j \right \vert, \quad \forall j \in S(\hat \btheta). 
\]
Squaring and summing over $j \in S(\hat \btheta)$, we have
\begin{equation}\label{eq:provesparsity1}
\begin{aligned}
\lambda_n^2 \vert S(\hat \btheta) \vert \leq& 4 \left \Vert \Big( \nabla \what R_n(\hat \btheta) - \nabla \what R_n(\btheta_0) \Big)_{S(\hat \btheta)} \right \Vert_2^2\\
\leq&  4 \left \Vert \Big( \frac{1}{n} \sum_{i=1}^n (\alpha_i(\hat \btheta) - \alpha_i(\btheta_0) )\bX_i \Big)_{S(\hat \btheta)} \right \Vert_2^2\\
= & 4 \left \Vert \Big( \frac{1}{n} \sum_{i=1}^n \beta_i( \btheta_i) \bX_i \bX_i^\sT (\hat \btheta - \btheta_0) \Big)_{S(\hat \btheta)} \right \Vert_2^2\\
=& \frac{4}{n^2}\Vert P_{S(\hat \btheta)} \mbX^\sT D \mbX (\hat \btheta - \btheta_0) \Vert_2^2\\
\end{aligned}
\end{equation}
Here, $\mbX = (\bX_1, \ldots, \bX_n)^\sT \in \R^{n \times d}$,  $\btheta_i$ are located on the line between $\hat \btheta$ and $\btheta_0$ obtained by intermediate value theorem, $D = \diag(\beta_1(\btheta_1),\ldots, \beta_n(\btheta_n)) \in \R^{n \times n}$ where $\alpha_i$ and $\beta_i$ are defined as
\begin{equation}
\begin{aligned}
\alpha_i(\btheta) =& -2(Y_i - \sigma(\btheta^\sT \bX_i))\sigma'(\btheta^\sT \bX_i),\\
\beta_i(\btheta) =& 2\Big( \sigma'(\btheta^{\sT}\bX_i)^2 + (\sigma(\btheta^{\sT}\bX_i)-Y_i)\sigma''(\btheta^{\sT}\bX_i) \Big),
\end{aligned}
\end{equation}
and $P_{S} \in \R^{d \times d}$ is a projection matrix onto the vector space with vectors supported on index set $S$.

\noindent{\bf Step 2. }
Now we are going to upper bound the right hand side for any stationary point $\hat \btheta \in \Ball_2(\btheta_0, r_s) \cap \mathbb C$. We claim that there exists a constant $c_1$ depending on $\delta$, such that as $n \geq c_1 s_0 \log d$, we have
\begin{equation}\label{eq:equ2}
\P \left( \sup_{\btheta \in \Ball_2(\btheta_0, r_s) \cap \mathbb C} \frac{1}{n} \Vert \mbX (\btheta - \btheta_0) \Vert_2^2 \leq 3 \tau^2 r_s^2 \right) \geq 1 - \delta/3.
\end{equation}

Indeed, due to the restricted smoothness property of the sub-Gaussian random matrices (See \cite[Theorem 6]{rudelson2011reconstruction}), we have for any design $\mbX$ with independent $\tau^2$-sub-Gaussian rows, there exists a constant $c_1$ depending on $\delta$, such that with probability at least $1 - \delta$, as $n \geq c_1 s_0 \log d$, we have 
\begin{equation}
\sup_{\btheta \in \mathbb C}\frac{\frac{1}{n}\Vert \mbX (\btheta - \btheta_0) \Vert_2^2 }{ \Vert \btheta - \btheta_0 \Vert_2^2} \leq 3 \tau^2. 
\end{equation}
Therefore, with probability at least $1 - \delta$, we have 
\begin{equation}
\sup_{\btheta \in \Ball_2(\btheta_0, r_s) \cap \mathbb C} \frac{1}{n} \Vert \mbX (\btheta - \btheta_0) \Vert_2^2 \leq 3 \tau^2 \cdot \sup_{\btheta \in \Ball_2(\btheta_0, r_s) \cap \mathbb C} \Vert \btheta - \btheta_0 \Vert_2^2 \leq 3 \tau^2 r_s^2. 
\end{equation}

%
%

\noindent{\bf Step 3. }
The diagonal matrix $D$ has elements with absolute values upper bounded by $c_2 = 2(\actlip^2 + \actlip)$. As the good event in equation (\ref{eq:equ2}) happens, we have
\begin{equation}
\sup_{\Vert D\Vert_2 \leq c_2} \sup_{\btheta \in \Ball_2(\btheta_0, r_s) \cap \mathbb C}\frac{1}{n} \Vert D \cdot \mbX (\btheta - \btheta_0) \Vert_2^2 \leq 3 c_2^2 \tau^2 r_s^2. 
\end{equation}

\noindent{\bf Step 4. }
In this step, we show that any stationary point $\hat \btheta$ must have support size $\vert S(\hat \btheta) \vert \leq n$. 

Note that the subgradient of the objective function gives
\[
\partial L_n(\btheta) = \Big\{ \frac{1}{n}\sum_{i=1}^n \alpha_i(\btheta) \bX_i + \lambda_n \bv(\btheta): \bv(\btheta) \in \partial\|\btheta\|_1 \Big\}.
\]
To show that the support size of any stationary point is smaller or equal to $n$, we would like to show that: for any $\btheta$ such that $s=\vert S(\btheta)\vert \geq n+1$, there exists a vector $\bw(\btheta) \in \R^d$ such that $0\notin \langle \partial L_n(\btheta), \bw(\btheta) \rangle$. 

We claim that a sufficient condition for the existence of such a $\bw(\btheta)$ is that the $n+1$ vectors $\{\bX_{1,S(\btheta)},\dots,\bX_{n,S(\btheta)},\bz\}$ are linearly independent for all $\bz\in\{\pm 1\}^{s}$. Indeed, the matrix
\begin{equation}
  \mbX_{:,S(\btheta)} = [\bX_{1,S(\btheta)}, \dots, \bX_{n,S(\btheta)}]^{\sT} \in \R^{n\times s}
\end{equation}
has a nonempty null space since $s>n$. Consequently, there exists some nonzero vector $\bu\in\R^s$ such that $\mbX_{:,S(\btheta)}\bu=0$. But for any $\bv(\btheta)\in\partial\|\btheta\|_1$, we have $\bv(\btheta)_{S(\btheta)}\in \{\pm 1\}^s$. Therefore, our condition guarantees that $\langle\bv(\btheta)_{S(\btheta)}, \bu\rangle \neq 0$. Let $\bw(\btheta)\in\R^d\setminus \{\bzero\}$ be a vector such that $\bw(\btheta)_{S(\btheta)} = \bu$ and $\bw(\btheta)_{S(\btheta)^c} = \bzero$. We get that
\begin{equation}
  \Big\<  \frac{1}{n}\sum_{i=1}^{n} \alpha_i(\btheta)\bX_i+ \lambda_n\bv(\btheta),  \bw(\btheta) \Big\> = \lambda_n\<\bv(\btheta), \bw(\btheta)\> = \lambda_n \langle \bv(\btheta)_{S(\btheta)}, \bu\rangle \neq 0
\end{equation}
for any sign vector $\bv(\theta)$. This shows that $0\notin \<\partial L_n(\btheta), \bw(\btheta)\>$.

Finally, let us verify this sufficient condition. As $\bX_i$ are continuously distributed, for any $S\subseteq [d]$ with cardinality  $s = \vert S \vert \ge n+1$ and $\bz\in\{\pm 1\}^{s}$, the probability that $\{\bX_{1,S},\dots,\bX_{n,S},\bz\}$ are linearly dependent is zero. Taking a union bound with all $S \subset [d]$ and $\bz \in \{ \pm 1\}^s$ (finitely many) gives that our sufficient condition is satisfied with probability one.

\noindent{\bf Step 5. }
We show that with high probability, the maximum of $\Vert\frac{1}{n} \mbX P_{S} \mbX^\sT \Vert_{\op}$ for all $S:|S|\le n$ is upper bounded by $O(\log d)$: there exists some constant $c_3$ depending on $(\tau^2, \delta)$ such that 
\begin{equation}\label{eq:equ3}
\P \left( \sup_{\vert S \vert \leq n} \lambda_{\max}\Big(\frac{1}{n} \mbX P_{S} \mbX^\sT\Big) \leq c_3 \log d \right) \geq 1-\delta/3.
\end{equation}

Indeed, note that $\frac{1}{n}\mbX P_S\mbX^\sT$ is increasing in $S$: whenever $S\subseteq T$ we have $\mbX P_S\mbX^\sT\preceq \mbX P_T\mbX^\sT$. This gives us
\begin{eqnarray*}
&& \P\Big( \sup_{\vert S \vert \leq n} \lambda_{\max}\Big(\frac{1}{n} \mbX P_{S} \mbX^\sT \Big) \geq \eps\Big)\\
&= & \P\Big(\sup_{\vert S \vert = n} \lambda_{\max} \Big(\frac{1}{n} \mbX P_{S} \mbX^\sT\Big) \geq \eps\Big)\\
&\leq & {d \choose n} \sup_{\vert S\vert =n} \P\Big(\lambda_{\max} \Big(\frac{1}{n} \mbX P_{S} \mbX^\sT\Big) \geq \eps\Big).
\end{eqnarray*}
Now, fixing any $|S|=n$, we are going to use a covering number argument to bound $\lambda_{\max}(\frac{1}{n}\mbX P_S\mbX^\sT)$: let $V$ be a $(1/4)$-cover of $\mathbb{S}^{n-1}=\{\bv\in\R^n:\|\bv\|_2=1\}$. We already know from Lemma \ref{lem:opnorm} that $\lambda_{\max}(A)\le 2\sup_{\bv\in V}\langle \bv,A\bv\rangle$ for any $A\in\R^{n\times n}$, so that
\begin{eqnarray*}
&& \P\Big(\lambda_{\max} \Big(\frac{1}{n} \mbX P_{S} \mbX^\sT\Big) \geq \eps\Big) \le \P\Big( \sup_{\bv\in V} \Big\langle \bv,\frac{1}{n}\mbX P_S\mbX^\sT\bv\Big\rangle \ge \eps/2 \Big) \\
&\le& N(1/4,\mathbb{S}^{n-1}) \cdot \sup_{\bv\in V} \P\Big( \Big\langle \bv,\frac{1}{n}\mbX P_S\mbX^\sT\bv\Big\rangle \ge \eps/2 \Big).
\end{eqnarray*}
Further, for any fixed $\bv\in\mathbb{S}^{n-1}$, we have
\begin{equation*}
\Big\< \bv, \frac{1}{n}\mbX P_S\mbX^\sT\bv \Big\> = \frac{1}{n}\sum_{i=1}^{n} \< \bv_S,\bX_{j,S} \>^2 \defeq \frac{1}{n}\sum_{i=1}^{n} W_j^2.
\end{equation*}
Since $\bX_j$ are i.i.d. $\tau^2$-sub-Gaussian, the random variables $W_j=\<\bv_S,\bX_{j,S}\>$ are also i.i.d. $\tau^2$-sub-Gaussian. Applying Lemma \ref{lem:subgaussian}.$(c)$, we get a Chernoff bound
\begin{eqnarray*}
\P\Big(\frac{1}{n} \sum_{j=1}^{n} W_j^2 \ge \eps/2 \Big)  \le \exp\Big(-\frac{n\eps}{8\tau^2}\Big) \E\Big[\exp\Big(\frac{\sum_{j=1}^{n} W_j^2}{4\tau^2}\Big)\Big] \le \exp\Big(-\frac{n\eps}{8\tau^2} + \frac{n}{2} \Big). 
\end{eqnarray*}
Putting together all the pieces, we get
\begin{eqnarray*}
&& \P\Big( \sup_{|S|\le n} \lambda_{\max}\Big( \frac{1}{n}\mbX P_S\mbX^\sT \Big) \ge \eps \Big) \\
&\le& {d \choose n} N(1/4,\mathbb{S}^{n-1}) \cdot \exp\Big( -\frac{n\eps}{8\tau^2}+\frac{n}{2} \Big) \\
&\le& \exp\Big( n\log d + n\log 9 - \frac{n\eps}{8\tau^2} + \frac{n}{2} \Big).
\end{eqnarray*}
To let the above probability be less or equal to $\delta$, it suffices to take $\eps = (22 + 8\log d + 8\log(1/\delta)/n)\tau^2$. 

\noindent{\bf Step 6. }
The good events in equations (\ref{eq:equ1}), (\ref{eq:equ2}), and (\ref{eq:equ3}) will simultaneously happen with probability at least $1-\delta$. When this happens, by Equation (\ref{eq:provesparsity1}) we have
\[
\begin{aligned}
\lambda_n^2 \vert S(\hat \btheta ) \vert  \leq & \sup_{\Vert D \Vert_2 \leq c_2} \sup_{\btheta \in \Ball_2(\btheta_0, r_s ) \cap \mathbb C, \vert S(\btheta) \vert \leq n}  \frac{4}{n^2}\Vert P_{S(\btheta)} \mbX^\sT D \mbX ( \btheta - \btheta_0) \Vert_2^2\\
\leq & 4 \sup_{\vert S \vert \leq n} \lambda_{\max}\Big(\frac{1}{n} \mbX P_{S} \mbX^\sT \Big) \cdot \sup_{\Vert D \Vert_2 \leq c_2} \sup_{\btheta \in \Ball_2(\btheta_0, r_s) \cap \mathbb C}\frac{1}{n} \Vert D \mbX (\btheta - \btheta_0) \Vert_2^2 \\
\leq & 4 c_3 \log d \cdot 3 c_2^2 \tau^2 r_s^2\\
\leq & 12 c_2^2 \tau^2 c_3 C_0^2 \Big( \frac{M^2 s_0\log d}{n} + s_0 \lambda_n^2\Big) \log d.
\end{aligned}
\]

Note that we already choose $C_2$. Taking $\lambda_n \geq C_2 M \sqrt{(\log d)/n}$ this gives us
\[
\vert S(\hat \btheta) \vert \leq 12 c_2^2 \tau^2 c_3 C_0^2 \left( 1/C_2^2 + 1\right)  s_0 \log d = C_3 s_0 \log d.
\]

\end{proof}

\begin{lemma}\label{lem:restricted_hessian_sparseclass}
For any positive constants $C_0$ and $\delta$, letting $r_0 = C_0 s_0 \log d$, there exists a positive constant $C_1$ depending on $(C_0,\sigma(\cdot),\tau^2,\rad, \actlip, \covlb,\delta)$, such that as we define the event $E_{\rm RH}(\eps)$ to be
\begin{equation}
E_{\rm RH}(\eps) = \left\{  \sup_{\btheta\in\Ball_2(\btheta_0,\rad)\cap \Ball_0(r_0)}\sup_{\bv\in\Ball_2(1) \cap \Ball_0(r_0)} \left\< \bv, \big(\grad^2 \what{R}_n(\btheta) - \nabla^2 R(\btheta)\big)\bv\right\> \leq  \eps \right\},
\end{equation}
we have the following claims:
\begin{enumerate}[label=$(\alph*)$, leftmargin=0.5cm]
\item Let $\eps_0$ and $\heslb$ be the constants defined in Lemma \ref{lemma:popriskClassification} depending on $(\sigma(\cdot),\tau^2,\rad, \actlip, \covlb, \delta)$ but independent of $(n, d)$. As the event $E_{\rm RH}(\heslb/2)$ happens, the regularized empirical risk $L_n(\btheta)$ cannot have two stationary points in the region $\Ball_2(\btheta_0,\eps_0)\cap \Ball_0(r_0/2)$.
\item Consequently, there exists a constant $C_1$ depending on $(C_0,\sigma(\cdot),\tau^2,\rad, \actlip, \covlb,\delta)$, as $n \geq C_1 s_0\log^2 d $, the regularized empirical risk $L_n(\btheta)$ cannot have two stationary points in the region $\Ball_2(\btheta_0,\eps_0)\cap \Ball_0(r_0/2)$ with probability at least $1-\delta$. 
\end{enumerate}
\end{lemma}
\begin{proof}

According to Theorem \ref{thm:uniformconvergence2}.$(b)$, there exists a constant $C'$ depending on $(C_0,\sigma(\cdot), \tau^2, \actlip, \rad, \covlb,\delta)$ such that as $n \geq C' s_0 \log^2 d / (1\wedge \heslb^2 /4)$, the event $E_{\rm RH}(\heslb/2)$ happens with probability at least $1-\delta$. Thus, part $(b)$ can be implied directly from part $(a)$. Now we prove part $(a)$.

\noindent{\bf Part $(a)$.}    
In Lemma \ref{lemma:popriskClassification}.$(b)$, we proved that there exist some constants $\eps_0>0$ and $\heslb>0$ such that the population risk $R(\btheta)$ is strongly convex in the region $\Ball_2(\btheta_0,\eps_0)$
\begin{equation}
\inf_{\btheta\in\Ball_2(\btheta_0,\eps_0)} \lambda_{\min}( \nabla^2 R(\btheta)) \ge  \heslb.
\end{equation}
Event $E_{\rm RH}(\heslb/2)$ implies the restricted strong convexity of the empirical risk
\begin{equation}
\inf_{\btheta\in\Ball_2(\btheta_0,\eps_0)\cap \Ball_0(r_0)}\inf_{\bv\in\Ball_2(1) \cap \Ball_0(r_0)} \Big\< \bv, \grad^2 \what{R}_n(\btheta) \bv\Big\> \ge \frac{\heslb}{2}.
\end{equation}

We argue that the above restricted strong convexity of the empirical risk $\what R_n$ makes the existence of two distinct sparse minima of the regularized empirical risk $L_n$ impossible. Indeed, suppose $\btheta_1,\btheta_2\in\Ball_2(\btheta_0,\eps_0)\cap\Ball_0(r_0/2)$ are two distinct local minima of $L_n(\btheta)=\what{R}_n(\btheta) + \lambda_n\|\btheta\|_1$. Define $\bu= (\btheta_2-\btheta_1)/\|\btheta_2-\btheta_1\|_2$. As $\btheta_1,\btheta_2$ are $(r_0/2)$-sparse, the vector $\bu$ is $r_0$-sparse, as well as $\btheta_1+t\bu$ for any $t\in\R$. Hence we have
\begin{eqnarray}
	\< \grad\what{R}_n(\btheta_2), \bu \> &=& \< \grad\what{R}_n(\btheta_1), \bu \> + \int_{0}^{\|\btheta_2-\btheta_1\|_2} \left \langle \bu, \grad^2\what{R}_n(\btheta_1+t\bu)\bu \right \rangle \ud t \\
&\ge& \< \grad\what{R}_n(\btheta_1), \bu \> + \frac{\heslb}{2}\|\btheta_2-\btheta_1\|_2.
\end{eqnarray}
Note that the regularization term $\lambda_n\|\btheta\|_1$ is also convex. It means that for any subgradients $\bv(\btheta_1)\in \partial \Vert \btheta_1\Vert_1, \bv(\btheta_2)\in\partial \Vert \btheta_2\Vert_1$, we have $\< \bv(\btheta_1) - \bv(\btheta_2), \btheta_1-\btheta_2 \> \ge 0$, or equivalently
\begin{align}
\lambda_n\<\bv(\btheta_1),\bu\>\ge\lambda_n\<\bv(\btheta_2),\bu\>
\end{align}
Adding the above two inequalities up gives 
\begin{equation}
	\< \nabla \what R_n(\btheta_2) + \lambda_n \bv(\btheta_2),\bu \> \ge\< \nabla \what R_n(\btheta_1) + \lambda_n \bv(\btheta_1),\bu \> + \frac{\heslb}{2}\|\btheta_2-\btheta_1\|_2,
\end{equation}
for any $\bv(\btheta_1) \in \partial \Vert \btheta_1 \Vert_1$ and $\bv(\btheta_2) \in \partial \Vert \btheta_2 \Vert_1$. This implies that we cannot have $\bzero \in \partial L_n(\btheta_1)$ and $ \bzero \in \partial L_n(\btheta_2)$ simultaneously.
\end{proof}

\subsubsection{Proof of the main theorem}
We are now in a good position to prove Theorem \ref{thm:SparseClass}.
\begin{proof}[Proof of Theorem \ref{thm:SparseClass}]
First, for any $\delta>0$, in Lemma \ref{lem:coneSparseClass}, we get $C_0$ and $C_1$ given by this lemma, and we define $\Cs = C_0$ and $C_{\lambda,1} = C_1$. Then in Lemma \ref{lem:supportSparseClass}, we choose $C_0 = \Cs$, and we get $C_1, C_2, C_3$ given by this lemma, and we define $\Csp = C_3$, $\Cl = \max\{C_{\lambda,1}, C_2\}$, $C_{n,1} = C_1$. Finally, in Lemma \ref{lem:restricted_hessian_sparseclass}, we choose $C_0 = \Csp$, and we get $C_1$ given by this lemma, and we define $\Cn = \{ C_{n,1}, C_1\}$.

As above, we defined all of our constants $\Cl, \Cs, \Cn$ necessary in Theorem \ref{thm:SparseClass}, and under the assumptions of Theorem \ref{thm:SparseClass} with these constants, all the claims in Lemmas \ref{lem:coneSparseClass}, \ref{lem:supportSparseClass}, and \ref{lem:restricted_hessian_sparseclass}.(c) happen simultaneously with probability at least $1-3\delta$. 


As all the claims happen, by Lemma \ref{lem:coneSparseClass}, as $n \geq \Cn s_0 \log d$ and $\lambda_n \geq \Cl M \sqrt{(\log d)/n}$, there will be no stationery point outside $\Ball_2^d(\btheta_0, \Cs \sqrt{(M^2 s_0 \log d)/n + s_0 \lambda_n^2})$. This proves Theorem \ref{thm:SparseClass}.$(a)$. By Lemma \ref{lem:supportSparseClass}, as $\Cs \sqrt{(M^2 s_0 \log d)/n + s_0 \lambda_n^2} \leq \eps_0$, the only possible stationery point of $L_n(\btheta)$ will be within the set $\C \cap \Ball_2(\btheta_0,\eps_0) \cap \Ball_0(\Csp s_0\log d)$. By Lemma \ref{lem:restricted_hessian_sparseclass}, as $n \geq \Cn s_0 \log^2 d$, such stationery point must be unique. This proves Theorem \ref{thm:SparseClass}.$(b)$.
\end{proof}


\section{Proof of Theorem \ref{thm:MainRobustRegression}: robust regression}

\subsection{Landscape of population risk}

\begin{lemma}\label{lemma:popriskRobustRegression}
Assume $\|\btheta_0\|_2\le \rad/3$ together with Assumption \ref{ass:Robust}. Then  we have the following:
\begin{enumerate}[label=$(\alph*)$, leftmargin=0.5cm]
\item {\bf Unique minimizer. } The population risk $R(\btheta)$ is minimized at $\btheta=\btheta_0$ and has no other stationary points. 
\item {\bf Bounds on the Hessian. } There exist an $\eps_0>0$ and some constants 
$0<\heslb<\hesub<\infty$ such that
\begin{equation}
\label{eq:PopHessian} 
\inf_{\btheta\in \Ball^d(\btheta_0,\eps_0)} \lambda_{\min}\Big( \grad^2 R(\btheta) \Big) \ge \heslb,\;\;\;\;\;\; \sup_{\btheta\in \Ball^d(\bzero,\rad)} \Big\Vert \grad^2 R(\btheta) \Big\Vert_{\op} \le \hesub.
\end{equation}
\item {\bf Bounds on the gradient. } For the same $\eps_0$ as in part (b), there exist some constants $0<\gradlb<\gradub<\infty$ and $\gradlip\in(0,\infty)$ such that
\begin{equation}
\inf_{\btheta\in \Ball^d(\bzero, \rad)\setminus \Ball^d(\btheta_0,\eps_0)} \Big\| \grad R(\btheta) \Big\|_2 \ge \gradlb,\;\;\;\;\;\;\sup_{\btheta\in \Ball^d(\bzero,\rad)} \Big\| \grad R(\btheta) \Big\|_2 \le \gradub,
\end{equation}
and for all $\btheta\in \Ball^d(\bzero,\rad)$,
\begin{equation}
\langle \btheta-\btheta_0, \grad R(\btheta)\rangle\ge \gradlip \|\btheta-\btheta_0\|_2^2.   \label{eq:GradientDirection} 
\end{equation}
\end{enumerate}
All constants $\eps_0,\heslb,\hesub,\gradlb,\gradub,\gradlip$ are functions of $(\rho(\cdot), \P_\eps, \rholip,\rad,\tau^2, \covlb)$ but do not depend on $d$ and the distribution of $\bX$.
\end{lemma}

\begin{proof}

The proof consists of five parts. Lower bounds of gradient and Hessian are a little involved,
and upper bounds are relatively easy to obtain.

\prg{Part $(a)$. No stationary points other than $\btheta_0$. } 

Part $(a)$ is a direct consequence of part $(c)$. 
%
%
%

\prg{Part $(b)$. Lower bounding the Hessian. }

We first look at the minimum eigenvalue of the Hessian at $\btheta = \btheta_0$. We have for any $\bu\in\R^d,\|\bu\|_2=1$,
\[
\begin{aligned}
\langle \bu, \nabla^2 R(\btheta_0) \bu \rangle =& \E[\psi'(\eps) \langle \bX, \bu \rangle^2] \\
=&  \E[\psi'(\eps)]\cdot \E[\langle \bX, \bu \rangle^2]\\
=& g'(0)\cdot \E[\langle \bX, \bu \rangle^2] \geq c_1 \covlb \tau^2,  
\end{aligned}
\]
where $c_1 = g'(0) >0$ by assumption. That is, we have
$\lambda_{\min}(\nabla^2 R(\btheta_0)) \geq c_1 \covlb \tau^2$. 

Then we look at the operator norm of $\nabla^2 R(\btheta) - \nabla^2 R(\btheta_0)$. We have for any $\bu\in\R^d,\|\bu\|_2=1$,
\[
\begin{aligned}
\Big \vert \langle \bu, (\nabla^2 R(\btheta) - \nabla^2 R(\btheta_0))\bu \rangle \Big \vert = & \Big \vert\E[ (\psi'(\langle \bX, \btheta_0 - \btheta \rangle + \eps) - \psi'(\eps)) \langle \bX, \bu \rangle^2]\Big \vert\\
=&\Big \vert \E[ \psi''(\xi) \langle \bX,\btheta_0 - \btheta\rangle \langle \bX, \bu \rangle^2] \Big \vert\\
\leq & \E\Big[\Big \vert \psi''(\xi) \Big \vert \cdot \Big \vert \langle \bX, \btheta_0 - \btheta \rangle \Big \vert \cdot \langle \bX, \bu \rangle^2 \Big]\\
\leq & \rholip \E\Big[\Big \vert \langle \bX, \btheta_0 - \btheta \rangle \Big \vert \cdot \langle \bX, \bu \rangle^2 \Big]\\
\leq & \rholip\Big ( \E[\langle \btheta - \btheta_0, \bX \rangle^2] \E [\langle \bX, \bu\rangle^4]\Big)^{1/2} \\
\leq & \rholip\Big( \Vert \btheta - \btheta_0 \Vert_2^2 \tau^2 \cdot C_4 \tau^4 \Big)^{1/2}\\
=& \rholip \sqrt{C_4} \cdot \Vert \btheta - \btheta_0 \Vert_2 \tau^3. 
\end{aligned}
\]

Hence, taking $\Vert \btheta - \btheta_0 \Vert_2 \leq \eps_0 \defeq (c_1 \covlb)/(2 \tau \rholip \sqrt C_4)$ guarantees that $\lambda_{\max}(\nabla^2 R(\btheta) - \nabla^2 R(\btheta_0)) \leq c_1 \covlb\tau^2/2$. Consequently, for all $\btheta \in \Ball^d(\btheta_0, \eps_0)$, we have 
\[
\lambda_{\min} (\nabla^2 R(\btheta)) \geq \heslb = \frac{c_1 \covlb}{2}\tau^2. 
\]

\prg{Upper bounding the Hessian. }
For any $\btheta \in \R^d$, we have
\[
\begin{aligned}
\Vert \nabla^2 R(\btheta) \Vert_{\op} =& \sup_{\Vert \bu \Vert_2 = 1} \E [\psi'(\langle \bX, \btheta_0 -\btheta \rangle + \eps) \langle \bu, \bX \rangle^2]  \\
\leq & \sup_{\Vert \bu \Vert_2 = 1}  \E \Big[ \big\vert\psi'(\langle \bX, \btheta_0 -\btheta \rangle + \eps)  \big\vert  \cdot \langle \bu, \bX \rangle^2\Big] \\
\leq & \rholip \tau^2. 
\end{aligned}
\]

The last inequality follows from Lemma \ref{lem:subgaussian}.$(a)$. Hence $\hesub = \rholip\tau^2$ is a global upper bound for Hessian. 

\prg{Part $(c)$. Lower bounding the gradient. }

Fix $\btheta\in \Ball^d(\bzero,\rad)$, then
$\|\btheta\|_2\le \rad$. Let $\bU\in\R^{2\times d}$ be an orthogonal
transform ($\bU\bU^{\sT}=\id_{2\times 2}$) from 
$\reals^d$ to $\R^2$ whose row space contains $\{\btheta,\btheta_0\}$. Define the event $A_s=\{\|\bU\bX\|_2\le 2s /(3\rad)\}$. Recall that $\|\btheta_0\|_2\le \rad/3$. Then on $A_s$, we have $\max\{|\langle \btheta,\bX\rangle|,|\langle \btheta_0,\bX\rangle|,|\langle \btheta-\btheta_0,\bX\rangle|\}\le s$.

We define $L(s) = \inf_{0 < z \leq s} g(z)/z$ for any $s> 0$. Since we assume that $g(z) = \E[\psi(z+\eps)] > 0$ as $z > 0$, and that $g'(0) > 0$, it is easy to see that $L(s) > 0$ for all $s > 0$. Thus, we have
\[
\begin{aligned}
\langle \btheta - \btheta_0, \nabla R(\btheta) \rangle =& \E \Big[ \E [\psi(z + \eps) z \vert z =  \langle \btheta_0 - \btheta, \bX \rangle ] \Big] \\
=& \E [ g( \langle \btheta_0 - \btheta, \bX \rangle ) \langle \btheta_0 - \btheta, \bX \rangle ] \\
\geq & L(s) \E[\langle \btheta - \btheta_0, \bX \rangle^2 \ones_{A_s}] \\
= & L(s) \E [\langle \btheta - \btheta_0, \bX \rangle^2 - \langle \btheta - \btheta_0, \bX \rangle^2 \ones_{A_s^c}] \\ 
\geq &  L(s) \Big[ \covlb \tau^2 \Vert \btheta - \btheta_0 \Vert_2^2 - \Big(\E[\langle \btheta - \btheta_0, \bX\rangle^4] \cdot \P(A_s^c)\Big)^{1/2}\Big] \\ 
\geq & L(s) \Vert \btheta - \btheta_0 \Vert_2^2 \tau^2( \covlb -\sqrt{C_4 \cdot \P(A_s^c)}).
\end{aligned}
\]

In addition, since $\langle \bU_j,\bX\rangle$ is $\tau^2$-sub-Gaussian for $j=1,2$, we have
\begin{equation}
\P(A_s^c) = \P\Big( \|\bU\bX\|_2>\frac{2s}{3\rad} \Big) \le \sum_{j=1}^{2} \P\Big( |\langle \bU_j,\bX\rangle| \ge \frac{\sqrt{2}s}{3\rad} \Big) \le 4\exp\Big( -\frac{s^2}{9\rad^2\tau^2} \Big),
\end{equation}
giving us
\begin{equation}
\langle \btheta-\btheta_0,\grad R(\btheta) \rangle \ge L(s)\|\btheta-\btheta_0\|_2^2\tau^2\Big( \covlb - 2\sqrt{C_4}e^{-\frac{s^2}{18\rad^2\tau^2}} \Big).
\end{equation}
So choosing $s\ge \tilde{c}\rad\tau$ for some constant $\tilde{c}>0$
and $\gradlip=\covlb L(\tilde{c}\rad\tau)\tau^2/2$ ensures that
\begin{equation}
\langle \btheta-\btheta_0,\grad R(\btheta)\rangle \ge \gradlip \|\btheta-\btheta_0\|_2^2
\end{equation}
and also $\|\grad R(\btheta)\|_2\ge \gradlip\|\btheta-\btheta_0\|_2$ from the Cauchy-Schwarz inequality.

Finally, choosing $\gradlb=\eps_0\gradlip$ ensures that $\|\grad R(\btheta)\|_2\ge\gradlb$ for all $\btheta\in \Ball^d(\bzero,\rad)\setminus \Ball^d(\btheta_0,\eps_0)$.

\prg{Upper bounding the gradient. }

For any $\btheta \in \R^d$, we have 
\[
\begin{aligned}
\Vert \nabla R(\btheta) \Vert_2 =& \Vert \E [ \psi(\langle \bX, \btheta_0 - \btheta \rangle + \eps )\bX] \Vert_2 = \sup_{\Vert \bv \Vert_2 = 1} \E[\psi(\langle \bX, \btheta_0 - \btheta \rangle + \eps )\langle \bX, \bv \rangle]\\
\leq &\sup_{\Vert \bv \Vert_2 = 1} \E\Big[ \big\vert \psi(\langle \bX, \btheta_0 - \btheta \rangle + \eps ) \big\vert \cdot \vert \langle \bv, \bX \rangle \vert \Big] \leq \rholip \sqrt C_2 \tau. 
\end{aligned}
\]
The last inequality follows from lemma \ref{lem:subgaussian}(a) and a Cauchy-Schwarz inequality. Thus, $\gradub = \rholip \sqrt C_2 \tau$ upper bounds $\Vert \nabla R(\btheta) \Vert_2$ for all $\btheta \in \R^d$. 

\prg{Dependence on model parameters. }
Notice that all constants $\gradlb$, $\gradub$, $\gradlip$, $\heslb$, $\hesub$, $\eps_0$ does not depend on $d$ and the distribution of $\bX$. This completes the proof of all of our statements. 

\end{proof}

\subsection{Landscape of empirical risk}
Similar to Lemma \ref{lemma:empriskClassification} for binary linear classification, $n = \Omega(d\log d)$ guarantees that the empirical risk $\what{R}_n(\btheta)$ of robust regression has properties the same as the properties described in Lemma \ref{lemma:empriskClassification}. The proof is basically the same, except the only difference that, here the parameters will depend on constants $(\rho, \P_\eps,\rholip, \rad,\tau^2, \covlb)$. 

\subsection{Gradient descent algorithm}
Similar to Lemma \ref{lemma:GDalgorithm} for binary linear classification, as $n = \Omega(d\log d)$, gradient descent algorithm is provably converging exponentially fast to the global minimum of the empirical risk $\what R_n(\btheta)$. The proof is basically the same.

\section{Proof of Theorem \ref{thm:Mixture}: Gaussian mixture model}

\subsection{Landscape of population risk}

\begin{lemma} \label{lemma:popriskGaussianMixture}
There exist constants $\rad$, $\eps_0$, $\gradlb$, $\gradub$,
$\heslb$, $\hesub$ that depend on the separation parameter $D$ but independent of $d$, such that the landscape of the population risk has the following properties:
\begin{enumerate}[label=$(\alph*)$, leftmargin = 0.5cm]
\item {\bf Three stationary points.} The population risk $R(\btheta)$ is minimized at $\btheta_+ = (\btheta_{0,1}, \btheta_{0,2})$ and $\btheta_- = (\btheta_{0,2}, \btheta_{0,1})$. There is a saddle point $\btheta_s = ((\btheta_{0,1}+\btheta_{0,2})/2,(\btheta_{0,1}+\btheta_{0,2})/2)$. There are no other stationary points. 
\item {\bf Absorbing region. } With the constants $\rad > 0$, $\eps_0 >0$, and $\gradlb > 0$, we have
\begin{align}
\inf_{\btheta \in \partial \BT^{2d}(\btheta_s, \frac{\rad}{2} + \eps_s)} \langle \nabla R(\btheta), \bn (\btheta) \rangle \geq& \gradlb, \quad \forall \eps_s \in [0, \eps_0],
\end{align}
where $\BT^{2d}(\btheta_s, \rad_0) = \Ball^d(\btheta_{s,1}, \rad_0) \times \Ball^d(\btheta_{s,2}, \rad_0)$ with $\btheta_s = (\btheta_{s,1}, \btheta_{s,2}) \in \R^{2d}$. For any $\btheta = (\btheta_1,\btheta_2) \in \partial \BT^{2d}(\btheta_s, \rad_0)$, $\boldsymbol n(\btheta)$ is a unit normal vector pointing out of $\BT^{2d}(\btheta_s, \rad_0)$. More 
specifically, letting $\bn_1(\btheta_1) = (\btheta_1 - \btheta_{s,1})/\|\btheta_1 - \btheta_{s,1}\|_2$, and $\bn_2(\btheta_2) = (\btheta_2 - \btheta_{s,2})/\|\btheta_2 - \btheta_{s,2}\|_2$, we have $\bn(\btheta) = (\bn_1(\btheta_1), \bzero)$ if $\|\btheta_1-\btheta_{s,1}\|_2 = \rad_0>\|\btheta_2-\btheta_{s,2}\|_2$,
$\bn(\btheta) = (\bzero, \bn_2(\btheta_2))$ if $\|\btheta_2-\btheta_{s,2}\|_2 = \rad_0>\|\btheta_1-\btheta_{s,1}\|_2$, 
and $\bn(\btheta) = (\bn_1(\btheta_1), \bn_2(\btheta_2))/ \sqrt 2$ if $\|\btheta_1-\btheta_{s,1}\|_2 = \rad_0=\|\btheta_2-\btheta_{s,2}\|_2$,

\item {\bf Bounds on the Hessian.} With the constants $\rad$, $\eps_0$, and $0 < \heslb < \hesub < \infty$, we have
\begin{align}
\inf_{\btheta \in \Ball^{2d}(\btheta_{+}, \eps_0) \cup \Ball^{2d} (\btheta_{-}, \eps_0)} \lambda_{\min} \left( \nabla^2 R(\btheta) \right) \geq& \heslb,\\
\sup_{\btheta \in \Ball^{2d}(\btheta_{s}, \eps_0)} \lambda_{\min}\left( \nabla^2 R(\btheta) \right) \leq& - \heslb,\\
\sup_{\btheta \in \Ball^{2d}(\btheta_s, \rad)} \Big \Vert \nabla^2 R(\btheta) \Big \Vert_{\op} \leq& \hesub.
\end{align}

\item {\bf Bounds on the gradient.} With the constants $\rad$, $\eps_0$, and $0 < \gradlb < \gradub < \infty$, 
define $G_d \equiv \Ball^{2d}(\btheta_s, \rad) \setminus \cup_{\btheta_* = \btheta_+, \btheta_-, \btheta_s}  \Ball^{2d}(\btheta_*, \eps_0/2)$ 
we have
\begin{align}
\inf_{\btheta \in G_d} \Big \Vert \nabla R(\btheta) \Big \Vert_2 \geq \gradlb, \quad \sup_{\btheta \in \Ball^{2d}(\btheta_s, \rad)} \Big \Vert \nabla R(\btheta) \Big \Vert_2 \leq \gradub. 
\end{align}
\end{enumerate}
\end{lemma}

\begin{proof}
See proofs below. 
\prg{$(a)$ Three stationary points.} This has been proved by \cite{xu2016global}. 

\prg{$(b)$ Absorbing region.} In the rest of this proof, we will denote by $\Rd(\btheta)$ the population risk for the $d$-dimensional model.
A straightforward calculation yields
\begin{align} 
\nabla_{\btheta_1} \Rd(\btheta) &= \E\big\{ \sp_1(\bX;\btheta_1,\btheta_2) \cdot (\btheta_1-\bX)\}
\end{align}
where the expectation is with respect to $\bX\sim (1/2)\normal(\btheta_{0,1},\id_{d \times d})+  (1/2)\normal(\btheta_{0,2},\id_{d \times d})$, 
and $\sp_1(\bx;\btheta_1,\btheta_2)$ is the posterior probability for point $\bx$ to belong to component $1$
\begin{align}
\sp_1(\bx;\btheta_1,\btheta_2) = \frac{1}
{1+e^{\<\btheta_2-\btheta_1,\bx\>+\frac{1}{2}(\|\btheta_1\|_2^2-\|\btheta_2\|_2^2)}}\, .
\end{align}

Without loss of generality, we can assume $\btheta_{0,1} = \theta \be_1=-\btheta_{0,2}$, $\theta=D/2$. 
We will first consider the case $d=3$.
Consider a point $ \btheta = ( \btheta_1, \btheta_2) \in \partial \BT^6(\bzero, \rad)$, where $\btheta_1 \in \partial \Ball^3(\rad)$, and 
$\| \btheta_2 \|_2<\rad$. 
In this case $\bn(\btheta) = (\btheta_1/\|\btheta_1\|_2,\bzero)$.
We have, denoting $\sp_1 = \sp_1(\bX;\btheta_1,\btheta_2) > 0$,
\begin{align}
\langle \nabla_{\btheta}  \Rt( \btheta), (\btheta_1, \bzero) \rangle &=
\E\big\{ \sp_1\, \<\btheta_1,\btheta_1-\bX\>\}\\
& =\E\big\{ \sp_1\, \<\btheta_1,(\btheta_1-\bX)\bfone_{\|\bX\|_2\le \rad}\>\} - 
\E\big\{ \sp_1\, \<\btheta_1,(\bX-\btheta_1)\bfone_{\|\bX\|_2> \rad}\>\} 
\\
&\ge \E\big\{ \sp_1\, \<\btheta_1,(\btheta_1-\bX)\bfone_{\|\bX\|_2\le \rad/2}\>\} - 
\E\big\{ \sp_1 \, \<\btheta_1,(\bX-\btheta_1)\bfone_{\|\bX\|_2> \rad}\>\} \\
&\ge \frac{\rad^2}{2}\E\big\{ \sp_1 \,\bfone_{\|\bX\|_2\le \rad/2}\} - 
\E\big\{ \sp_1\, \big|\<\btheta_1,(\bX-\btheta_1)\bfone_{\|\bX\|_2> \rad}\>\big|\} \\
&\equiv  D_1(\rad) - D_2(\rad).
\end{align}

For two non-negative functions $f(\rad)$ and $g(\rad)$, we  write 
$f(\rad)\gtrdot g(\rad)$ (or $g(\rad)\lessdot f(\rad)$) if $f$ dominates $g$ to leading exponential orderin $\rad^2$, i.e. 
if
\begin{align}
\lim_{\rad \rightarrow \infty} \frac{1}{\rad^2} \log\frac{f(\rad)}{g(\rad)} \ge 0 \, .
\end{align}
We write $f(\rad) \doteq g(\rad)$ if  $f(\rad) \lessdot g(\rad)$ and $f(\rad) \gtrdot g(\rad)$.

We estimate $D_1(\rad)$ for large $\rad$ by Laplace method
\begin{align}
D_1(\rad) =& \frac{1}{2}\int_{\Ball^3(\frac{\rad}{2})} [\phi(\bx-\theta\be_1)+\phi(\bx+\theta\be_1)] 
\sp_1(\bx;\btheta_1,\btheta_2)\, \ud \bx\\
\gtrdot & [\phi(\btheta_1/2-\theta\be_1)+\phi(\btheta_1/2+\theta\be_1)]\, \sp_1(\btheta_1/2;\btheta_1,\btheta_2)\\
\doteq& \exp(-\rad^2/8) \frac{1}{1 + \exp(\langle \btheta_1, \btheta_2 \rangle/2 - \Vert \btheta_2 \Vert_2^2/2)} \\
\gtrdot& \exp(-\rad^2/8) \frac{1}{1 + \exp( -  ( \Vert \btheta_2 \Vert_2 - \frac{1}{2}\Vert \btheta_1 \Vert_2)^2/2 + \Vert \btheta_1 \Vert_2^2/8)} \\
\gtrdot& \exp(-\rad^2/8) \frac{1}{ 1+ \exp(\rad^2/8)}\\
\doteq&\exp(- \rad^2/4). 
\end{align}

By a similar calculation
\begin{align}
D_2(\rad) \lessdot & \int_{\Ball^3(\rad)^c} [\phi(\bx-\theta\be_1)+\phi(\bx+\theta\be_1)] 
\sp_1(\bx;\btheta_1,\btheta_2) \, \ud\bx\\
\doteq & \sup_{\bx \in \Ball^3(\rad)^c} [\phi(\bx-\theta\be_1)+\phi(\bx+\theta\be_1)] \sp_1(\bx;\btheta_1,\btheta_2) \\
\lessdot & \sup_{\bx \in \Ball^3(\rad)^c}  [\phi(\bx-\theta\be_1)+\phi(\bx+\theta\be_1)] \\
\doteq & \exp(-\rad^2/2).
\end{align}
Using these inequalities, we obtain $D_1(\rad) - D_2(\rad)\gtrdot \exp(- \rad^2/4)$. 
Consequently, there exists an $\rad$ (which is independent of $d$) such that for any $\rad_s \geq \rad/2$, we have
\[
\inf_{\btheta \in \partial \BT^6(\bzero, \rad_s)} \langle \nabla_{ \btheta} \Rt( \btheta), \bn (\btheta) \rangle >0. 
\]

On the other hand, for any $\eps_0 \geq 0$, $\cup_{\eps_s\in[0,\eps_0]}\partial \BT^6(\bzero, (\rad/2) + \eps_s) $ is a compact set. 
A continuous function on a compact set can attain its infimum. Therefore, for any $\eps_0 >0$, there exists an $\gradlb > 0$, such that
\begin{align}
\inf_{\btheta \in \partial \BT^6(\bzero, \frac{\rad}{2} + \eps_s)} \langle \nabla_{\btheta} \Rt( \btheta), \boldsymbol n (\btheta) \rangle \geq \gradlb, \quad \forall \eps_s \in [0,\eps_0]. 
\end{align}
We will choose this $\eps_0$ in part $(c)$ of the proof. 

Then, we consider the $d$-dimensional Gaussian mixture model. For any point $ \btheta \in \BT^{2d}(\bzero, \rad)$, we denote $\btheta = (\btheta_1, \btheta_2)$. We can find a rotation matrix $U \in \R^{d \times d}$, such that $(U(\btheta_1))_1 = (\btheta_1)_1$, $(U(\btheta_2))_1 = (\btheta_2)_1$, and $(U(\btheta_1))_{4:d} = \bzero$, $(U(\btheta_2))_{4:d} = \bzero$. With a little abuse of notation, we denote $U(\btheta) = (U(\btheta_1), U(\btheta_2))$. Due to the symmetry property of the Gaussian mixture model, $U(\btheta)$ satisfy the following properties:
\begin{align}
\Vert U(\btheta) - \btheta_* \Vert_2 =& \Vert \btheta - \btheta_* \Vert_2, \quad \text{ for }\btheta_* = \btheta_+, \btheta_-, \btheta_s, \\
\langle \nabla \Rd(\btheta), \boldsymbol n(\btheta) \rangle =& \langle \nabla \Rd(U(\btheta)), \boldsymbol n(U(\btheta)) \rangle, \quad \forall \btheta \in \partial \BT^{2d}(\bzero, \rad), 
\end{align}
and $U(\BT^{2d}(\bzero, \rad)) = \BT^{2d}(\bzero, \rad)$. 

Thus, we have
\[
\inf_{\btheta \in \partial \BT^{2d}(\bzero,\frac{\rad}{2}+\eps_s)} \langle \nabla \Rd(\btheta), \boldsymbol n(\btheta) \rangle = \inf_{ \btheta \in \partial \BT^{6}(\bzero, \frac{\rad}{2}+\eps_s)} \langle \nabla \Rt( \btheta), \boldsymbol n(\btheta) \rangle \geq \gradlb, \quad \forall \eps_s \in [0, \eps_0]. 
\]

\prg{$(c)$ Bounds on the Hessian. } 
First, we consider the case $d=3$. Due to part $(a)$, we know that this model has exactly three critical points. A direct calculation of the
Hessian shows that $\btheta_+ = (\btheta_{0,1}, \btheta_{0,2})$ and $\btheta_- = (\btheta_{0,2},  \btheta_{0,1})$ are strict local minimums, and 
$ \btheta_s = (\bzero,\bzero)$ is a strict saddle point. That is, we have $\lambda_{\min}(\nabla^2 \Rt(\btheta_+)) = \lambda_{\min}(\nabla^2 \Rt( \btheta_-)) > 0$, and $\lambda_{\min}(\nabla^2 \Rt(\btheta_s)) < 0$. Due to the continuity of the Hessian, we know that there exist two positive constants $\eps_0$ and 
$\heslb$, such that $\inf_{ \btheta \in \Ball^6(\btheta_+, \eps_0) \cup \Ball^6( \btheta_-, \eps_0)} \lambda_{\min}(\nabla^2 \Rt(\btheta)) \geq \heslb$, 
and $\sup_{ \btheta \in \Ball^6( \btheta_s, \eps_0)} \lambda_{\min}(\nabla^2 \Rt(\btheta)) \leq -\heslb$. Further, since $\Ball^6(\bzero, \rad)$ is a compact set, and the Hessian is continuous, we have $\sup_{ \btheta \in \Ball^6(\bzero, \rad)} \Vert \nabla^2 \Rt(\btheta) \Vert_{\op} \leq \hesub$. 

Next,  we consider the general case  $d\ge 3$. For any $\btheta \in \Ball^{2d}(\bzero, \rad)$, we also define the rotational matrix 
$U \in \R^{d \times d}$ as in part $(b)$. Due to the symmetry property of the Gaussian mixture model, we have
\begin{align}
\lambda_i(\nabla^2 \Rd(\btheta)) = & \lambda_i (\nabla^2 \Rd(U(\btheta))), \quad \forall i \in [d], \quad \forall \btheta \in  \Ball^{2d}(\btheta_s,\rad),
\end{align}
and $U(\Ball^{2d}(\btheta_*,R_0)) = \Ball^{2d}(\btheta_*,R_0)$ for any $\btheta_* = \btheta_+, \btheta_-, \btheta_s$ and for any $R_0 > 0$. 

Thus, we have
\begin{align}
\inf_{\btheta \in \Ball^{2d}(\btheta_+, \eps_0) \cup \Ball^{2d}(\btheta_-, \eps_0)} \lambda_{\min} (\nabla^2 \Rd(\btheta)) =& \inf_{ \btheta \in \Ball^{6}( \btheta_+, \eps_0) \cup \Ball^{6}( \btheta_-, \eps_0)} \lambda_{\min} (\nabla^2  \Rt( \btheta)) \geq \heslb,\\
\sup_{\btheta \in \Ball^{2d}(\btheta_s, \eps_0)} \lambda_{\min} (\nabla^2 \Rd(\btheta)) =& \sup_{ \btheta \in \Ball^{6}( \btheta_s, \eps_0)} \lambda_{\min} (\nabla^2 \Rt( \btheta)) \leq -\heslb,\\
\sup_{\btheta \in \Ball^{2d} (\bzero, \rad)} \Big \Vert \nabla^2 \Rd(\btheta) \Big \Vert_{\op} =& \sup_{ \btheta \in \Ball^6(\bzero, \rad)} \Big \Vert \nabla^2  \Rt( \btheta) \Big \Vert_{\op} \leq \hesub. 
\end{align}

\prg{$(d)$ Bounds on the gradient. } 

First, we consider the case $d=3$.
Since the closure of $G_3 = \left(\Ball^6( \bzero, \rad) \setminus \cup_{\btheta_* = \btheta_+,  \btheta_-, \btheta_s} \Ball^6(\btheta_*, \eps_0/2)\right)$ and 
$\Ball^6(\bzero, \rad)$ are two compact set, and $\nabla  \Rt( \btheta)$ is a continuous non-zero function on this set, there exist two constants 
$0 < \gradlb < \gradub < \infty$ such that
\begin{align}
\inf_{ \btheta \in G_3}  \Big \Vert \nabla \Rt( \btheta) \Big \Vert_2 \geq & \gradlb, \\
\sup_{ \btheta \in \Ball^6(\bzero, \rad)} \Big \Vert \nabla \Rt(\btheta) \Big \Vert_2 \leq & \gradub.
\end{align}
Next, we consider our $d$-dimensional Gaussian mixture model. For any $\btheta \in \Ball^{2d}(\bzero, \rad)$, we also define the rotation 
matrix $U \in \R^{d \times d}$ as in part $(b)$. Due to the symmetry property of the Gaussian mixture model, we have
\begin{align}
\Big \Vert \nabla \Rd(\btheta) \Big \Vert_2 = & \Big \Vert \nabla \Rd(U(\btheta)) \Big \Vert_2, \quad \forall \btheta \in \Ball^{2d}(\bzero, \rad).  
\end{align}

Thus, we have
\[
\begin{aligned}
\inf_{\btheta\in G_d}  \Big \Vert \nabla \Rd(\btheta) \Big\Vert_2 =& \inf_{\btheta\in G_3} \Big \Vert \nabla \Rt(\btheta) \Big \Vert_2 \geq \gradlb, \\
\sup_{\btheta \in \Ball^{2d}(\bzero, \rad)} \Big\Vert \nabla \Rd(\btheta) \Big\Vert_2 =& \sup_{ \btheta \in \Ball^6(\bzero, \rad)} \Big \Vert \nabla \Rt(\btheta) \Big \Vert_2 \leq \gradub.
\end{aligned}
\]

\prg{Dependence on model parameters.} Notice that all constants $\gradlb, \gradub, \heslb, \hesub, \eps_0, \rad$ are defined in the three dimensional Gaussian mixture model, and thus depend only on $D$ but not depend on $d$. This completes the proof of all of our statements. 
\end{proof}

\subsection{Landscape of empirical risk}

\begin{lemma}\label{lemma:empriskGaussianMixture}
Let $\rad$, $\eps_0$, $\heslb$, $\hesub$, $\gradlb$, $\gradub$ be the constants defined in Lemma \ref{lemma:popriskGaussianMixture}, then there exists a large constant $C$ depending on $(D, \delta)$, such that as $n \geq C d \log d$, the following hold with probability at least $1-\delta$:
\begin{enumerate}[label=$(\alph*)$, leftmargin = 0.5cm]
\item { \bf Bounds on the Hessian.} 
\begin{align}
\inf_{\btheta \in \Ball^{2d}(\btheta_{+}, \eps_0) \cup \Ball^{2d} (\btheta_{-}, \eps_0)} \lambda_{\min} \left( \nabla^2 \what R_n(\btheta) \right) \geq& \heslb/2, \\
\sup_{\btheta \in \Ball^{2d}(\btheta_{s}, \eps_0)} \lambda_{\min} \left( \nabla^2 \what R_n(\btheta) \right) \leq& -\heslb/2;\\
\sup_{\btheta \in \Ball^{2d}(\btheta_s, \rad)}  \Big\Vert \nabla^2 \what R_n(\btheta) \Big \Vert_{\op} \leq& 2 \hesub. 
\end{align}
\item { \bf Bounds on the gradient.} Letting $G_d = 
\Ball^{2d}(\btheta_s, \rad) \setminus \cup_{\btheta_* = \btheta_+, \btheta_-, \btheta_s}  \Ball^{2d}(\btheta_*, \eps_0/2)$
\begin{equation}
\inf_{\btheta \in G_d}\Big\Vert \nabla \what R_n(\btheta) \Big\Vert_2 \geq \gradlb/2, \quad \sup_{\btheta \in \Ball^{2d}(\btheta_s, \rad)} \Big\Vert \nabla \what R_n(\btheta) \Big\Vert_2 \leq 2\gradub.
\end{equation}
\item {\bf Absorbing region. }
\begin{align}
\inf_{\btheta \in \partial \BT^{2d}(\btheta_s, \frac{\rad}{2} + \eps_s)} \langle \nabla \what R_n(\btheta), \boldsymbol n(\btheta) \rangle \geq& \gradlb/2,\quad \forall \eps_s \in [0,\eps_0].
\end{align}

\item {\bf Stationary points.} The empirical risk $\what R_n(\btheta)$ has two local minimum inside $\Ball^{2d}(\btheta_+, C \sqrt{(d \log n)/n})$ and $\Ball^{2d}(\btheta_-, C \sqrt{(d \log n)/n})$ respectively. There are saddle points but no local minima inside $\Ball^{2d}(\btheta_s, \eps_0)$. The empirical risk has no other critical points. 
\end{enumerate}
\end{lemma}

\begin{proof}
We begin by verifying Assumptions \ref{ass:GradientSub}, \ref{ass:Hessianub}, and \ref{ass:BoundedHessian} for Theorem \ref{thm:uniformconvergence1}. In the following calculations, we let $\bv = (\bv_1, \bv_2) \in  \partial \Ball^{2d}(1)$, $\btheta = (\btheta_1, \btheta_2) \in \Ball^{2d}(\rad)$. Denote $\phi_i =1 /(2\pi)^{d/2} \cdot \exp(-\Vert \bX - \btheta_i \Vert_2^2/2)$, $i=1,2$. Let $w_1 = \phi_1/(\phi_1 + \phi_2)$, $w_{12} = \phi_1 \phi_2/(\phi_1 + \phi_2)^2$, and $w_{112} = \phi_1 \phi_2 (\phi_2 - \phi_1)/(\phi_1 + \phi_2)^3$. It is easy to see that $w_1, w_{12}, w_{112} \in [-1,1]$.

\vskip 0.1cm
\noindent{\bf Assumption \ref{ass:GradientSub}.} We need to verify that the directional gradient of the loss function is $O(1 + D^2 + \rad^2)$-sub-Gausian. The directional gradient of the loss is given by
\begin{equation}
\begin{aligned}
\langle \nabla \ell (\btheta;\bX), \bv \rangle = & \langle w_1(\btheta_1 - \bX),\bv_1 \rangle+\langle  (1-w_1)(\btheta_2 - \bX), \bv_2 \rangle \\
=& -\langle w_1 \bX, \bv_1 - \bv_2 \rangle - \langle \bX, \bv_2 \rangle \\
&+ \langle w_1 \btheta_1, \bv_1 \rangle + \langle (1-w_1) \btheta_2, \bv_2 \rangle,
\end{aligned}
\end{equation}
Since $\Vert \bv_i \Vert_2 \leq 1$, $\langle \bX, \bv_i \rangle$ is $(1 + D^2)$-sub-Gaussian. Since $w_1$ is a bounded random variable in the interval $[0,1]$, $\langle w_1 \btheta_1, \bv_1\rangle$ and $\langle (1-w_1)\btheta_2, \bv_2 \rangle$ are bounded by $\rad$, and thus are $O(\rad^2)$-sub-Gaussian. Next, the random variable $\langle w_1 \bX, \bv_1- \bv_2 \rangle$ is the product of a bounded random variable and a zero mean Gaussian random variable, and due to Lemma \ref{lem:subgaussian}.$(d)$, it is $O(1+D^2)$-sub-Gaussian. According to Lemma \ref{lem:sumofsub}, there exists a universal constant $C_1$, such that $\langle \nabla \ell(\btheta;\bZ), \bv \rangle$ is $C_1(1+D^2 + \rad^2)$-sub-Gaussian. 

\vskip 0.1cm
\noindent{\bf Assumption \ref{ass:Hessianub}.} We need to verify that the directional Hessian of the loss function is $O(1 + D^2 + \rad^2)$-sub-exponential. The directional Hessian of the loss is given by
\begin{equation}
\begin{aligned}
\langle \bv, \nabla^2 \ell (\btheta; \bX) \bv \rangle =& w_1 \Vert \bv_1 \Vert_2^2 + (1-w_1) \Vert \bv_2 \Vert_2^2 - w_{12} \Big( \langle \bX - \btheta_1, \bv_1 \rangle - \langle \bX - \btheta_2, \bv_2 \rangle \Big)^2.
\end{aligned}
\end{equation}
First, observe that $w_1 \Vert \bv_1 \Vert_2^2 + (1-w_1) \Vert \bv_2 \Vert_2^2$ is $1$ bounded, thus it is $O(1)$-sub-Gaussian. Then, we have $\langle \bX, \bv_1\rangle - \langle \bX, \bv_2\rangle$ is $O(1 + D^2)$-sub-Gaussian, and $\langle \btheta_1, \bv_1\rangle - \langle \btheta_2, \bv_2 \rangle$ is $O(\rad)$ bounded. Thus, letting $\eps_r$ be a Rademacher random variable, the random variable 
$\eps_r(\langle \bX - \btheta_1, \bv_1 \rangle - \langle \bX - \btheta_2, \bv_2 \rangle)$ is mean zero and $O(1 + D^2+\rad^2)$-sub-Gaussian. According to Lemma \ref{lem:subgaussian}.$(c)$, $(\langle \bX - \btheta_1, \bv_1 \rangle - \langle \bX - \btheta_2, \bv_2 \rangle)^2$ is $O(1 + D^2+\rad^2)$-sub-exponential. Note that $w_{12}$ is $1$ bounded, according to Lemma 
\ref{lem:subgaussian}.$(e)$, the random variable $w_{12}\{(\langle \bX - \btheta_1, \bv_1 \rangle - \langle \bX - \btheta_2, \bv_2 \rangle )^2 - \E[( \langle \bX - \btheta_1, \bv_1 \rangle - \langle \bX - \btheta_2, \bv_2 \rangle )^2]\}$ is $O(1+D^2 + \rad^2)$-sub-exponential. 
Further, we know that $\E[ (\langle \bX - \btheta_1, \bv_1 \rangle - \langle \bX - \btheta_2, \bv_2 \rangle)^2 ] = O(1+ D^2 + \rad^2)$.  Next, since $w_{12}$ is $O(1)$-sub-exponential, the bounded random variable $\bw_{12} \E [ \langle \bX - \btheta_1, \bv_1 \rangle - \langle \bX - \btheta_2, \bv_2 \rangle ]^2$ is $O(1 + D^2 + \rad^2)$-sub-exponential. According to 
Lemma \ref{lem:sumofsub}, there exists a universal constant $C_2$, such that $\langle\bv,  \nabla^2 \ell(\btheta;\bX) \bv \rangle $ is $C_2 (1+D^2 + \rad^2)$-sub-exponential. 

\vskip 0.1cm
\noindent{\bf Assumption \ref{ass:BoundedHessian}.} We need to verify that there exists a constant $\Ch$ which does not depend on $d$, such that $\HUB \leq \tau^2 d^{\Ch}$ and $J_*\leq  \tau^3 d^{\Ch}$ (as $d \geq 2$). Here, we assume, without loss of generality, 
$\btheta_0= \btheta_{0,1} = -\btheta_{0,2}$:
\[
\begin{aligned}
\HUB = & \sup_{\Vert \bv\Vert_2 = 1} \Big|\langle \bv, \nabla^2 R(\bzero) \bv \rangle \Big|\\
= & \sup_{\Vert \bv\Vert_2 = 1} \frac{1}{2} \left[1 - \frac{1}{2}(\langle \btheta_0, \bv_1\rangle - \langle \btheta_0, \bv_2 \rangle)^2 \right] \leq C_3 (1 + D^2),\\
J_* = & \E[\sup_{\btheta \in\Ball^{2d}(\rad)} \sup_{\Vert \bv\Vert_2 = 1} \langle\nabla^3 \ell(\btheta;\bX), \bv^{\otimes 3}\rangle]\\
\leq & \E\Big[\sup_{\btheta \in\Ball^{2d}(\rad)}\sup_{\Vert \bv\Vert_2 = 1} 3 w_{12} (\Vert \bv_1 \Vert_2^2 - \Vert \bv_2 \Vert_2^2) (\langle \bX - \btheta_1, \bv_1 \rangle - \langle \bX - \btheta_2, \bv_2 \rangle) \\
&- w_{112}(\langle \bX - \btheta_1, \bv_1 \rangle - \langle \bX - \btheta_2, \bv_2 \rangle )^3 \Big ]\\
= & O \Big( \E\Big[\sup_{\btheta \in\Ball^{2d}(\rad)}\sup_{\bv \in \Ball^{2d}(1)} \Big \vert \langle \bX - \btheta_1, \bv_1 \rangle - \langle \bX - \btheta_2, \bv_2 \rangle\Big 
\vert \\
&+ \Big \vert ( \langle \bX - \btheta_1, \bv_1 \rangle - \langle \bX - \btheta_2, \bv_2 \rangle )^3 \Big \vert \Big ] \Big)\\
= & O \Big( \E \Vert \bX \Vert_2 + \Vert \btheta_1 \Vert_2 + \Vert \btheta_2 \Vert_2 + \E \Vert \bX \Vert_2^3 + \Vert \btheta_1 \Vert_2^3 + \Vert \btheta_2 \Vert_2^3 \Big) \\\leq& C_4(d^{3/2} + D^3 + \rad^3). \\
\end{aligned}
\]

%
%
%
%

In Assumptions \ref{ass:GradientSub} and \ref{ass:Hessianub}, we take $\tau^2 = \max\{C_1, C_2\}\cdot(1 + D^2 + \rad^2)$, and in Assumption \ref{ass:BoundedHessian}, we take $\Ch =  \max\{ \log_2 (C_3),  \log_2 C_4 + 3/2 \}$. Then, all the assumptions for Theorem \ref{thm:uniformconvergence1} are satisfied. Now we take $\eps_g = \gradlb/2$ and $\eps_h = \heslb/2$ which depends on $D$ but not on $d$. By Theorem \ref{thm:uniformconvergence1}, there exist constants $C$ and $C'$ depending on $(D,\delta)$ but independent of $n$ and $d$, such that as $n$ is large enough when $n \geq C d \log {d}$, with probability at least $1-\delta$, the following good event happens: 
\begin{equation}
E_{\rm good} = \left\{
\begin{aligned}
&\sup_{\btheta \in \Ball^{2d}(\btheta_s,\rad + \eps_0)}  \left \Vert \nabla \what R_n(\btheta) - \nabla R(\btheta) \right \Vert_2 \leq \tau \sqrt{\frac{C' d \log n}{n}} \leq \eps_g,\\
&\sup_{\btheta \in \Ball^{2d}(\btheta_s,\rad + \eps_0)} \left\Vert \nabla^2 \what R_n(\btheta) - \nabla^2 R(\btheta) \right\Vert_{\op} \leq \tau^2 \sqrt{\frac{C' d\log n}{n}} \leq \eps_h.
\end{aligned}
\right\}
\end{equation}

All the following arguments are deterministic on the good event $E_{\rm good}$. 

\prg{Part $(a)$, $(b)$, $(c)$.} 
Given the uniform convergence of the gradient and Hessian and the properties of the population risk, these properties of the empirical risk are obvious.

\prg{Part $(d)$.}
Part $(b)$ implies that there is no local minimum inside the interior of $G_d = \Ball^{2d} (\btheta_s, \rad) \setminus \cup_{\btheta_* = \btheta_+, \btheta_-, \btheta_s} \Ball^{2d}(\btheta_*, \eps_0/2)$. Part $(a)$ implies that there is a unique minimum inside $\Ball^{2d}(\btheta_+, \eps_0/2)$ and $\Ball^{2d}(\btheta_-, \eps_0/2)$.

Let $\hat \btheta_{+} \in \Ball^{2d}(\btheta_+, \eps_0/2)$ denotes the unique local minimizer near $\btheta_+$. Note that, by the intermediate value theorem, there exists $\btheta' \in\Ball^{2d}(\btheta_+,\eps_0/2)$ such that
\begin{align}
\hR_n(\hat \btheta_+) = \hR_n(\btheta_+)+\<\nabla\hR_n(\btheta_+),\hat \btheta_+-\btheta_+\> + \frac{1}{2}\<\nabla^2\hR_n(\btheta'),(\hat \btheta_+-\btheta_+)^{\otimes 2}\> \le  \hR_n(\btheta_+)\, .
\end{align}
where the inequality follows by optimality of $\hR_n(\hat \btheta_+)$. Using Cauchy-Schwarz inequality, the lower bound on the Hessian
in point $(b)$, and the uniform convergence of the gradient, we get 
\[
\begin{aligned}
\|\hat \btheta_+-\btheta_+\|_2\le& \frac{4\|\nabla\hR_n (\btheta_+)\|_2}{\heslb}\\
\leq & \frac{4\tau}{ \heslb} \sqrt{\frac{C' d \log n}{n}}.
\end{aligned}
\]
The same conclusions hold for $\hat \btheta_-$. 

Note that the Hessian of empirical risk in $\Ball^{2d}(\btheta_s, \eps_0)$ has negative eigenvalues. Thus, there would be no local minimum inside $\Ball^{2d}(\btheta_s, \eps_0)$.

\end{proof}

\subsection{Trust region algorithm}

\begin{lemma}
For the Gaussian mixture model, there exists large constants $C_1$ and $C_2$ depending on $(D, \delta)$, such that as $n \geq C_1 d \log d$, the following holds.
The trust region iteration converges to one of the global minima of the empirical risk for any initialization in $\BT^{2d}(\btheta_s, C_2)$.
\end{lemma}
\begin{proof}
Let $C_1 = C$ and $C_2 = \rad/2$ where $C$ is as defined in Lemma \ref{lemma:empriskGaussianMixture} and $\rad$ is as defined in Lemma \ref{lemma:popriskGaussianMixture}. Then, all the conclusions for Lemma \ref{lemma:empriskGaussianMixture} and Lemma \ref{lemma:popriskGaussianMixture} hold. Since $\BT^{2d}(\btheta_s, \rad/2)$ is an absorbing region and trust region method is a descent method, for any initial point in $\BT^{2d}(\btheta_s, \rad/2)$ with suitable parameters, any iterate will not go out of the region $\BT^{2d}(\btheta_s, \rad/2)$. 

According to Lemma \ref{lemma:empriskGaussianMixture}, for every $\btheta\in\BT^{2d}(\btheta_s, \rad/2)$, one of the three things will happen: the gradient
$\nabla \hR_n(\btheta)$ is lower bounded by $\gradlb/2$ (in the region $\BT^{2d}(\btheta_s, \rad/2)\setminus( \Ball^{2d}(\btheta_+, \eps_0) \cup \Ball^{2d}(\btheta_-, \eps_0) \cup \Ball^{2d}(\btheta_s, \eps_0)) $); the Hessian $\nabla^2 \hR_n(\btheta)$ has a direction of negative curvature, namely an  eigenvalue upper bounded by $-\heslb/2$ (in the region $\Ball^{2d}(\btheta_s, \eps_0)$); the least eigenvalue of the Hessian $\nabla^2 \hR_n(\btheta)$ is lower bounded by $\heslb/2$ (in the region $\Ball^{2d}(\btheta_+, \eps_0) \cup \Ball^{2d}(\btheta_-, \eps_0)$). Further, it is immediate that $\Vert \nabla^3 \what R_n(\btheta)\Vert_{\op} < \infty$. According to standard trust region method results \cite[Theorem 6.6.4]{conn2000trust}, running trust region method for a finite number of steps will find a point $\btheta$ with $\Vert \nabla \what R_n(\btheta) \Vert_2 \leq \gradlb/3$ and $\nabla^2 \what R_n(\btheta) \succeq - \heslb/3$. Such point $\btheta$ will be within $\Ball^{2d}(\btheta_+,\eps_0)\cup \Ball^{2d}(\btheta_-,\eps_0)$. 

Since the empirical risk in $\Ball^{2d}(\btheta_+,\eps_0)$ and $\Ball^{2d}(\btheta_-,\eps_0)$ are both strongly convex, if we continue running trust region method, the iterates will eventually converge one of the local minima. 
\end{proof}

\bibliographystyle{amsalpha}
\newcommand{\etalchar}[1]{$^{#1}$}
\providecommand{\bysame}{\leavevmode\hbox to3em{\hrulefill}\thinspace}
\providecommand{\MR}{\relax\ifhmode\unskip\space\fi MR }
\providecommand{\MRhref}[2]{%
  \href{http://www.ams.org/mathscinet-getitem?mr=#1}{#2}
}
\providecommand{\href}[2]{#2}

\end{document}